\documentclass[11pt]{amsart}

\usepackage[letterpaper,top=3.5cm,bottom=3.5cm,left=3cm,right=3cm,marginparwidth=1.75cm]{geometry}
\usepackage{amsmath,amsfonts,amsthm,amssymb,amscd}
\usepackage{lipsum}
\usepackage{amsfonts}
\usepackage{graphicx}
\usepackage{epstopdf}
\usepackage{algorithmic}
\usepackage{appendix}
\usepackage{subfig}
\usepackage{graphicx}
\usepackage{color}
\usepackage{xcolor}
\usepackage[hidelinks]{hyperref}
\usepackage{cleveref}
\usepackage{bm}
\usepackage{soul}
\usepackage[normalem]{ulem}
\usepackage[foot]{amsaddr}

\usepackage{etoolbox}
\appto\appendix{\addtocontents{toc}{\protect\setcounter{tocdepth}{1}}}
\appto\listoffigures{\addtocontents{lof}{\protect\setcounter{tocdepth}{1}}}
\appto\listoftables{\addtocontents{lot}{\protect\setcounter{tocdepth}{1}}}

\newcommand{\N}{\mathcal{N}}

\newcommand{\E}{\mathbb{E}}

\newcommand{\V}{\Phi_R}
\newcommand{\VL}{\Phi}
\newcommand{\calD}{\mathcal{D}}
\newcommand{\PP}{\mathcal{P}}
\newcommand{\PPG}{\mathcal{P}^G}
\newcommand{\EE}{\mathcal{E}}
\newcommand{\Cov}{\mathrm{Cov}}
\newcommand{\M}{M}

\newcommand{\I}{I}
\newcommand{\R}{\mathbb{R}}
\newcommand{\T}{\mathcal{T}}

\newcommand{\rhoa}{a}
\newcommand{\rrho}{\mathfrak{r}}
\newcommand{\fg}{\mathfrak{g}}
\newcommand{\fM}{\mathfrak{M}}
\newcommand{\la}{\lambda_{\star,\max}}

\newcommand{\Prec}{P}
\newcommand{\bR}{\mathbb R}

\newcommand{\bigO}{\mathcal{O}}
\newcommand{\Hess}{\nabla_{\theta}\nabla_{\theta}}
\newcommand{\ranglen}{\rangle_{\R^{N_a}}}

\DeclareMathOperator*{\argmin}{arg\,min}
\newcommand{\mA}{\mathsf{A}}
\newcommand{\mb}{\mathsf{b}}
\newcommand{\dd}{\mathrm{d}}

\newenvironment{newremark}[1]{%
    \begin{remark}#1}{%
    \Endofdef\end{remark}%
}
\newcommand{\xqed}[1]{%
    \leavevmode\unskip\penalty9999 \hbox{}\nobreak\hfill
    \quad\hbox{\ensuremath{#1}}}
\newcommand{\Endofdef}{\xqed{\lozenge}}

\newtheorem{theorem}{Theorem}[section]
\newtheorem{lemma}[theorem]{Lemma}

\newtheorem{proposition}[theorem]{Proposition}
\newtheorem{definition}[theorem]{Definition}

\theoremstyle{remark}
\newtheorem{remark}[theorem]{Remark}
\numberwithin{equation}{section}

\definecolor{darkred}{rgb}{.6,0,0}
\definecolor{darkblue}{rgb}{0,0,.7}
\definecolor{darkgreen}{rgb}{0,.7,0}
\definecolor{darkbrown}{rgb}{0.8,0.4,0.4}

\newcommand{\as}[1]{{\color{black}{#1}}}

\newcommand{\sr}[1]{{\color{black}{#1}}}

\begin{document}
\title[Gradient Flows for Sampling]{Gradient Flows for Sampling: Mean-Field Models, Gaussian Approximations
and Affine Invariance}


\author{Yifan~Chen\textsuperscript{2,1}}
\address{\textsuperscript{1}California Institute of Technology, Pasadena, CA}
\address{\textsuperscript{2}Courant Institute, New York University, NY}
\email{yifan.chen@nyu.edu}

\author{Daniel~Zhengyu~Huang\textsuperscript{3,1}}
\email{huangdz@bicmr.pku.edu.cn}
\address{\textsuperscript{3}Beijing International Center for Mathematical Research, Peking University, Beijing, China}
\author{Jiaoyang Huang\textsuperscript{4}}
\address{\textsuperscript{4}University of Pennsylvania, Philadelphia, PA}
\email{huangjy@wharton.upenn.edu}
\vspace{0.1in}
\author{Sebastian Reich\textsuperscript{5}}
\address{\textsuperscript{5}Universit\"{a}t Potsdam, Potsdam, Germany}
\email{sebastian.reich@uni-potsdam.de}
\author{Andrew M. Stuart\textsuperscript{1}}
\email{astuart@caltech.edu}
    \keywords{Bayesian inference, sampling, gradient flow,  mean-field dynamics, Gaussian approximation, variational inference, affine invariance.}    
    \subjclass[2010]{68Q25, 68R10, 68U05}
\maketitle

\vspace{-3em}
\begin{abstract}
Sampling a probability distribution with an unknown normalization constant is a fundamental problem in computational science and engineering. This task may be cast as an optimization problem over all probability measures, and an initial distribution can be evolved to the desired minimizer
(the target distribution) dynamically via gradient flows. Mean-field models,
whose law is governed by the gradient flow in the space of probability measures, 
may also be identified; particle approximations of these mean-field models form the
basis of algorithms. The gradient flow approach is also the basis of algorithms for variational inference, in which the optimization is performed over a parameterized family of probability distributions such as Gaussians, and the underlying gradient flow is restricted to the parameterized family. 

By choosing different energy functionals and metrics for the gradient flow, different algorithms with different convergence properties arise. In this paper, we concentrate on the Kullback–Leibler divergence as the energy functional after showing that, up to scaling, it has the \textit{unique} property (among all $f$-divergences) that the gradient flows resulting from this choice of energy do not depend on the normalization constant of the target distribution. For the metrics, we focus on the Fisher-Rao, Wasserstein, Stein metrics and their variants. The Fisher-Rao metric is known to be the \textit{unique} one (up to scaling) that is diffeomorphism invariant, leading to a uniform exponential rate of convergence of the gradient flow to the target distribution. We introduce a relaxed, affine invariance property for the metrics, gradient flows, and their corresponding mean-field models, determine whether a given metric leads to affine invariance, and modify it to make it affine invariant if it does not. 

We study the resulting gradient flows in both the space of all probability density functions and in the subset of all Gaussian densities. The flow in the Gaussian space may be understood as a Gaussian approximation of the flow in the density space. We demonstrate that, under mild assumptions, the Gaussian approximation based on the metric and through moment closure coincide; the moment closure approach is more convenient for calculations. We establish connections between these approximate gradient flows, discuss their relation to natural gradient methods in parametric variational inference, and study their long-time convergence properties showing, for some classes of problems and metrics, the advantages of affine invariance. Furthermore, numerical experiments are included which demonstrate that affine invariant gradient flows have desirable convergence properties for a wide range of highly anisotropic target distributions.

\end{abstract}

\tableofcontents

\section{Introduction}

\subsection{Context} This paper is concerned with the problem of sampling a 
probability distribution (the target) known up to normalization. This problem
is fundamental in many applications arising in computational science and engineering and is
widely studied in applied mathematics, machine learning 
and statistics communities. 
A particular application is Bayesian inference for large-scale inverse problems; such
problems are ubiquitous, arising in applications from climate science~\cite{isaac2015scalable,schneider2017earth,huang2022iterated,lopez2022training}, through numerous problems in engineering~\cite{yuen2010bayesian,cui2016dimension,cao2022bayesian} to machine learning~\cite{rasmussen2003gaussian,murphy2012machine, chen2021solving, chen2021consistency}. These applications have fueled the need for efficient and scalable algorithms which employ noisy data to learn about unknown parameters $\theta$ appearing in models and 
perform uncertainty quantification for predictions then made by those models. 

Mathematically, the objective is to sample the target probability
distribution with density $\rho_{\rm post}(\cdot)$, for the parameter $\theta \in \R^{N_{\theta}}$, given by
\begin{align}
\label{eq:posterior}
    \rho_{\rm post}(\theta) \propto \exp(-\V(\theta)),
\end{align}
where $\V: \R^{N_{\theta}} \to \R_+$ is a known function.
We use the notation $\rho_{\rm post}$ because of the potential application
to Bayesian inference; however, we do not explicitly use the Bayesian
structure in this paper, and our analysis applies to arbitrary target distributions.

We study the use of gradient flows in the space of probability distributions 
in order to create algorithms to sample the target distribution. 
By studying gradient flows
with respect to different metrics, and by studying mean-field based particle models
and Gaussian approximations, both related to these underlying gradient flows,
we provide a unifying approach to the construction of a wide family of algorithms. 
The choice of metric plays a key role in the behavior
of the resulting methods and we highlight the importance of affine invariance
in this regard. In \Cref{ssec:L} we provide a literature review pertinent to our contributions; the contributions we make are described in \Cref{ssec:OC} and in \Cref{ssec:ORG} we describe the organization of the remainder of the paper.

\subsection{Literature Review} 
\label{ssec:L}

In this subsection, we describe the research landscape in which our work sits. We start
by discussing the general background, and describing our work in this context, 
and then we give more detailed literature reviews relating to the topics of gradient flows,
mean-field models, Gaussian approximations, and affine invariance.

\subsubsection{Background}
\label{ssec:bg}

Numerous approaches to the sampling problem
have been proposed in the literature.
One way of classifying them is into: a) methods which deform a {\em given} source measure
(for example the prior in Bayesian inference) into the target measure, in a fixed time
of fixed finite number of steps or in a finite continuous time interval; 
and b) methods which transform {\em any} initial measure into
the target measure after an infinite number of steps, or at time infinity
in continuous time.
Continuous time formulations are used for insight into the algorithms; discrete time
must be used in practice. Typical methods in category a) are sequential Monte Carlo (SMC)
approaches \cite{doucet2009tutorial}, with (typically not optimal) transport 
being the underpinning continuous time concept \cite{villani2009optimal}; 
typical methods in category b) are Markov chain Monte Carlo (MCMC) approaches
\cite{brooks2011handbook}, with stochastic differential equations (SDEs)
which are ergodic with respect to the target,  such as Langevin equations \cite{pavliotis2014stochastic}, 
being the underpinning continuous time concept. 
Making practical algorithms out of these ideas, for large scale problems in
science and engineering, often requires invocation of further reduction of the
space in which solutions are sought, for example by variational inference \cite{blei2017variational}
or by ensemble Kalman approximation \cite{calvello22}.

In this paper we focus primarily on methods in category b) and describe a general
methodology for the derivation of a wide class of sampling algorithms;
however the methods can be interpreted as being partially motivated by
dynamics of transports arising in the methods of type a).
Specifically, we focus on methods created by studying the gradient flow, 
in various metrics, induced by
an energy that measures the divergence of the current estimate of the target from the true
target. With this perspective, we provide a unifying viewpoint on a number of sampling
methods appearing in the literature and a methodology for deriving new methods.
We focus on the Fisher-Rao, Wasserstein, and Stein metrics, and variants thereof.
Creating useful algorithms out of this picture requires further simplifications;
we study mean-field models, which lead to particle methods, and methods based on
confining the gradient descent of the energy to the space of Gaussians. 
In both settings we precisely define the concept of being affine invariant;
roughly speaking this concept requires that any invertible affine transformation of $\theta$  
makes no difference to the gradient flow.
We include numerical experiments which demonstrate the advantage of affine
invariant methods for anisotropic targets. Because the analysis is cleaner
we work in continuous time; but time-discretization is employed to make
implementable algorithms. 
Furthermore, we emphasize that our statements about the existence of
gradient flows are purely formal. For the rigorous underpinnings of
gradient flows see \cite{ambrosio2005gradient}; and for a recent
extension of this rigorous analysis to a sub-class of gradient flows 
with respect to an affine invariant metric see \cite{burger2023covariance}.

\subsubsection{Gradient Flows}

There is existing literature on the use of gradient flows in the probability density space,
employing a variety of different metric tensors, to minimize an energy defined as the Kullback–Leibler (KL) divergence between the current density and the target distribution. Particle realizations of these flows then lead to sampling schemes. For example, the Wasserstein gradient flow~\cite{jordan1998variational,otto2001geometry} and Stein variational gradient flow~\cite{liu2016stein,liu2017stein} have led to sampling algorithms based on
Langevin dynamics and Stein variational gradient descent respectively;
in~\cite{lu2022birth}, the Fisher-Rao gradient flow, using kernel-based density 
approximations, has been proposed for sampling. Furthermore, the paper
\cite{lu2019accelerating} proposed the Wasserstein-Fisher-Rao gradient flow to sample multi-modal distributions. In~\cite{garbuno2020interacting,garbuno2020affine}, the Kalman-Wasserstein metric was introduced and gradient flows with respect to this
metric were advocated. Interpolation between the Wasserstein metric and Stein metric was studied in \cite{he2022regularized}. Accelerated gradient flows in the probability space have been studied in \cite{wang2022accelerated}. A recent overview of the use
of gradient flows in optimization and sampling 
can be found in \cite{garcia2020bayesian,TrillosNoticeAMS}.

The Wasserstein gradient flow was identified in the seminal work~\cite{jordan1998variational}. The
authors showed that the Fokker-Planck equation is the Wasserstein gradient flow of the KL divergence
of the current density estimate from the target. Since then, Wasserstein gradient flow has played a significant role in optimal transport~\cite{santambrogio2017euclidean}, sampling~\cite{chen2018natural,lambert2022variational}, machine learning~\cite{chizat2018global,salimans2018improving}, partial differential equations~\cite{otto2001geometry,carrillo2018analytical} and many other areas. 
The Fisher-Rao metric was introduced by \as{C.R. Rao~\cite{rao1945information}} via the Fisher information matrix. The original definition is in parametric density spaces, and the corresponding Fisher-Rao gradient flow in the parameter space leads to natural gradient descent~\cite{amari1998natural}. 
The Fisher-Rao metric in infinite dimensional probability spaces was discussed in~\cite{FRmetricInfDim1991, srivastava2007riemannian}. The concept underpins information geometry \cite{amari2016information,ay2017information}. The gradient flow of the KL divergence under the Fisher-Rao metric is induced by a mean-field model of birth-death type. The birth-death process has been used in sequential Monte Carlo samplers to reduce the variance of particle weights~\cite{del2006sequential} and to accelerate Langevin sampling~\cite{lu2019accelerating,lu2022birth}. 
The discovery of the Stein metric~\cite{liu2017stein} follows the introduction of the Stein variational gradient descent algorithm~\cite{liu2016stein}. The study of the Stein gradient flow~\cite{liu2017stein,lu2019scaling,duncan2019geometry} sheds light on the analysis and improvements of the algorithm~\cite{detommaso2018stein, wang2020information, wang2022accelerated}.

\subsubsection{Mean-Field Models}

It is natural to ask which evolution equations in state space $\R^{N_{\theta}}$ give rise
to a given gradient flow in the space of probability measures. Continuous time
linear Markov processes with continuous sample paths are limited to It\^{o} SDEs (or equivalent models written in terms of Stratonovich or other stochastic integrals) \cite{pavliotis2014stochastic}. It is thus natural to seek mean-field models in the form of It\^{o} SDEs which depend on their
own density, which is therefore governed by a nonlinear Fokker-Planck equation.
Examples of particle models giving rise to linear and nonlinear Fokker-Planck
equations with gradient structure include Langevin dynamics~\cite{jordan1998variational,otto2001geometry} and Stein variational gradient descent~\cite{liu2016stein,liu2017stein} for sampling the Wasserstein gradient flow and the Stein variational gradient flow respectively. It is also of potential interest to go beyond
It\^{o} SDEs and include Levy (jump) processes \cite{bertoin1996levy,applebaum2009levy}, as well as birth-death models \cite{karlin1957classification}. 
Finally, we mention mean-field models for ensemble Kalman type algorithms \cite{calvello22}; these typically do not have a law which is a gradient flow in the space of probability measures, except
in the linear-Gaussian setting. These mean-field models combine gradient flow, Gaussian approximations, and mean-field equations \cite{ding2021ensemble,blomker2019well,calvello22}.

In practice, mean-field models are approximated by interacting particle systems \cite{ito1996diffusion}, in which
integration against the density is replaced by integration against the empirical
measure of the particle system. At the level of the nonlinear Markov process for the density
on $\R^{N_{\theta}}$ defined by the mean-field model,
this corresponds to approximation by a linear Markov process
for the density on $\R^{JN_{\theta}}$, where $J$ is the number of particles; the concepts of exchangeability
and propagation of chaos may be used to relate the two Markov processes. See~\cite{mckean1967propagation,sznitman1991topics,chaintron2021propagation} and the references therein.

%

\subsubsection{Gaussian Approximations}

There is substantial work on the use of gradient flows in the space of Gaussian, or other parametric density spaces, to minimize the Kullback–Leibler (KL) divergence~\cite{wainwright2008graphical,blei2017variational}.
These methods, in the Gaussian setting, aim to solve the problem 
\begin{align}
\label{eq:GVI}
    (m^\star, C^\star) = \argmin_{m, C}~\mathrm{KL}[\N(m, C) \Vert  \rho_{\rm post}].
\end{align} 
Again, but now restricted to variations in the space of Gaussian densities, different metric tensors lead to different gradient flows to identify $(m^\star, C^\star).$
Recently, the work \cite{lambert2022variational} proved the global convergence of the Wasserstein natural gradient descent algorithm when the posterior is log-concave.
Other work on the use of Gaussian variational inference methods includes the papers \cite{opper2009variational,quiroz2018gaussian,khan2017conjugate,lin2019fast,galy2021particle, yumei2022variational}.

In addition to their role in parametric variational inference, Gaussian approximations have been widely deployed in various generalizations of Kalman filtering~\cite{kalman1960new,sorenson1985kalman,julier1995new,wan2000unscented,evensen1994sequential}.  For Bayesian inverse problems, iterative ensemble Kalman samplers have been proposed which are in category a) defined in subsection \ref{ssec:bg} \cite{emerick2013investigation,chen2012ensemble,wan2000unscented}. 
The paper \cite{huang2022efficient} introduced an ensemble Kalman methodology falling in category b),  defined in subsection \ref{ssec:bg}, based on a novel mean-field dynamical system  that depends on its own filtering distribution.
For all these algorithms based on a Gaussian ansatz, the accuracy depends on some measure
of being close to Gaussian.
Regarding the use of Gaussian approximations in Kalman inversion we highlight, in addition to
the approximate Bayesian methods already cited, the use of ensemble Kalman methods for optimization: see \cite{iglesias2013ensemble,chada2022convergence,kim2022hierarchical,huang2022iterated,weissmann2021adaptive}. Kalman filtering has also been used in combination with variational inference~\cite{lambert2022continuous}. 
The relation between iterative Kalman filtering and Gauss-Newton or Levenberg Marquardt algorithms are studied in~\cite{bell1993iterated,bell1994iterated,huang2022iterated,chada2020iterative},
and leads to ensemble Kalman based optimization methods which are affine invariant.

\subsubsection{Affine Invariance}

The idea of affine invariance was introduced
for MCMC methods in \cite{goodman2010ensemble,foreman2013emcee},  motivated by the empirical success
of the Nelder-Mead simplex algorithm \cite{nelder1965simplex} in optimization. 
Sampling methods with the affine invariance property can be effective for highly anisotropic distributions; this is because they behave identically in all coordinate systems related through an affine transformation; 
in particular, they can be understood by studying the best possible coordinate system, which
reduces anisotropy to the maximum extent possible within the class of affine transformations. The numerical studies presented in~\cite{goodman2010ensemble} demonstrate that affine-invariant MCMC methods offer significant performance improvements over standard MCMC methods. 
This idea has been further developed to enhance sampling algorithms in more general contexts. Preconditioning strategies for Langevin dynamics to achieve affine-invariance were discussed in \cite{leimkuhler2018ensemble}. 
And in~\cite{garbuno2020interacting}, the Kalman-Wasserstein 
metric was introduced, gradient flows in this metric were advocated and in
\cite{garbuno2020affine} the methodology was shown
to achieve affine invariance. Moreover, the authors in \cite{garbuno2020interacting,garbuno2020affine,pidstrigach2021affine} used the empirical covariance of an interacting particle approximation of the mean-field
limit, leading to a family of derivative-free sampling approaches in continuous time.
Similarly, the work \cite{liu2022second} employed the empirical covariance  to precondition second order Langevin dynamics.  
Affine invariant samplers can also be combined with the pCN (preconditioned Crank–Nicolson) MCMC method~\cite{cotter2013mcmc}, to boost the performance of MCMC in function space~\cite{coullon2021ensemble,dunlop2022gradient}.
Another family of affine-invariant sampling algorithms is based on Newton or Gauss-Newton, since 
the use of the Hessian matrix as the preconditioner in Newton's method induces the affine invariance property. 
Such methods include stochastic Newton MCMC~\cite{martin2012stochastic} and the Newton flow with different metrics~\cite{detommaso2018stein, wang2020information}.

\subsection{Our Contributions} 
\label{ssec:OC}

The primary contributions of the work are as follows:

\vspace{0.1in}

\begin{itemize}

\item we highlight a general methodology for the design of algorithms to sample
a target probability distribution known up to normalization, unifying and generalizing an
emerging scattered literature; the methodology is based on the introduction of gradient flows of the KL divergence between the target density and the time-dependent density that represents the solution of the gradient flow;


\item we justify the choice of the KL divergence as the energy functional by showing that, among all $f$-divergences, it is the unique choice (up to scaling) for which the resulting gradient flow is independent of the normalization constant of the target distribution; 

\item the notion of gradient flow requires a metric and we employ
the Fisher-Rao, Wasserstein and Stein metrics
to provide concrete instantiations of the methodology;

\item to design implementable algorithms from the gradient
flows we discuss the use of particle approximations of mean-field models, 
whose law is governed by the gradient flow,  and restriction of
the gradient flow to a parameterized Gaussian family, which
we show to be equivalent to a moment closure approach;

\item we  define the concept of affine invariant metrics, demonstrate links to
affine invariant mean-field models and variational methods restricted to the set of Gaussians, and describe numerical results highlighting the benefits 
of affine invariant methods;

\item we prove results concerning the long time behavior of the underlying gradient flows, in both  the full and Gaussian density spaces, further highlighting the benefits of affine invariance in some cases.

\end{itemize}

\subsection{Organization}
\label{ssec:ORG}
The remainder of the paper is organized as follows. 
In \Cref{sec:Energy}, we introduce energy functionals in the density space.
In \Cref{sec:GF and AI}, we review the basics of gradient flows in the space of probability density functions, covering Fisher-Rao, Wasserstein and Stein gradient flows, and establishing links to mean-field models. Building upon the notion of diffeomorphism invariance in the Fisher-Rao gradient flow, we propose to adopt
the weaker, but more computationally tractable, notion of affine invariant metrics, leading to affine invariant gradient flows and mean-field models. 
In particular, we introduce affine invariant Wasserstein and Stein gradient flows. Their convergence properties are studied theoretically; some of these results highlight the benefits of affine invariance.
In \Cref{sec:GGF}, we review the basics of Gaussian approximate 
gradient flows. We define different Gaussian approximate gradient flows under the 
aforementioned metrics, computing the dynamics governing the evolution of the mean and covariance,
studying their convergence properties, and again identifying the effects of affine invariance. 
A by-product of our computations is to show that the evolution equations for mean and covariance
can be computed by simple use of moment closure. 
In \Cref{sec:Numerics}, numerical experiments are provided to empirically confirm the theory 
and in particular to demonstrate the effectiveness of the affine invariance property
in designing algorithms for certain classes of problems.
We make concluding remarks in \Cref{sec:conclusion}. Five appendices contain
details of the proofs of the results stated in the main body of the paper.

\section{Energy Functional} 
\label{sec:Energy}


An energy functional in the density space maps a probability density to a real number. An important example of an energy functional is the KL divergence: 
\begin{align}
\label{eq:energy}
\mathcal{E}(\rho) = \mathrm{KL}[\rho \Vert  \rho_{\rm post}]  =\int\rho \log\Bigl(\frac{ \rho}{ \rho_{\rm post}}\Bigr)\,\mathrm{d}\theta,
\end{align}
but we will also discuss other energy functionals in this paper. A key property of any energy functional is that its minimizer is the target distribution $\rho_{\rm post}$; this then suggests the derivation of algorithms to
identify $\rho_{\rm post}$ based on minimization of $\mathcal{E}(\rho)$.

\begin{newremark} 
Gradient flows, as the continuous counterpart of gradient descent algorithms, are typical approaches to minimize the energy functional. Methodologically, to introduce gradient flows, one needs a differential structure in the density space, which then leads to the definition of tangent spaces and metric tensors that determine a gradient flow. In general, the appropriate function spaces over which to define $\mathcal{E}(\cdot)$, its first variation, the tangent spaces, and the metric tensors are very technical. And whilst they may be rigorously determined in specific settings\footnote{The rigorous theory of gradient flows in suitable infinite-dimensional functional spaces and its link with evolutionary PDEs is a long-standing subject; see \cite{ambrosio2005gradient, ambrosio2006gradient} for discussions and a rigorous treatment of gradient flows in metric space.}
we seek to keep such technicalities to a minimum and focus on a formal methodology for deriving algorithms.
For the purposes of understanding the formal methodology that we adopt
it suffices to consider
\begin{equation}
\label{eqn-smooth-positive-densities}
\mathcal{P} = \Bigl\{\rho \in C^{\infty}(\R^{N_\theta}) : \int \rho \mathrm{d}\theta = 1 ,\, \rho > 0\Bigr\}
\end{equation}
as the appropriate space of probability densities; that is, we assume $\rho, \rho_{\rm post} \in \PP$. 

The above choice of $\mathcal{P}$ is useful as it allows us to use formal differential structures under the smooth topology to calculate\footnote{The formal Riemannian geometric calculations in the density space were first proposed by Otto in \cite{otto2001geometry}.  The calculations in the smooth setting are rigorous if $\R^{N_{\theta}}$ is replaced by a compact manifold, as noted in \cite{Lott08}. For rigorous results in general probability space, we refer to \cite{ambrosio2005gradient}.} many gradient flow equations; see \Cref{sec:GF and AI}. Once the gradient flow equation has been identified through the formal calculation, we can use it directly for general probability distributions and study the theoretical and numerical properties of this equation rigorously based on PDE tools;  a collection of such results may be found in \Cref{ssec:gf-theory}.
\end{newremark}


\subsection{KL Divergence is A Special Energy Functional}
In principle, we can use any energy functional, and we are not limited to the KL divergence. Here we discuss some desired properties of energy functionals in the context of sampling. In doing so, we identify the KL divergence as a special choice of energy functional that is favorable in sampling problems.

When minimizing $\mathcal{E}(\rho)$, the first variation plays a central role.
For the choice of KL divergence as the energy functional,
the first variation is formally given by
\begin{align}
\label{eq:vare}
\frac{\delta \mathcal{E}}{\delta \rho} = \log \rho- \log \rho_{\rm post} + \text{constant} ,
\end{align}
where we have used the fact that $(\rho \log \rho)'=\log \rho +1$. Here $\frac{\delta \mathcal{E}}{\delta \rho}$ is defined up to a constant since what really matters is the action integral $\int \frac{\delta \mathcal{E}}{\delta \rho} \sigma$ for a signed measure $\sigma$ satisfying $\int \sigma =0$; for more details about the definition of the first variation see \Cref{sec:GF and AI}.

From the formula \cref{eq:vare} we observe that, for the KL divergence, $\frac{\delta \mathcal{E}}{\delta \rho}$  remains unchanged (up to constants) if we scale $\rho_{\rm post}$ by any positive constant $c >0$, i.e. if we change $\rho_{\rm post}$ to $c\rho_{\rm post}$. This  property eliminates the need to know the normalization constant of $\rho_{\rm post}$ in order to calculate
the first variation. It is common in
Bayesian inference for the normalization to be unknown and 
indeed the fact that MCMC methods do not need the normalization constant 
is central to their widespread use; it is desirable that
the methodology presented here has the same property.

The above property can also be phrased in terms of the energy functional. If we write $\mathcal{E}(\rho;\rho_{\rm post})$, making explicit the dependence on $\rho_{\rm post}$, then the property can be stated as: $\mathcal{E}(\rho;c\rho_{\rm post})-\mathcal{E}(\rho;\rho_{\rm post})$ is independent of $\rho$, for any $c \in (0,\infty)$. 

The following \Cref{prop: KL unique f divergence} shows that this property of the KL divergence is special: among all $f$-divergences with continuously differentiable  $f$ defined on the positive reals
it is the only one to have this property.
Here the $f$-divergence between two continuous
density functions $\rho$ and $\rho_{\rm post}$, positive everywhere, is defined as 
\[D_f[\rho \Vert \rho_{\rm post}] = \int \rho_{\rm post} f\Bigl(\frac{\rho}{\rho_{\rm post}}\Bigr) {\rm d}\theta.\]
For convex $f$ with $f(1) = 0$, Jensen's inequality implies that $D_f[\rho \Vert \rho_{\rm post}] \geq 0$. The KL divergence used in \eqref{eq:energy} corresponds to the choice $f(x)=x\log x.$ In what follows we view this $f$-divergence as a function of the probability density $\rho$, parameterized
by $\rho_{\rm post}$; in particular we observe that this parameter-dependent function of
probability density $\rho$ makes sense if $\rho_{\rm post}$ is simply a positive function: 
it does not need
to be a probability density; we may thus scale $\rho_{\rm post}$ by any positive real.

\begin{theorem}
\label{prop: KL unique f divergence} Assume that $f:(0,\infty) \to \mathbb{R}$ is continuously differentiable and $f(1) = 0$.
    Then the KL divergence is the only $f$-divergence (up to scalar factors) such that $D_f[\rho \Vert c\rho_{\rm post}]-D_f[\rho \Vert \rho_{\rm post}]$ is independent of $\rho \in \mathcal{P}$,
    for any $c \in (0,\infty) $ and 
    for any $\rho_{\rm post} \in \mathcal{P}$.
\end{theorem}
The proof of the theorem can be found in \Cref{sec: Proofs of prop: KL unique f divergence}.

\begin{newremark} 
As a consequence of Theorem \ref{prop: KL unique f divergence},
the gradient flows defined by the energy \eqref{eq:energy} 
do not depend on the normalization 
constant of the posterior, as we will see in the next section. Hence the numerical implementation is more straightforward in comparison with the use
of other divergences or metrics to define the energy $\mathcal{E}(\rho)$. This justifies the choice of the KL divergence as an energy functional for sampling,
and our developments in most of this paper are specific to the energy 
\cref{eq:energy}. However, other energy functionals can be, and are, used for 
constructing gradient flows; for example, the chi-squared divergence~\cite{chewi2020svgd,lindsey2022ensemble}:
\begin{equation}
\label{eq:chi-squared}
    \chi^2(\rho \Vert \rho_{\rm post}) = \int \rho_{\rm post}\Bigl(\frac{\rho}{\rho_{\rm post}} - 1\Bigr)^2 \mathrm{d}\theta
    = \int \frac{\rho^2}{\rho_{\rm post}}\mathrm{d}\theta - 1.
\end{equation}

The normalization constant can appear explicitly in the gradient flow equation for general energy functionals. Additional structures need to be explored to simulate such flows. For example, when the energy functional is the chi-squared divergence, in \cite{chewi2020svgd}, kernelization is used to avoid the normalization constant in the Wasserstein gradient flow. Moreover, in \cite{lindsey2022ensemble} where a modification of the Fisher-Rao metric is used, ensemble methods with birth-death type dynamics are adopted to derive numerical methods; the normalization constant can be absorbed into the birth-death rate.
\end{newremark}

\subsection{Constrained Minimization}
In the context of algorithms, it is also of interest to consider minimization of 
$\mathcal{E}(\cdot)$  given by \cref{eq:energy} over parameterized manifolds
in $\PP$, which leads to parametric variational inference. To illustrate this, we consider the 
manifold of Gaussian densities\footnote{The extension to $C \succeq 0$ may be
relevant for some applications but we work in the simpler, strictly positive covariance, setting
here.} $\PPG \subset \PP$

\begin{subequations}
\label{eq:PPG}
\begin{align}
    \PPG &:= \Bigl\{\rho_{\rhoa}: \rho_{\rhoa}(\theta) = \frac{\exp\bigl({-\frac{1}{2}(\theta - m)^T C^{-1} (\theta - m)}\bigr)}{\sqrt{|2\pi C|}} \textrm{ with } \rhoa =(m,C) \in \mathcal{A}\Bigr\},\\ 
    \mathcal{A}&=
    \Bigl\{(m,C) : m\in\R^{N_\theta},\,C\succ 0\in \R^{N_\theta \times N_\theta}
    \Bigr\};
\end{align}
\end{subequations}
here $|\cdot|$ denotes the determinant when the argument is a matrix.
This definition leads to Gaussian variational inference:
\begin{equation}
\label{eq:Gaussian-KL0}
\min_{\rho \in \PPG} \mathrm{KL}[\rho \Vert  \rho_{\rm post}] .
\end{equation}
Any minimizer $\rho_{\rhoa_\star} = \N(m_{\star},C_{\star})$ satisfies \cite{lambert2022recursive}\footnote{We use $\Hess f(\theta)$ to denote the Hessian matrix 
associated with scalar field $f(\theta)$. In doing so we follow the convention 
in the continuum mechanics literature, noticing that the Hessian operator is formed as the composition
of the gradient acting on a scalar field, followed by the gradient acting 
on a vector field~\cite{gonzalez2008first}. The notation 
$\nabla_{\theta}^2 f(\theta)$ is used by some authors to denote the Hessian;
we avoid this because of potential confusion with its useage in some fields 
to denote the Laplacian (trace of the Hessian).}
\begin{align}
\label{eq:Gaussian-KL}
\E_{\rho_{\rhoa_\star}}\bigl[\nabla_\theta  \log \rho_{\rm post}(\theta)   \bigr] = 0 \quad \textrm{and}
\quad C_{\star}^{-1} = -\E_{\rho_{\rhoa_\star}}\bigl[\Hess\log \rho_{\rm post}(\theta)\bigr] .
\end{align}

\section{Gradient Flow}
\label{sec:GF and AI}
In this section, we start by introducing the concept of metric, the gradient flow of energy
\cref{eq:energy} that it induces, the related mean-field dynamics in the state space $\R^{N_{\theta}}$, and 
the concept of affine invariance, all in \Cref{sec:gradient-flow}. Then, in subsequent subsections, we introduce the Fisher-Rao (\Cref{ssec:FR}), Wasserstein (\Cref{ssec:Wasserstein}) and Stein gradient flows (\Cref{ssec:Stein}), together with affine invariant modifications when relevant.  
In fact the Fisher-Rao metric has a stronger invariance property: it is invariant under any diffeomorphism of the parameter space. Finally, we discuss the convergence properties of these gradient flows in \Cref{ssec:gf-theory}.

\subsection{Basics of Gradient Flows}
\label{sec:gradient-flow}
In this subsection, we introduce gradient flows in the probability space and affine invariance in this context. Our focus is on formal calculations to derive these flows. We do not focus on the rigorous analytical underpinnings of gradient flows in a metric space; the reader interested in
further details should consult \cite{ambrosio2005gradient}.

\subsubsection{Metric}
\label{ssec:m}
Recall that the density space we consider is the manifold $\PP$ of smooth 
strictly positive densities \cref{{eqn-smooth-positive-densities}}. At any $\rho \in \PP$, the tangent 
space of $\PP$ satisfies\footnote{The inclusion becomes equality if $\R^{N_{\theta}}$ is replaced by a compact manifold; see related analysis in \cite{Lott08}.}
\begin{align}
    T_\rho \PP \subseteq \Bigl\{\sigma \in C^{\infty}(\R^{N_\theta}): \int \sigma \mathrm{d}\theta = 0\Bigr\}.
\end{align}
The cotangent space $T_\rho^* \PP$ is the dual of $T_
\rho\PP$, which can be identified as a subset of distributions on $\Omega$; see \cite[Section 7]{semmesintroduction}.
We can introduce a bilinear map $\langle\cdot,\cdot\rangle$ as the dual pairing
$T_\rho^* \PP \times T_\rho \PP \rightarrow \R$.
For any $\psi \in T^*_{\rho}\PP$ and $\sigma \in T_{\rho}\PP$, if the distribution $\psi$ is a classical function,
the duality pairing $\langle\cdot,\cdot\rangle$ between $T^*_{\rho}\PP$ and $T_{\rho}\PP$ can be identified in terms of $L^2$ integration:
$\langle\psi,\sigma \rangle = \int \psi \sigma {\rm d}\theta$.



Given a metric tensor at $\rho$,  denoted by $M(\rho): T_{\rho}\PP \rightarrow T_{\rho}^{*}\PP$, we may define the Riemannian metric $g_{\rho}: T_\rho \PP \times T_\rho \PP \to \R$ via $g_{\rho}(\sigma_1, \sigma_2) = \langle M(\rho)\sigma_1, \sigma_2\rangle$. 
The symmetric property of the Rimannian metric $g_{\rho}$ implies that $\langle M(\rho)\sigma_1, \sigma_2\rangle = \langle M(\rho)\sigma_2, \sigma_1\rangle$. 
The inverse of $M(\rho)$, denoted by $M(\rho)^{-1}: T_{\rho}^{*}\PP \rightarrow T_{\rho}\PP$, is sometimes referred to as the
Onsager operator~\cite{onsager1931reciprocal,onsager1931reciprocal-II,mielke2016generalization}.


The geodesic distance 
$\calD: \PP \times \PP \to \R^+$ under metric $g$ is defined by the formula
\begin{equation}
\label{eqn: geodesic distance}
\calD(\rho_A, \rho_B)^2 = \inf_{\rho_t} \Bigl\{\int_0^1 g_{\rho_t}(\partial_t{\rho}_t, \partial_t{\rho}_t) \mathrm{d}t: \,\rho_{0} = \rho_A, \rho_{1} = \rho_B\Bigr\}.
\end{equation}
Here $\rho_t$ is a smooth curve in $\PP$ with respect to $t$.
The distance $\mathcal{D}$ defines a metric on probability densities; however,
to avoid confusion with the Riemannian metric $g$, in this paper we always refer to $\mathcal{D}$ as a distance.

We also recall that the geodesic distance $\calD$ has the following property \cite{do1992riemannian}:
\begin{equation}
    \label{eq:geo}
    \lim_{\epsilon \to 0}\frac{1}{\epsilon^2}\calD(\rho+\epsilon \sigma,\rho)^2=g_{\rho}(\sigma,\sigma) = \langle M(\rho)\sigma, \sigma \rangle.
\end{equation}

\subsubsection{Flow Equation}
\label{ssec:FE}
Recall that the first variation of $\mathcal{E}(\rho)$, denoted by $\frac{\delta \mathcal{E}}{\delta \rho} \in T_{\rho}^*\PP$, is defined by
\begin{align*}
\Bigl\langle \frac{\delta \mathcal{E}}{\delta \rho}, \sigma \Bigr\rangle
=  \lim_{\epsilon \rightarrow 0} \frac{\mathcal{E}(\rho + \epsilon \sigma) - \mathcal{E}(\rho)}{\epsilon},
\end{align*}
for any $\sigma \in T_\rho \PP$. 
The gradient of $\mathcal{E}$ under the Riemannian metric, denoted by $\nabla \mathcal{E}$, is defined via the condition 
$$\forall \sigma \in T_\rho \PP\qquad g_{\rho}(\nabla \mathcal{E}, \sigma) = 
\Bigl\langle \frac{\delta \mathcal{E}}{\delta \rho},\sigma \Bigr\rangle.$$ 
Using the metric tensor, we can write $\nabla \mathcal{E}(\rho) = M(\rho)^{-1} \frac{\delta \mathcal{E}}{\delta \rho}$.

The gradient flow of $\mathcal{E}$ with respect to this metric is thus defined by 
\begin{equation}
\label{eq:GF}
    \frac{\partial \rho_t}{\partial t}  = -\nabla \mathcal{E}(\rho_t)=-M(\rho_t)^{-1}\frac{\delta \mathcal{E}}{\delta \rho}\Bigr|_{\rho = \rho_t}, 
\end{equation}
in which the right hand side is an element in $T_{\rho_t} \PP.$

\begin{newremark}
    The gradient flow can also be interpreted from the proximal perspective.
Given the metric $g$ and the geodesic distance function under this metric, $\calD$,  the proximal point method uses the following iteration
\begin{align}
\label{eq:prox0}
    \rrho_{{n+1}} = \argmin_{\rho \in \PP}\Bigl(\EE(\rho) + \frac{1}{2\Delta t}\calD(\rho, \rrho_{n})^2\Bigr)
\end{align}
to minimize the energy functional $\EE$ in  density  space $\PP$. When $\Delta t$ is small
it is natural to seek $\rrho_{n+1}=\rrho_n+\Delta t\sigma_n$ and note that, invoking the
approximation implied by \cref{eq:geo}, 
\begin{align*}
    \sigma_{{n}} &\approx \argmin_{\sigma \in T_{\rrho_n} \PP}\Bigl(\EE(\rrho_n)+\Delta t\Bigl\langle  \frac{\delta \mathcal{E}}{\delta \rho}\Bigr|_{\rho = \rrho_n}, \sigma\Bigr\rangle + \frac{1}{2}{\Delta t}\Bigl\langle M(\rrho_n)\sigma, \sigma \Bigr\rangle \Bigr).
\end{align*}
To leading order in $\Delta t$, this expression is minimized by choosing
$$\sigma_n \approx -M(\rrho_n)^{-1} \frac{\delta \mathcal{E}}{\delta \rho}\Bigr|_{\rho = \rrho_n}. $$  
Letting
$\rrho_n \approx \rho_{n\Delta t}$, the formal continuous time limit of the proximal
algorithm leads to the corresponding gradient flow~\cref{eq:GF}~\cite{jordan1998variational}.
\end{newremark}


\subsubsection{Affine Invariance}
We now introduce the concept of affine invariance.
The concept of affine invariance in sampling is first introduced for MCMC methods in \cite{christen2010general,goodman2010ensemble},  motivated by the attribution of the empirical success
of the Nelder-Mead algorithm \cite{nelder1965simplex}
for optimization to a similar property; further development of the method in the context of sampling
algorithms is discussed in~\cite{leimkuhler2018ensemble,garbuno2020affine,pidstrigach2023affine}. 
Moreover, Newton's method for optimization exhibits affine invariance, which inspired a diverse range of affine invariant samplers, such as Mirror-Langevin process~\cite{hsieh2018mirrored,zhang2020wasserstein,chewi2020exponential}.
However, the concept of affine invariance in the context of gradient flow has not been systematically explored or discussed. Roughly speaking, affine invariant gradient flows are invariant under any invertible affine transformations of the density variables;
as a consequence, the convergence rate is independent of the affine transformation. It is
thus natural to expect that algorithms with this property have an advantage for sampling highly anisotropic posteriors.

Let $\varphi:\theta \rightarrow \tilde\theta$ denote a diffeomorphism in $\R^{N_{\theta}}$. When $\varphi(\theta) = A\theta + b$, $A \in \R^{N_{\theta}\times N_{\theta}}, b \in \R^{N_{\theta}}$ and $A$ is invertible, the diffeomorphism is an affine transformation. 

\begin{definition} 
\label{def:pushforward}
We define the {\em pushforward operation} $\#$ for various objects as follows:
\begin{itemize}
    \item for density $\rho$, we write $\tilde{\rho} = \varphi \# \rho$, which satisfies $\tilde{\rho}(\tilde\theta) = \rho(\varphi^{-1}(\tilde{\theta}))|\nabla_{\tilde{\theta}} \varphi^{-1}(\tilde{\theta})|$;
    \item for tangent vector $\sigma \in T_\rho \PP$, we have $\tilde{\sigma} = \varphi\#\sigma \in T_{\tilde{\rho}} \PP$ which satisfies \[\tilde{\sigma}(\tilde\theta) = \sigma(\varphi^{-1}(\tilde{\theta}))|\nabla_{\tilde{\theta}} \varphi^{-1}(\tilde{\theta})|;\] 
    \item for functional $\mathcal{E}$ on $\PP$, we define $\tilde{\EE} = \varphi\#\EE$ via $\tilde{\EE}(\tilde{\rho}) = \EE(\varphi^{-1}\# \tilde{\rho})$.
\end{itemize}
\end{definition}
We note that the pushforward operation $\#$ is defined for general measures through duality. More precisely consider probability measures $\mu, \nu$ in $\R^{N_{\theta}}$. Then $\nu = \varphi \# \mu$ if and only if
\[ \int f(\theta) {\rm d}\nu = \int f(\varphi(\theta)) {\rm d}\mu, \]
for any integrable $f$ under measure $\nu$. If $\mu, \nu$ admit densities $\rho$ and $\tilde{\rho}$ respectively, we can use the change-of-variable formula to derive the forms of $\tilde{\rho}$ in \Cref{def:pushforward}. Many results in this paper involving affine transformations may be proved alternatively using the definition of pushforward via duality; by adopting this approach the arguments could be extended to general measures, rather
than those with smooth Lebesgue density. However, in the present study we consider probability densities $\rho \in \PP$, for which the pushforward defined through \Cref{def:pushforward} is convenient.


Now we can define affine invariant gradient flow, metric, and mean-field dynamics.

\begin{definition}[Affine Invariant Gradient Flow] 
\label{def: Affine Invariant Gradient Flow} Fix a Riemannian metric $g$ and the gradient
operation $\nabla$ with respect to this metric. 
Consider the gradient flow
\[\frac{\partial \rho_t}{\partial t} = -\nabla \mathcal{E}(\rho_t)\] and the affine transformation $\tilde\theta = \varphi(\theta) = A\theta + b$. Let $\tilde{\rho}_t  :=\varphi \# \rho_t$ denote the distribution of $\tilde{\theta}$ at time $t$ and set $\tilde{\mathcal{E}} = \varphi\#\EE.$ 
The {\em gradient flow is affine invariant} if 
\[\frac{\partial \tilde{\rho}_t}{\partial t} = -\nabla \tilde{\mathcal{E}}(\tilde{\rho}_t),\]
for any invertible affine transformation $\varphi$.
\end{definition}

The key idea in the preceding definition is that, after the change of variables, the dynamics of $\tilde{\rho}_t$ is itself a gradient flow, in the same metric as the gradient flow 
in the original variables.

\begin{definition}[Affine Invariant Metric]
\label{def: Affine Invariant Metric}
 Define the {\em pull-back operator on Riemannian metric $g$} by 
    \[ (\varphi^{\#}g)_{\rho} (\sigma_1,\sigma_2) = g_{\varphi \# \rho} (\varphi \# \sigma_1,\varphi \# \sigma_2), \]
    for any $\rho \in \PP$ and $\sigma_1, \sigma_2 \in T_\rho \PP$.
We say that {\em Riemannian metric $g$ is affine invariant} if $\varphi^{\#}g = g$ for any affine transformation $\varphi$.  
\end{definition}
The affine invariance of gradient flows is closely related to that of the Riemannian metric:
\begin{proposition}
\label{prop: affine invariant metric tensor}
The following two conditions are equivalent:
\begin{enumerate}
    \item the gradient flow under Riemannian metric $g$ is affine invariant for any $\mathcal{E}$;
    \item the Riemannian metric $g$ is affine invariant.
\end{enumerate}
\end{proposition}
We provide a proof for this proposition in \cref{proof: prop affine invariant metric tensor}. 
Given this, it suffices to focus on the affine invariance of the Riemannian metrics that we 
consider in this paper; furthermore, we may modify them where needed to make them affine invariant. 
\begin{newremark}
    In \cref{prop: affine invariant metric tensor}, we consider the affine invariance property to hold for any $\mathcal{E}$; the metric is \textit{independent} of $\mathcal{E}$. However, it is possible to choose a metric that depends on the energy functional $\mathcal{E}$. An example of this is Newton's method where the Riemannian metric is given by the Hessian of the energy functional, assuming it is positive definite; see the discussion of Newton's flow on probability space in \cite{wang2020information}.
\end{newremark}
\begin{newremark}
\label{remark: transform E and transform posterior}
    Recall that our motivation for introducing the affine invariance property is that algorithms with this property will, in settings where an affine transformation removes anisotropy, have
    favorable performance when sampling highly anisotropic posteriors. Our current definition of affine invariance is tied to the energy functional $\mathcal{E}$ without direct reference to $\rho_{\rm post}$. Given $\mathcal{E}$, an affine invariant gradient flow has the same convergence property when the energy functional changes to $\mathcal{E}(\varphi^{-1}\# \rho)$ where $\varphi$ is an invertible affine transformation.  To connect the transformation of the energy functional to that of $\rho_{\rm post}$, we note that the KL divergence satisfies the property
    \begin{equation}
    \label{eq:Echange}
        \mathcal{E}(\varphi^{-1}\# \rho) = {\rm KL}[\varphi^{-1}\# \rho \Vert \rho_{\rm post}] = {\rm KL}[ \rho \Vert \varphi\# \rho_{\rm post}]. 
    \end{equation}
    Therefore, affine invariant gradient flows of the KL divergence have the same convergence property when $\rho_{\rm post}$ changes to $\varphi\# \rho_{\rm post}$, for any invertible affine transformation $\varphi$. This suggests that the flow will have favorable behavior for sampling highly anisotropic posteriors provided, under at least one affine transformation, the anisotropy is removed.
\end{newremark}

\subsubsection{Mean-Field Dynamics}
\label{sssec:MFD}
Approximating the dynamics implied by \cref{eq:GF} is often a substantial task. One
approach is to identify a mean-field stochastic
dynamical system, with state space $\R^{N_{\theta}}$, 
defined so that its law is given by  \cref{eq:GF}. 
For example, we may introduce the It\^o SDE
\begin{equation}
\begin{aligned}
\label{eq:MFD}
\mathrm{d}\theta_t = f(\theta_t; \rho_t, \rho_{\rm post})\mathrm{d}t + h(\theta_t; \rho_t,\rho_{\rm post}){\rm d}W_t, 
\end{aligned}
\end{equation}
where $W_t\in \R^{N_{\theta}}$ is a standard Brownian
motion. Because the drift $f: \R^{N_{\theta}} \times \PP \times \PP\rightarrow \R^{N_{\theta}}$ and diffusion coefficient $h : \R^{N_{\theta}}\times\PP \times \PP \rightarrow \R^{N_{\theta}\times N_{\theta}}$ are evaluated at $\rho_t$, the density of $\theta_t$ itself, this is a mean-field 
model. While other types of mean-field dynamics, such as the birth-death dynamics, do exist, here we mainly consider the It\^o SDE type dynamics.

The density is governed by a nonlinear Fokker-Planck equation
\begin{equation} 
\label{eq:MFD-2}
 \frac{\partial \rho_t}{\partial t} = -\nabla_\theta \cdot(\rho_t f) + \frac12\nabla_{\theta}\cdot\bigl(\nabla_\theta \cdot(hh^T \rho_t)\bigr).
\end{equation}
By choice of $f, h$ it may be possible to ensure that \cref{eq:MFD-2} coincides with
\cref{eq:GF}. Then an interacting particle system can be used to approximate
\cref{eq:MFD}, generating an empirical measure which approximates $\rho_t.$

As the affine invariance property is important for gradient flows, we also need to study this property for mean-field dynamics that are used to approximate these flows.

\begin{definition}[Affine Invariant Mean-Field Dynamics]
\label{def: affine invariant mean-field equation}
 Consider the mean-field dynamics~\cref{eq:MFD}
and the affine transformation $\tilde\theta = \varphi(\theta) = A\theta + b$. 
The mean-field dynamics is affine invariant, when 
\begin{subequations}
\label{eq:aMFD-AI}
\begin{align}
\begin{split}
\label{eq:aMFD-AI-f}
&Af(\theta; \rho, \rho_{\rm post}) = f(\varphi(\theta); \varphi \# \rho, \varphi \# \rho_{\rm post}), 
\end{split}\\
\begin{split}
\label{eq:aMFD-AI-sigma}
&Ah(\theta; \rho, \rho_{\rm post}) = h(\varphi(\theta); \varphi \# \rho, \varphi \# \rho_{\rm post}),
\end{split}
\end{align}
\end{subequations}
for any affine transformation $\varphi$. This implies that $\tilde{\theta}_t = \varphi(\theta_t)$ satisfies a SDE of the same form as \cref{eq:MFD}:
\begin{equation}
\begin{aligned}
\mathrm{d}\tilde{\theta}_t = f(\tilde{\theta}_t; \tilde{\rho}_t, \tilde{\rho}_{\rm post})\mathrm{d}t + h(\tilde{\theta}_t; \tilde{\rho}_t,\tilde{\rho}_{\rm post}){\rm d}W_t, 
\end{aligned}
\end{equation}
where $\tilde{\rho}_t = \varphi \# \rho_t$ and $\tilde{\rho}_{\rm post} = \varphi \# \rho_{\rm post}$ by definition.
\end{definition}

If we use this definition, then mean-field dynamics of affine invariant gradient flows need not be affine invariant, since there may be different $f, h$ giving rise to the same flow -- equivalence classes. For the affine invariance of the corresponding mean-field dynamics, we have the following proposition, noting that the condition on the energy is satisfied for \cref{eq:energy}
by \cref{eq:Echange}.
\begin{proposition}
\label{prop: affine invariance of mean-field dynamics}
    Consider the energy functional $\mathcal{E}(\rho; \rho_{\rm post})$, making explicit the dependence on $\rho_{\rm post}$, and assume that $\mathcal{E}(\varphi^{-1} \# \rho; \rho_{\rm post}) = \mathcal{E}(\rho; \varphi \#\rho_{\rm post})$ holds. Then, corresponding to any
    affine invariant gradient flow of $\mathcal{E}$  
    there is a mean-field dynamics of the form \eqref{eq:MFD}
    which is affine invariant.
    \end{proposition}
The proof of this proposition may be found in
\cref{appendix: Proof of prop: affine invariance of mean-field dynamics}.

As a consequence, \cref{prop: affine invariance of mean-field dynamics} unifies the affine invariance property of the gradient flow in probability space and the corresponding mean-field dynamics. We note, however, that the mean-field dynamics
is not unique and we only prove the existence of one choice (amongst many)
which is affine invariant. In our later discussions, we will give some specific construction of the mean-field dynamics for several gradient flows, and show that they are indeed affine invariant.

\begin{newremark}
The condition assumed in \cref{prop: affine invariance of mean-field dynamics} indicates that the pushforward of the functional $\EE$ (See \cref{def:pushforward}) satisfies
\begin{equation}
\tilde{\mathcal{E}}(\tilde{\rho})=\tilde{\mathcal{E}}(\tilde{\rho};\rho_{\rm post}) = \mathcal{E}(\varphi^{-1} \# \tilde{\rho};\rho_{\rm post}) = \mathcal{E}(\tilde{\rho};\varphi \# \rho_{\rm post}) = \mathcal{E}(\tilde{\rho};\tilde{\rho}_{\rm post}).
\end{equation}
Thus, this condition allows to connect the affine invariance defined via the transformation of energy functional and via the transformation of the target posterior distributions, as explained in \cref{remark: transform E and transform posterior}. Beyond the KL divergence (see \cref{eq:Echange}), the condition is also satisfied by various widely used energy functionals, such as the Hellinger distance and the chi-squared divergence.
\end{newremark}

\subsection{Fisher-Rao Gradient Flow} 
\label{ssec:FR}

\subsubsection{Metric}

The Fisher-Rao Riemannian metric is
\[g_{\rho}^{\mathrm{FR}}(\sigma_1,\sigma_2)=\int \frac{\sigma_1\sigma_2}{\rho} \mathrm{d}\theta.\]

\begin{newremark}
    Writing tangent vectors on a multiplicative scale, by setting $\sigma=\rho \psi_\sigma$, we
see that this metric may be written as
\[g_{\rho}^{\mathrm{FR}}(\sigma_1,\sigma_2) = \int \psi_{\sigma_1} \psi_{\sigma_2} \rho \mathrm{d}\theta,\]
and hence that in the $\psi_\sigma$ variable 
the metric is described via the $L_\rho^2$ inner-product. That is, the Fisher-Rao Riemannian metric measures the multiplicative factor via the $L_\rho^2$ energy. In \Cref{remark-Wasserstein-as-L2-velocity}, we will see that another important metric, the Wasserstein Riemannian metric, may also be understood as a $L_{\rho}^2$ measurement, but of the velocity field instead.
\end{newremark}

The Fisher-Rao metric tensor $\M^{\mathrm{FR}}(\rho)$ associated to $g_{\rho}^{\mathrm{FR}}$ satisfies\footnote{ Although functions in $T^*_{\rho}\PP$ are not uniquely defined under the $L^2$ inner product, since 
$\langle \psi,\sigma\rangle = \langle \psi+c,\sigma \rangle$ for all $\sigma \in \T_\rho \PP$ and any constant $c$,
a unique representation can be identified by requiring, for example, that $\psi \in T_\rho^*\PP$ satisfies $\E_\rho[\psi]=0$. Under this choice, the Fisher-Rao metric tensor naturally reduces to a multiplication by the density $\rho$: $\M^{\mathrm{FR}}(\rho)^{-1} \psi = \rho\psi, \quad  \forall~\psi \in T_{\rho}^{*}\PP$.}
\begin{subequations}
\label{eq:gm}
\begin{align}
&\M^{\mathrm{FR}}(\rho) \sigma = \psi_{\sigma} - \int \psi_{\sigma} {\rm d}\theta, \quad  \forall~\sigma \in T_{\rho}\PP\\
&\M^{\mathrm{FR}}(\rho)^{-1} \psi = \rho(\psi - \E_\rho[\psi]), \quad  \forall~\psi \in T_{\rho}^{*}\PP.
\end{align}
\end{subequations}
The corresponding geodesic distance $\calD^{\rm FR}: \PP \times \PP \to \R^+$ is 
\begin{equation}
\label{eqn: FR geodesic distance}
\calD^{\rm FR}(\rho_A, \rho_B)^2 = \inf_{\rho_t} \Bigl\{\int_0^1 \mathrm{d}t\int\frac{|\partial_t{\rho}_t|^2}{\rho_t} \mathrm{d}\theta : \,\rho_0 = \rho_A, \rho_1 = \rho_B\Bigr\}.
\end{equation}
If we do not restrict the distributions $\rho_t$ to be on the probability space and we allow them to have any positive mass, then by using the relation $$\frac{|\partial_t{\rho}_t|^2}{\rho_t} = 4\Bigl|\frac{\mathrm{d}}{\mathrm{d}t}\sqrt{\rho_t}\Bigr|^2$$ 
and the Cauchy-Schwarz inequality, we can solve the optimization problem in \cref{eqn: FR geodesic distance} explicitly. The optimal objective value will be $4\int |\sqrt{\rho_A} - \sqrt{\rho_B}|^2 \mathrm{d}\theta.$
This is (up to a constant scaling) the Hellinger distance \cite{gibbs2002choosing}. 

On the other hand, if we constrain $\rho_t$ to be on the probability space, then the geodesic distance will be (up to a constant scaling) the spherical Hellinger distance:
\[\calD^{\rm FR}(\rho_A, \rho_B)^2 = 4 \operatorname{arccos}^2\left(\int \sqrt{\rho_A}\sqrt{\rho_B}{\rm d}\theta\right).\]
For more discussions, see \cite{halder2018gradient, laschos2019geometric, lu2022birth}. In view of this relation, Fisher-Rao gradient flows are sometimes referred to as spherical Hellinger gradient flows in the literature \cite{lindsey2022ensemble,lu2022birth}.

\subsubsection{Flow Equation}
From \eqref{eq:vare} and (\ref{eq:gm}b) we see that
the Fisher-Rao gradient flow of the KL divergence is
\begin{equation}
\begin{aligned}    
\label{eq:mean-field-Fisher-Rao}
\frac{\partial \rho_t}{\partial t} =& -\M^{\mathrm{FR}}(\rho_t)^{-1}\frac{\delta \mathcal{E}}{\delta \rho}\Bigr|_{\rho=\rho_t}, \\
=&  \rho_t \bigl(\log \rho_{\rm post} - \log \rho_t\bigr) - \rho_t\E_{\rho_t}[\log \rho_{\rm post} - \log \rho_t].
\end{aligned}
\end{equation}
\begin{newremark}
The gradient flow \cref{eq:mean-field-Fisher-Rao} in probability space has the form typical 
of a mean-field model which is a birth-death process -- it is possible to create and kill particles to sample this process. However, the support of the empirical distribution using this algorithm never increases during evolution. To address this issue, the work~\cite{lu2019accelerating} added Langevin diffusion to the birth-death process, resulting in what they term  the Wasserstein-Fisher-Rao gradient flow. Alternatively, the authors in ~\cite{lindsey2022ensemble} utilized a Markov chain kernel and MCMC to sample the birth death dynamics arising from the Fisher-Rao gradient flow, using the chi-squared divergence~\cite{liese2006divergences} instead of \cref{eq:energy}.
\end{newremark}

\begin{newremark} \label{rem:homotopy}
When the target distribution \eqref{eq:posterior} arises from a Bayesian
inverse problem it may be written in the form
\begin{align}
\label{eq:posterior1}
    \rho_{\rm post}(\theta) \propto \exp(-\VL(\theta))\rho_0(\theta);
\end{align}
function $\VL: \R^{N_{\theta}} \to \R_+$ is the negative log likelihood and $\rho_0$ is the prior.
In this context, it is interesting to consider 
\begin{align}
\label{eq:Bayesian_U}
\mathcal{E}(\rho) = \int \rho(\theta) \,\VL(\theta)\, \mathrm{d}\theta
\end{align}
with associated Fisher-Rao gradient flow
\begin{align}
    \frac{\partial \rho_t}{\partial t} = -\rho_t \left( \VL - \E_{\rho_t}[\VL]\right).
\end{align}
It may be shown that the density $\rho_t$ is explicitly given by
\begin{align}
    \rho_t (\theta) = \frac{\exp(-t\VL(\theta))\rho_0(\theta)}{\E_{\rho_0}[\exp(-t\VL)]}.
\end{align}
Hence we recover (\ref{eq:posterior1}) at $t=1$. This observation is at the heart of homotopy-based approaches to Bayesian inference \cite{del2006sequential}, leading to methods based on particle filters; the link to
an evolution equation for $\rho_t$ is employed and made explicit in various other approaches to filtering \cite{daumetal2010,reich2011dynamical}. See \cite{chopin2020introduction,calvello22} for overviews. Such Fisher-Rao gradient flow structure for Bayes updates has also been identified in the context of filtering in \cite{laugesen2015poisson, halder2018gradient, halder2017gradient}.

We also note that by letting $t\to \infty$ one finds that 
\begin{align} 
    \lim_{t\to \infty}\rho_t = \delta_{\theta^\ast}
\end{align}
where $\theta^\ast$ denotes the (assumed unique) minimizer of $\VL$, in the support of $\rho_0$, 
and $\delta_{\theta^\ast}$ denotes the Dirac delta function centred at $\theta^\ast$.
\end{newremark}

\subsubsection{Affine Invariance} 
\label{ssec:FR-AI}
The Fisher-Rao metric is affine invariant. One may understand this property through the affine invariance property of Newton's method when the energy functional is the KL divergence. To see this
note that, from \cref{eq:vare}, the Hessian of $\mathcal{E}$ given by \cref{eq:energy} has the
form
\begin{align}
 \frac{\delta^2 \mathcal{E}(\rho)}{\delta\rho^2}=\frac{\delta^2 \mathrm{KL}[\rho\Vert \rho_{\rm post}]}{\delta\rho^2} = \frac{1}{\rho} = \M^{\text{FR}}(\rho).
\end{align}
Therefore, the Fisher-Rao gradient flow of the KL divergence behaves like Newton's method, which is affine invariant.
In fact, the Fisher-Rao metric is invariant under \textit{any diffeomorphism} of the parameter space, not just invertible affine transformations. Indeed, it is the only metric, up to constant, that satisfies this strong invariance property \cite{cencov2000statistical, ay2015information, bauer2016uniqueness}. This diffeomorphism invariance implies that the convergence property of the Fisher-Rao gradient flows for general target densities is the same as for Gaussian target distributions. This intuition explains why the Fisher-Rao gradient flows have an exceptional uniform exponential convergence rate for general target distributions; see \Cref{ssec:gf-theory}.

\subsubsection{Mean-Field Dynamics} 
The Fisher-Rao gradient flow (\ref{eq:mean-field-Fisher-Rao}) in $\rho_t$ can be realized as 
the law of a mean-field ordinary differential equation in $\theta_t$ 
\begin{equation}
\label{eq:Fisher-Rao-MD}
    \frac{\mathrm{d}\theta_t}{\mathrm{d}t} = f(\theta_t; \rho_t, \rho_{\rm post}).
\end{equation}
Writing the nonlinear Liouville equation associated with this model and equating it to
\cref{eq:mean-field-Fisher-Rao} shows that drift $f$ satisfies 
\begin{equation} \label{eq:Fisher-Rao-drift}
-\nabla_\theta \cdot (\rho_t f) = 
\rho_t \bigl(\log \rho_{\rm post} - \log \rho_t\bigr) - \rho_t\E_{\rho_t}[\log \rho_{\rm post} - \log \rho_t].
\end{equation}
Note that $f$ is not uniquely determined by (\ref{eq:Fisher-Rao-drift}). Writing $f$
as a gradient of a potential, with respect to $\theta$, shows that the potential satisfies a linear elliptic
PDE, and under some conditions this will have a unique solution; but there will be other
choices of $f$ which are not a pure gradient, leading to nonuniqueness. 

By \cref{prop: affine invariance of mean-field dynamics}, for affine invariant gradient flows, there exist mean-field dynamics (i.e., via choosing certain $f$ in \eqref{eq:Fisher-Rao-drift}) that are affine invariant. Here, we construct a specific class of $f$ that leads to affine invariant mean-field dynamics for the Fisher-Rao gradient flow.

First, we introduce a matrix valued function: 
$\Prec: \R^{N_\theta}\times \PP \to \R^{N_\theta\times N_\theta}_{\succ 0},$
where the output space is the cone of positive-definite symmetric matrices; we refer to matrices such as $\Prec$ as \emph{preconditioners} throughout this paper. Then, the following proposition shows that the choice of $f= \Prec(\theta, \rho_t)\nabla \phi$ leads to affine invariance of the dynamics, under certain conditions on $P$. The proof can be found in \cref{proof:AI-Fisher-Rao-MD}.

\begin{proposition}
\label{proposition:AI-Fisher-Rao-MD}
Consider any invertible affine transformation $\tilde\theta = \varphi(\theta) = A \theta + b$ and correspondingly $\tilde{\rho} = \varphi \# \rho$. Assume that the preconditioning matrix satisfies
\begin{align} 
\label{eq:P-affine}
\Prec(\tilde\theta, \tilde\rho) = A \Prec(\theta, \rho) A^T.
\end{align}
Assume, furthermore, that the solution 
$\phi(\theta; \rho, \rho_{\rm post})$ of the equation 
\begin{equation}
\label{eqn: affine invariant Fisher Rao mean field dynamics}
-\nabla_\theta \cdot (\rho \Prec \nabla_{\theta} \phi) = 
\rho \bigl(\log \rho_{\rm post} - \log \rho\bigr) - \rho\E_{\rho}[\log \rho_{\rm post} - \log \rho]
\end{equation}
exists, is unique (up to constants) and belongs to $C^2(\mathbb{R}^{N_{\theta}})$, 
for any $\rho \in \mathcal{P}$. Then, the corresponding mean-field equation \cref{eq:Fisher-Rao-MD} with $f = \Prec \nabla_{\theta} \phi$ is affine invariant.
\end{proposition}

\begin{newremark}
    Examples of preconditioning matrices that satisfy \cref{eq:P-affine} include the covariance matrix $P(\theta,\rho) = C(\rho)$ and some local preconditioners,  such as \[P(\theta,\rho) = \bigl(\theta - m(\rho)\bigr)\bigl(\theta - m(\rho)\bigr)^T\] or \[P(\theta,\rho) = \int (\theta' - m(\rho))(\theta' - m(\rho))^T\kappa(\theta,\theta', \rho) \rho(\theta') d\theta'.\]
    Here $\kappa: \R^{N_{\theta}}\times\R^{N_{\theta}}\times\PP \rightarrow \R$ is a positive definite kernel for any fixed $\rho$, and it is affine invariant, namely $\kappa(\tilde\theta, \tilde\theta', \tilde\rho) = \kappa(\theta, \theta', \rho)$ 
    under any invertible affine transformation $\tilde\theta = \varphi(\theta) = A \theta + b$ and correspondingly $\tilde{\rho} = \varphi \# \rho$. A potential choice is $$\kappa(\theta,\theta', \rho) = \exp\bigl\{-\frac{1}{2}(\theta - \theta')^TC(\rho)^{-1}(\theta - \theta')\bigr\}.$$
\end{newremark}
\begin{newremark}
More generally, given any alternative functional $\mathcal{E}$, such as 
(\ref{eq:Bayesian_U}), one can define affine invariant mean-field ordinary differential equations of the form (\ref{eq:Fisher-Rao-MD}) with drift
$f = \Prec \nabla_{\theta} \phi$ and potential $\phi$ satisfying the equation
\begin{align} \label{eq:Fisher-Rao-Liouville}
    -\nabla_\theta \cdot (\rho \Prec \nabla_{\theta} \phi) = 
\rho \Bigl( \frac{\delta \mathcal{E}}{\delta \rho}  - \E_{\rho}\Bigl[
\frac{\delta \mathcal{E}}{\delta \rho}\Bigr] \Bigr).
\end{align}
\end{newremark}
In addition to the above choice of mean field models, birth-death type mean field dynamics have also been used to simulate Fisher-Rao gradient flows for sampling; see \cite{lu2019accelerating,lu2022birth}.

\subsection{Wasserstein Gradient Flow} 
\label{ssec:Wasserstein}
\subsubsection{Metric}
Generalizing the relationship between $\sigma$ and $\psi_\sigma$ introduced
in the Fisher-Rao context, we define 
$\psi_\sigma$ to be the solution of the PDE
\begin{equation}
\label{eq:liouville}
-\nabla_{\theta} \cdot (\rho \nabla_{\theta} \psi_\sigma)=\sigma.
\end{equation}
This definition requires specification of function spaces to ensure 
unique invertibility
of the divergence form elliptic operator.
One then defines the Wasserstein metric tensor $\M^{\mathrm{W}}(\rho)$ 
and its inverse by
\begin{subequations}
\label{eq:gm2}
\begin{align}
&\M^{\mathrm{W}}(\rho) \sigma = \psi_\sigma \quad \forall~\sigma \in T_{\rho}\PP,\\
&\M^{\mathrm{W}}(\rho)^{-1} \psi = -\nabla_\theta \cdot (\rho \nabla_\theta \psi) , \quad  \forall~\psi \in T_{\rho}^{*}\PP.
\end{align}
\end{subequations}
Elementary manipulations show that the corresponding Riemannian metric
is given by
\begin{subequations}
\label{eq:Wasserstein-metric}
\begin{align}
g_{\rho}^{\mathrm{W}}(\sigma_1,\sigma_2) &= \langle \M^{\mathrm{W}}(\rho)\sigma_1, 
 \sigma_2 \rangle\\
&=\langle  \psi_{\sigma_1}, \M^{\mathrm{W}}(\rho)^{-1}\psi_{\sigma_2} \rangle\\
&= \int \rho(\theta) \nabla_{\theta} \psi_{\sigma_1}(\theta)^T  \nabla_{\theta} 
\psi_{\sigma_2}(\theta) \mathrm{d}\theta.
\end{align}
\end{subequations}
Here $g_{\rho}^{\mathrm{W}}$ is positive-definite and hence a valid metric. It is termed the Wasserstein Riemannian metric throughout this paper.

\begin{newremark}
\label{remark-Wasserstein-as-L2-velocity}
    The Wasserstein Riemannian metric has a transport interpretation. To understand
    this fix $\sigma \in T_\rho \PP$ and consider the family of \emph{velocity
    fields} $v$ related to $\sigma$ via the constraint $\sigma = -\nabla_\theta \cdot (\rho v)$. Then define $v_{\sigma} = \argmin_v \int \rho |v|^2$ in which
    the minimization is over all $v$ satisfying the constraint. A formal Lagrange multiplier argument can be used to deduce that $v_{\sigma} = \nabla_\theta \psi_{\sigma}$ for some $\psi_{\sigma}$. This motivates the relationship
    appearing in \cref{eq:liouville} as well as the form of the
Wasserstein Riemannian metric appearing in \cref{eq:Wasserstein-metric}
which may then be viewed as measuring the kinetic energy $\int \rho |v_{\sigma}|^2 \mathrm{d}\theta$. We emphasize that, for ease of understanding, our discussion on the Riemannian structure of the Wasserstein metric is purely formal; for rigorous treatments, the reader can consult \cite{ambrosio2005gradient}.
\end{newremark}

To further develop the preceding discussion, consider the Liouville equation for the dynamical system in $\R^{N_{\theta}}$ driven by vector field $v_{\sigma} := \nabla_\theta \psi_{\sigma}$.
Let $\rho_A, \rho_B$ be two elements in $\PP$ and let $\rho_t$ be a path in time
governed by this Liouville equation, and satisfying the boundary conditions
$\rho_0 = \rho_A, \rho_1 = \rho_B.$ Then 
\begin{equation}
\label{eq:Leq}
\frac{\partial \rho_t}{\partial t} + \nabla_\theta\cdot(\rho_t \nabla_\theta \psi_t) = 0, \rho_0 = \rho_A,\rho_1 = \rho_B.
\end{equation}
With this equation, we can write the geodesic distance $\calD^{\rm W}: \PP \times \PP \to \R^+$ as: 
\begin{equation}
\begin{split}
    \calD^{\rm W}(\rho_A, \rho_B)^2 
    &= \inf_{\rho_t} \Bigl\{\int_0^1 g_{\rho_t}^{\rm W}(\partial_t{\rho}_t, \partial_t{\rho}_t) \mathrm{d}t: \,\rho_0 = \rho_A, \rho_1 = \rho_B\Bigr\}\\
    &= \inf_{\psi_t \in \mathsf{L}}\Bigl\{\int_0^1 \mathrm{d}t \int \rho_t  |\nabla_\theta \psi_t|^2 \mathrm{d}\theta\Bigr\},
\end{split}
\end{equation}
where $\mathsf{L}$ is the set of time-dependent potentials $\psi_t$ such that
equation \cref{eq:Leq} holds.
This is the celebrated Benamou-Brenier formula for the 2-Wasserstein distance~\cite{benamou2000computational}.

\subsubsection{Flow Equation}
The Wasserstein gradient flow of the KL divergence is
\begin{equation}
\begin{aligned}    
\label{eq:mean-field-Wasserstein}
\frac{\partial \rho_t}{\partial t} &= -\M^{\mathrm{W}}(\rho_t)^{-1}\frac{\delta \mathcal{E}}{\delta \rho}\Bigr|_{\rho=\rho_t}\\
&= \nabla_{\theta} \cdot \big(\rho_t (\nabla_{\theta} \log \rho_t  - \nabla_{\theta} \log \rho_{\rm post})\big) \\
& = -\nabla_{\theta} \cdot (\rho_t \nabla_\theta \log \rho_{\mathrm{post}}) + \Delta_{\theta} \rho_t.
\end{aligned}
\end{equation}
This is simply the Fokker-Planck equation for the Langevin dynamics
\begin{equation}
\label{eq:lan}
\mathrm{d}\theta_t=\nabla_{\theta}
\log\rho_{\mathrm{post}}(\theta)\mathrm{d}t+\sqrt{2}\mathrm{d}W_t,
\end{equation}
where $W_t\in \R^{N_{\theta}}$ is a standard Brownian
motion. This is a trivial mean-field model of the form \cref{eq:MFD} in the sense that
there is no dependence on the density $\rho_t$ associated with the law of
$\theta.$

\begin{newremark}
We note that elliptic equations defining certain potentials arise in the context of both the Fisher-Rao as well as the Wasserstein metric. However, while (\ref{eq:liouville}) appears in the definition of the Wasserstein metric only, 
solving (\ref{eq:Fisher-Rao-Liouville}) is required for obtaining the 
mean-field equations (\ref{eq:Fisher-Rao-MD}) in the Fisher-Rao setting.
Returning to the cost functional (\ref{eq:Bayesian_U}), 
we find that the associated Wasserstein gradient mean-field dynamics simply reduces to gradient descent 
\begin{align}
    {\rm d}\theta_t = -\nabla_\theta \VL (\theta){\rm d}t
\end{align}
while the associated Fisher-Rao mean-field equations are more complex and linked to Bayesian inference as discussed earlier in \Cref{rem:homotopy}.
\end{newremark}

\subsubsection{Affine Invariance}
\label{ssec:W-AI}

The Wasserstein Riemannian metric~\cref{eq:Wasserstein-metric} is not affine invariant. Hence,
in this subsection, we introduce an affine invariant modification to the  Wasserstein metric.  To this end, we consider preconditioner 
$\Prec: \R^{N_\theta}\times \PP \to \R^{N_\theta\times N_\theta}_{\succ 0},$
where the output space is the cone of positive-definite symmetric matrices.

We generalize \cref{eq:liouville} and let $\psi_\sigma$ solve the PDE
\begin{equation}
\label{eq:liouville_AI}
-\nabla_\theta \cdot (\rho P(\theta,\rho) \nabla_\theta \psi_\sigma)=\sigma,
\end{equation}
again noting that specification of function spaces is needed to ensure 
unique invertibility of the divergence form elliptic operator
(see \Cref{proposition:AI-Fisher-Rao-MD} where similar
considerations arise).
We may then generalize the metric tensor in \cref{eq:gm2} 
to obtain $\M^{\mathrm{AIW}}(\rho)$ and inverse given by
\begin{subequations}
\label{eq:gm3}
\begin{align}
&\M^{\mathrm{AIW}}(\rho) \sigma = \psi_{\sigma}, \quad \forall~\sigma \in T_{\rho}\PP,\\
&\M^{\mathrm{AIW}}(\rho)^{-1} \psi = -\nabla_\theta \cdot (\rho P(\theta,\rho) \nabla_\theta \psi) , \quad  \forall~\psi \in T_{\rho}^{*}\PP.
\end{align}
\end{subequations}
Manipulations similar to use in \cref{eq:Wasserstein-metric}, but using $\M^{\mathrm{AIW}}(\rho)$, show that
\begin{align*}
g_{\rho}^{\mathrm{AIW}}(\sigma_1,\sigma_2) &= \langle \M^{\mathrm{AIW}}(\rho) \sigma_1, 
 \sigma_2 \rangle\\
&=\langle \psi_{\sigma_1}, \M^{\mathrm{AIW}}(\rho)^{-1} \psi_{\sigma_2} \rangle\\
&= \int \rho(\theta) \nabla_{\theta} \psi_{\sigma_1}(\theta)^T \Prec(\theta,\rho)  \nabla_{\theta} 
\psi_{\sigma_2}(\theta) \mathrm{d}\theta.
\end{align*}
It follows that $g_{\rho}^{\mathrm{AIW}}$ is positive-definite and hence 
a valid metric tensor. 
We have the following proposition to guarantee this metric tensor is affine invariant:
\begin{proposition}
\label{proposition:Wasserstein-affine-invariant}
Under the assumption on $P$ given in 
\cref{proposition:AI-Fisher-Rao-MD}, leading to \eqref{eq:P-affine}, 
the metric corresponding to $M^{\mathrm{AIW}}$ is affine invariant.  
Consequently, the associated gradient flow of the KL divergence, namely
\begin{equation}
\begin{aligned}    
\label{eq:AI-WGF}
\frac{\partial \rho_t(\theta)}{\partial t} 
&= \nabla_{\theta} \cdot \Bigl(\rho_t \Prec(\theta, \rho_t) (\nabla_{\theta} \log \rho_t  - \nabla_{\theta} \log \rho_{\rm post})\Bigr),
\end{aligned}
\end{equation}
is affine invariant.
\end{proposition}
The proof of this proposition is provided in \cref{proof:Wasserstein-affine-invariant}. Henceforth we refer to $M^{\mathrm{AIW}}$ satisfying the condition of the
preceding proposition as an affine invariant Wasserstein metric tensor.

\subsubsection{Mean-Field Dynamics}
\label{sssec:WGF-MD}
As discussed in relation to the topic of affine invariance in \Cref{sssec:MFD},
mean-field models with a given law are not unique. In the specific context
of the Wasserstein gradient flow  which suggests
looking beyond \cref{eq:lan} for a mean-field model with governing law
given by \cref{eq:mean-field-Wasserstein}. This can be achieved as 
follows~\cite{subrahmanya2021ensemble}. Fix arbitrary 
$h : \R^{N_{\theta}} \times \R   \rightarrow \R^{N_{\theta}\times N_{\theta}}$,
define $D(\theta,\rho) = \frac{1}{2}h(\theta,\rho)h(\theta,\rho)^T$ and choose
$d(\theta,\rho) = \nabla_{\theta} \cdot D(\theta,\rho)$. Then, for any $h$, consider
the SDE
\begin{equation}
\begin{aligned}    
\label{eq:particle-Wasserstein}
\mathrm{d}\theta_t = 
\Bigl(
\nabla_{\theta} \log \rho_{\rm post}(\theta_t)  + \bigl(D(\theta_t,\rho_t) - I \bigr)\nabla_{ \theta} \log \rho_t(\theta_t) + d(\theta_t, \rho_t) 
\Bigr)\mathrm{d}t 
+ h(\theta_t,\rho_t){\rm d}W_t,
\end{aligned}
\end{equation}
When $h=\sqrt{2}I$ we recover \cref{eq:lan}. When this condition does
not hold, so that $D(\theta,\rho_t)  \neq I$, the equation requires
knowledge of the score function $\nabla_{ \theta} \log \rho_t(\theta_t)$;
and particle methods to approximate \cref{eq:particle-Wasserstein} will require
estimates of the score; various approaches have been adopted in the literature \cite{maoutsa2020interacting, wang2022optimal, shen2022self, boffi2022probability}. See also \cite{song2020sliced} and references therein
for discussion of score estimation. Notably, by choosing $h=0$ in \cref{eq:particle-Wasserstein}, one can obtain a deterministic particle system, which may be preferred in practical implementations. Alternatively, in \cite{he2022regularized}, interpolation between the Wasserstein metric and Stein metric was studied to derive deterministic particle approximations of the Wasserstein gradient flow.

We now apply similar considerations to the  preconditioned Wasserstein gradient  flow~\cref{eq:AI-WGF}. Employing the same choices of $d$ and $D$ from $h$ as in the
unpreconditioned case we obtain the following 
mean-field evolution equation:
\begin{equation} 
\label{eq:AI-Wasserstein-MD-Ct}
\begin{aligned}
\mathrm{d}\theta_t &= 
\Prec(\theta_t, \rho_t)\nabla_{\theta} \log \rho_{\rm post}(\theta_t)\mathrm{d}t\\
&\quad\quad\quad+ \Bigl(\bigl(D(\theta_t,\rho_t) - \Prec(\theta_t, \rho_t) \bigr)\nabla_{ \theta} \log \rho_t(\theta_t) + d(\theta_t, \rho_t)  
\Bigr)\mathrm{d}t\\
&\qquad\qquad\qquad\qquad\qquad\qquad+ h(\theta_t, \rho_t){\rm d}W_t.
\end{aligned}
\end{equation}
For this specific mean-field  equation~\cref{eq:AI-Wasserstein-MD-Ct}, we can also
establish affine invariance; see the following proposition and its proof in \cref{proof:AI-Wasserstein-MD}.
\begin{proposition}
\label{lem:AI-Wasserstein-MD}
    The mean-field equation~\cref{eq:AI-Wasserstein-MD-Ct} is affine invariant under the assumption on the preconditioner $P$ given in \cref{proposition:AI-Fisher-Rao-MD}, leading to \eqref{eq:P-affine}, 
    and the assumptions on $h$ given in \cref{eq:aMFD-AI-sigma}.
\end{proposition}

In particular, let $C(\rho)$ denote the covariance matrix of $\rho.$
If we take $\Prec(\theta, \rho) = C(\rho)$ 
 then we recover the affine invariant Kalman-Wasserstein metric introduced in \cite{garbuno2020interacting,garbuno2020affine}. Furthermore, then making the choice of $h(\theta,\rho) = \sqrt{2C(\rho)}$ leads to the following affine invariant overdamped Langevin equation, also introduced in \cite{garbuno2020interacting,garbuno2020affine}:
\begin{equation}
\begin{aligned}    
\label{eq:AI-Wasserstein-MD-Ct2}
\mathrm{d}\theta_t = 
C(\rho_t)\nabla_{\theta} \log \rho_{\rm post}(\theta_t)\mathrm{d}t 
+ \sqrt{2C(\rho_t)}{\rm d}W_t.
\end{aligned}
\end{equation}
Comparison with \cref{eq:lan} demonstrates that it is a preconditioned
version of the standard overdamped Langevin equation.

\subsection{Stein Gradient Flow} 
\label{ssec:Stein}

\subsubsection{Metric}

Generalizing \cref{eq:liouville} we let $\psi_\sigma$ solve the integro-partial
differential equation
\begin{equation}
\label{eq:liouville2}
 -\nabla_{\theta} \cdot \Bigl(\rho(\theta)  \int \kappa(\theta,\theta',\rho) \rho(\theta')\nabla_{\theta'} \psi_\sigma(\theta') \mathrm{d}\theta'\Bigr)=\sigma(\theta),
\end{equation}
Here $\kappa: \R^{N_{\theta}}\times\R^{N_{\theta}}\times\PP \rightarrow \R$ is a positive definite kernel for any fixed $\rho$. As before definition of function space setting is required to ensure that
this equation is uniquely solvable.
Now define the Stein metric tensor  $\M^{\mathrm{S}}(\rho)$, and its inverse, as follows:
\begin{subequations}
\label{eq:Stein-metric}
\begin{align}
&\M^{\mathrm{S}}(\rho) \sigma = \psi_\sigma, \quad \forall~\sigma \in T_{\rho}\PP,\\
&\M^{\mathrm{S}}(\rho)^{-1} \psi =   -\nabla_{\theta} \cdot \Bigl(\rho(\theta)  \int \kappa(\theta,\theta',\rho) \rho(\theta')\nabla_{\theta'} \psi(\theta') \mathrm{d}\theta'\Bigr), \quad \forall~\psi \in T_{\rho}^{*}\PP.
\end{align}
\end{subequations}
Computations analogous to those shown in \cref{eq:Wasserstein-metric} show that
the Stein Riemannian metric implied by metric tensor $\M^{\mathrm{S}}$ is given by
\begin{subequations}
\begin{align}
g_{\rho}^{\mathrm{S}}(\sigma_1,\sigma_2) &= \langle \M^{\mathrm{S}}(\rho)\sigma_1,\sigma_2 \rangle\\
&= \langle \psi_{\sigma_1},\M^{\mathrm{S}}(\rho)^{-1}\psi_{\sigma_2} \rangle\\
&=  \int\int \kappa(\theta,\theta',\rho)\rho(\theta)\nabla_\theta \psi_{\sigma_1}(\theta)^T \nabla_{\theta'}\psi_{\sigma_2}(\theta') \rho(\theta')\mathrm{d}\theta\mathrm{d}\theta'.
\end{align}
\end{subequations}

\begin{newremark}
 As in the Wasserstein setting, the Stein Riemannian metric~\cite{liu2017stein} also has a transport interpretation. The Stein metric identifies, for each $\sigma \in T_\rho \PP$, the set of velocity fields $v$ satisfying the constraint $\sigma = -\nabla_\theta \cdot (\rho v)$. Then $v_{\sigma} = \argmin_v  \|v\|_{\mathcal{H}_\kappa}^2$, with minimization over all $v$ satisfying the constraint, and where $\mathcal{H}_\kappa$ is a Reproducing Kernel Hilbert Space (RKHS) with kernel $\kappa$. A formal Lagrangian multiplier argument shows that 
$$v_{\sigma} = \int \kappa(\theta,\theta',\rho) \rho(\theta')\nabla_{\theta'} \psi_{\sigma}(\theta') \mathrm{d}\theta'$$ for some $\psi_{\sigma}.$ The Stein metric measures this transport change via the RKHS norm $\|v_{\sigma}\|_{\mathcal{H}_\kappa}^2$, leading to
the interpretation that the Stein Riemannian metric can be written in the form
\[g_{\rho}^{\mathrm{S}}(\sigma_1,\sigma_2) = \langle v_{\sigma_1}, v_{\sigma_2}\rangle_{\mathcal{H}_\kappa}. \]
\end{newremark}
Analogously to \cref{eq:Leq}, for any $\rho_A, \rho_B \in \mathcal{P}$, we may write a path to connect these endpoints, which is defined by
\begin{equation}
\label{eq:Leq2}
\frac{\partial \rho_t}{\partial t} + \nabla_\theta\cdot\Bigl(\rho_t \int \kappa(\theta,\theta',\rho_t) \rho_t(\theta')\nabla_{\theta'} \psi_t(\theta') \mathrm{d}\theta'\Bigr) = 0, \rho_0 = \rho_A, \rho_1 = \rho_B.
\end{equation}
The corresponding geodesic distance $\calD^{\rm S}: \PP \times \PP \to \R^+$ is 
\begin{equation}
\begin{split}
    \calD^{\rm S}(\rho_A, \rho_B)^2 
    &= \inf_{\rho_t} \Bigl\{\int_0^1 g_{\rho_t}^{\rm S}(\partial_t{\rho}_t, \partial_t{\rho}_t) \mathrm{d}t: \,\rho_0 = \rho_A, \rho_1 = \rho_B\Bigr\}\\
    &= \inf_{\psi_t \in \mathsf{L}} \Bigl\{\int_0^1 \mathrm{d}t \int\int \kappa(\theta,\theta',\rho_t)\rho_t(\theta)\nabla_\theta \psi_t(\theta)\cdot \nabla_{\theta'}\psi_{t}(\theta') \rho_t(\theta')\mathrm{d}\theta\mathrm{d}\theta'\Bigr\},
\end{split}
\end{equation}
where $\mathsf{L}$ is the set of time-dependent potentials $\psi_t$ such that
equation \cref{eq:Leq2} holds.

\subsubsection{Flow Equation}
The Stein variational gradient flow is
\begin{equation}
\begin{aligned}    
\label{eq:Stein-GF}
\frac{\partial \rho_t(\theta)}{\partial t} &=-\Bigl(\M^{\mathrm{S}}(\rho_t)^{-1}\frac{\delta \mathcal{E}}{\delta \rho}\Bigr|_{\rho=\rho_t}\Bigr)(\theta) \\
&= \nabla_{\theta}\cdot\Bigl(\rho_t(\theta)\int \kappa(\theta,\theta',\rho_t)\rho_t(\theta')\nabla_{\theta'} \bigl(\log\rho_t(\theta') - \log \rho_{\rm post}(\theta') \bigr)\mathrm{d}\theta' \Bigr).
\end{aligned}
\end{equation}

\subsubsection{Affine Invariance}
The Stein metric~\cref{eq:Stein-metric} is not affine invariant. To address this,
in this subsection, we introduce an affine invariant modification. The generalization
is similar to that undertaken to obtain an affine invariant version of the
Wasserstein metric and so we will make the presentation brief. We define
$$\M^{\mathrm{AIS}}(\rho) : T_{\rho}\PP \rightarrow T_{\rho}^{*}\PP,$$ 
so that for any $\psi \in T_{\rho}^{*}\PP$, it holds that
\begin{align}
\M^{\mathrm{AIS}}(\rho)^{-1} \psi = - \nabla_{\theta}\cdot\Bigl(\rho(\theta)\int \kappa(\theta,\theta',\rho)\rho(\theta')\Prec( \theta, \theta',\rho)\nabla_{\theta'}\psi(\theta')\mathrm{d}\theta' \Bigr).
\end{align}
Here $\kappa: \R^{N_{\theta}}\times\R^{N_{\theta}}\times\PP \rightarrow \R$ is a positive definite kernel and we factorize the preconditioner $\Prec: \R^{N_{\theta}}\times\R^{N_{\theta}}
\times \PP \to \R^{N_{\theta}\times N_{\theta}}$, which can be written in the form
$\Prec(\theta, \theta',\rho) = L(\theta,\rho)L(\theta',\rho)^T.$
With this in hand it follows that
\begin{align*}
&\langle \psi, \M^{\mathrm{AIS}}(\rho)^{-1}\psi \rangle \\
= &\int \int \kappa(\theta,\theta',\rho)\rho(\theta) \left(L(\theta,\rho)^T\nabla_\theta\psi(\theta)\right)^T \left(L(\theta',\rho)^T\nabla_{\theta'}\psi(\theta')\right) \rho(\theta')\mathrm{d}\theta \mathrm{d}\theta' \geq 0
\end{align*}
and the resulting metric is well-defined.
We have the following proposition to guarantee this metric tensor is affine invariant:
\begin{proposition}
\label{proposition:Stein-affine-invariant}
Consider the invertible affine transformation $\tilde\theta = \varphi(\theta) = A \theta + b$ and correspondingly $\tilde{\rho} = \varphi \# \rho$; moreover $\tilde{\theta}' = \varphi(\theta')$. Assume that the preconditioning matrix satisfies
$$ \kappa(\tilde\theta, \tilde\theta',\tilde\rho)\Prec(\tilde \theta, \tilde \theta',  \tilde\rho)
= \kappa(\theta, \theta', \rho) A \Prec(\theta, \theta', \rho) A^T.$$
Then the metric corresponding to $M^{\mathrm{AIS}}$ is affine invariant. Consequently, the associate gradient flow of the KL divergence, namely
\begin{equation}
\begin{aligned}    
\label{eq:AI-Stein}
\frac{\partial \rho_t(\theta)}{\partial t} &= \nabla_{\theta}\cdot(\mathsf{f})\\
\mathsf{f}&= \Bigl(\rho_t(\theta)\int \kappa(\theta,\theta',\rho_t)\rho_t(\theta')\Prec( \theta, \theta',\rho_t)\nabla_{\theta'} \bigl(\log\rho_t(\theta') - \log \rho_{\rm post}(\theta') \bigr)\mathrm{d}\theta' \Bigr)
\end{aligned}
\end{equation}
is affine invariant.
\end{proposition}
The proof of this proposition is in~\cref{proof:Stein-affine-invariant}. 
Henceforth we refer to $M^{\mathrm{AIS}}$ satisfying the condition of the
preceding proposition as an affine invariant Stein metric tensor. As an example, we can obtain an affine invariant Stein metric by making the
choices $P = C(\rho)$ and $ \kappa(\theta, \theta',\rho) \propto \exp\bigl\{-\frac{1}{2}(\theta - \theta')^TC(\rho)^{-1}(\theta - \theta')\bigr\}$; this set-up is considered in our numerical experiments; see \Cref{sec:Numerics}.

\subsubsection{Mean-Field Dynamics}
\label{rmk: AI-SGF}
The Stein gradient flow (\ref{eq:Stein-GF}) has the following
mean-field counterpart~\cite{liu2016stein,liu2017stein} in $\theta_t$ with the
law $\rho_t$:
\begin{equation}
\begin{aligned}    
\label{eq:particle-Stein}
\frac{\mathrm{d}\theta_t}{\mathrm{d}t} 
&= 
\int \kappa(\theta_t,\theta',\rho_t)\rho_t(\theta') \nabla_{\theta'}\bigl(\log \rho_{\rm post}(\theta') - \log\rho_t(\theta') \bigr)\mathrm{d}\theta' \\
&= 
\int \kappa(\theta_t,\theta',\rho_t)\rho_t(\theta') \nabla_{\theta'}\log \rho_{\rm post}(\theta') +  \rho_t(\theta')\nabla_{\theta'}\kappa(\theta_t,\theta', \rho_t)\mathrm{d}\theta'.
\end{aligned}
\end{equation}
Here, the second equality is obtained using integration by parts; it facilitates
an expression that avoids the score (gradient of the log density function of $\rho_t$). 
This is useful because, when implementing particle methods, the resulting integral can then be approximated directly by Monte Carlo methods.

Similarly, for the preconditioned Stein gradient flow~\cref{eq:AI-Stein}, we can construct the following mean-field equation:
\begin{equation}
\begin{aligned}    
\label{eq:AI-Stein-MD}
\frac{\mathrm{d}\theta_t}{\mathrm{d}t} 
&= 
\int \Bigl(\kappa(\theta_t,\theta', \rho_t)\rho_t(\theta') \Prec( \theta_t, \theta', \rho_t ) \nabla_{\theta'}\log \rho_{\rm post}(\theta')\\
&\qquad\qquad\qquad\qquad\qquad\qquad +  \nabla_{\theta'}\cdot (\kappa(\theta_t,\theta',\rho_t) \Prec(\theta_t, \theta', \rho_t) ) \rho_t(\theta')\Bigr) \mathrm{d}\theta'.
\end{aligned}
\end{equation}
The mean-field equation~\cref{eq:AI-Stein-MD} is affine invariant; see the following proposition and its proof in~\cref{proof:AI-Stein-MD}.
\begin{proposition}
\label{lem:AI-Stein-MD}
    The mean-field equation~\cref{eq:AI-Stein-MD} is affine invariant under the assumption on the preconditioner in \cref{proposition:Stein-affine-invariant}.
\end{proposition}

\subsection{Large-Time Asymptotic Convergence}
\label{ssec:gf-theory}

In the three preceding subsections, we studied gradient flows, under various
different metrics, of the energy $\mathcal{E}$ given in \cref{eq:energy}. We derived the gradient flow equations by using the differential structures of smooth positive densities.
In this subsection, we study the convergence of these gradient flows,
surveying known results, and adding new ones. In this subsection we no longer assume the probability density is smooth as this smoothness assumption was made purely for the purpose of deriving the form of the equation. We will detail the assumptions of the probability densities for each of the results presented in this subsection. 
In short, the convergence  of the Fisher-Rao gradient flow occurs at rate $\bigO(\exp(-t))$ and is hence
problem independent; this reflects the invariance of the metric under any diffeomorphism. In contrast, the proven results for Wasserstein and Stein gradient flows have convergence rates that depend on the problem, even after being modified to be affine invariant. We note, however, that when $\rho_{\rm post}$ is Gaussian, the affine invariant Wasserstein gradient flows also achieve $\bigO(\exp(-t))$ \cite{garbuno2020interacting, garbuno2020affine}. Numerical results illustrating
and complementing the analysis in this section may be found in \Cref{sec:Numerics}.

\subsubsection{Fisher-Rao Gradient Flow}
We have the following proposition concerning large-time convergence of the gradient flow:
\begin{proposition}
\label{proposition:FR-convergence}  
Assume that there exist constants $K, B>0$ 
such that the initial density $\rho_0$ satisfies
\begin{align}\label{e:asup1}
    e^{-K(1+|\theta|^2)}\leq \frac{\rho_0(\theta)}{\rho_{\rm post}(\theta)}\leq e^{K(1+|\theta|^2)},
\end{align}
and both $\rho_0, \rho_{\rm post}$ have bounded second moment
\begin{align}\label{e:asup2}
    \int |\theta|^2 \rho_0(\theta)\mathrm{d} \theta\leq B, \quad \int |\theta|^2 \rho_{\rm post}(\theta)\mathrm{d} \theta\leq B.
\end{align}
Let $\rho_t$ solve the Fisher-Rao gradient flow \cref{eq:mean-field-Fisher-Rao}. Then, for any $t\geq \log\bigl((1+B)K\bigr)$,
\begin{align}\label{e:KLconverge}
    {\rm KL}[\rho_{t} \Vert  \rho_{\rm post}]\leq (2+B+eB)Ke^{-t}.
\end{align}
\end{proposition}
It is notable that the exponential convergence rate is \textit{independent} of the
properties of the target distribution $\rho_{\rm post};$ this reflects invariance
of the flow under any diffeomorphism.
The proof of this proposition is in \cref{proof:FR-convergence}. Similar propositions are in~\cite[Theorem 3.3]{lu2019accelerating} and~\cite[Theorem 2.3]{lu2022birth};
our results relax the assumptions required on the initial condition.

\subsubsection{Wasserstein Gradient Flow}
The convergence of the Wasserstein gradient flow \cref{eq:mean-field-Wasserstein} is widely studied \cite{villani2021topics}. A variety of different conditions on $\rho_{\rm post}$ lead to the exponential convergence of the Wasserstein gradient flow to $\rho_{\rm post}$ with convergence rate $e^{-2\alpha t}$~\cite{bakry2014analysis}. 
They include that 
$\rho_{\rm post}$ is $\alpha$-strongly logconcave (\cref{assumption:logconcave}) \cite{bakry1985diffusions}
or that $\rho_{\rm post}$ satisfies the log-Sobolev inequality~\cite{gross1975logarithmic} or Poincar\'e inequality~\cite{poincare1890equations} with constant $1/\alpha$. 
We have the following proposition concerning the convergence of the affine-invariant Wasserstein gradient flow~\cref{eq:AI-WGF}:

\begin{definition}
\label{assumption:logconcave}
    The distribution $\rho_{\rm post}(\theta)$ is called $\alpha$-strongly logconcave, if the function $\log \rho_{\rm post}$ is twice differentiable and 
\begin{align}\label{a:logconcave}
 -\Hess \log \rho_{\rm post}(\theta) \succeq \alpha I. 
\end{align}
\end{definition}

\begin{proposition}
\label{proposition:W-convergence}
Assume $\rho_{\rm post}(\theta)$ is $\alpha$-strongly logconcave and there exists $\lambda > 0$ such that  $\Prec(\theta, \rho) \succeq \lambda I$ along the affine-invariant Wasserstein gradient flow.
Then the solution $\rho_t$ of the affine-invariant Wasserstein gradient flow~\cref{eq:AI-WGF}    satisfies
$$ \frac{1}{2}\lVert\rho_t - \rho_{\rm post} \rVert^2_{L_1} \leq \mathrm{KL}[\rho_0 \Vert  \rho_{\rm post}] e^{-2\alpha \lambda t},$$
where $\lVert \cdot \rVert_{L_1}$ denotes the $L_1$ norm.
\end{proposition}

The proof of the proposition is in \cref{proof:W-convergence}. 
It is a generalization of~\cite[Proposition 3.1]{garbuno2020interacting}
which concerns the specific preconditioner $\Prec_t:=P(\theta_t,\rho_t)$
chosen to equal $C_t$, the covariance at time $t.$
A key point to appreciate is that, in contrast to the exponential rates reported
for Fisher-Rao gradient descent, the exponential rates reported here depend on
the problem. When $\rho_{\rm post}$ is Gaussian, the affine invariant Wasserstein gradient flows, however, provably achieves convergence rate $\bigO(\exp(-t))$ \cite{garbuno2020interacting, garbuno2020affine}; it would be of interest to identify
classes of non-Gaussian problems where this rate is also achievable for the
affine invariant Wasserstein gradient flow.


\subsubsection{Stein Gradient Flow}
For the Stein gradient flow~\cref{eq:Stein-GF} the solution $\rho_t$ converges weakly to $\rho_{\rm post}$ as $t\rightarrow \infty$, under certain assumptions~\cite[Theorem 2.8]{lu2019scaling}\cite[Proposition 2]{korba2020non}; the exponential rates
are problem-dependent, similar to those for Wasserstein gradient flows in the preceding subsection, and in contrast to those for the Fisher-Rao gradient flow
which give a universal rate across wide problem classes.
Quantitative rates and necessary functional inequalities for the exponential convergence near the equilibrium in terms of the decay of the KL divergence are discussed in~\cite{duncan2019geometry}. However, the speed of convergence for initial distributions far from equilibrium remains an open and challenging problem. 


\section{Gaussian Approximate Gradient Flow}
\label{sec:GGF}
In this section, we revisit the gradient flows of the energy \cref{eq:energy}
under the Fisher-Rao, Wasserstein, and Stein metrics. 
We confine variations
to the manifold of Gaussian densities $\PPG$ defined in \cref{eq:PPG}, in contrast
to the previous \Cref{sec:GF and AI}, in which we consider variations in the whole of $\PP$ defined in \cref{eqn-smooth-positive-densities}.
The corresponding Gaussian approximate gradient flows underpin Gaussian variational inference, which aims to identify the minimizers of \cref{eq:GVI}.
We first introduce the basics of metrics and gradient flow in the Gaussian density space, identify the ways that Gaussian approximations can be made and develop the concept of affine invariance for them in \Cref{ssec:GGF}. Then we introduce the Gaussian approximate Fisher-Rao gradient flow in \Cref{ssec:Gaussian-Fisher-Rao}, Gaussian approximate Wasserstein gradient flow in \Cref{ssec:Gaussian-Wasserstein} and Gaussian approximate Stein gradient flow in \Cref{ssec:Gaussian-Stein}; in all cases we 
also discuss affine invariance and introduce affine invariant modifications where
appropriate.  We find that different affine invariant metrics lead to very similar gradient flows; and in particular to flows with very similar large time behavior.
We discuss the large time convergence properties of these Gaussian approximate 
gradient flows in \Cref{ssec:GGF-converge}.

\subsection{Basics of Gaussian Approximate Gradient Flows}
\label{ssec:GGF}
In this subsection, we introduce gradient flows
in the Gaussian density space; we follow the structure of \Cref{sec:gradient-flow}. We study the problem from the perspective of the metric in
\Cref{ssec:GM}, the perspective of the flow equations
in \Cref{ssec:GF}, the perspective of affine invariance
in \Cref{ssec:GA}, and the perspective of mean-field equations in \Cref{ssec:GMD}.
For Gaussian evolutions, the mean-field models are evolution equations
for the state defined by affine (in the state) tangent vector field;
the affine map is defined by mean-field expectations with 
respect to the Gaussian with mean and covariance of the state.

\subsubsection{Metric}
\label{ssec:GM}
Recall the manifold of Gaussian densities in \cref{eq:PPG}, which 
has dimension $N_\rhoa$. We assume we are given a metric
$g_\rho$ and metric tensor $M(\rho)$, depending on $\rho \in \PP$,
and we now wish to find corresponding objects defined for parametric
variations within the family of Gaussian densities $\PPG$\footnote{In fact
our development is readily generalized to the determination of the corresponding
objects  for any parametrically dependent manifold of densities, not just Gaussians.}.
To this end we introduce $\rho_a$, with $a \in \R^{N_\rhoa}$, denoting
the parametric family. We aim to find \as{reduced metric $\fg$} and metric
tensor $\fM(a)$ in the parameter space $\R^{N_\rhoa}$ rather than in $\PP.$

Noting that
\begin{align}\label{e:Tidentify}
\lim_{\epsilon\rightarrow 0}\frac{\rho_{\rhoa+\epsilon \sigma}-\rho_\rhoa}{\epsilon}=\nabla_\rhoa \rho_\rhoa \cdot \sigma,
\end{align}
we see that any element in the tangent space $T_{\rho_\rhoa}\PPG$ can be identified 
with a vector  $\sigma\in \bR^{N_\rhoa}$. We denote the Riemannian metric restricted to $\PPG$ at 
$\rho_\rhoa$ as $\fg_{\rhoa}$. Then
\begin{align}
\label{eq:def-M_alpha}
\fg_{\rhoa}(\sigma_1,\sigma_2):=g_{\rho_\rhoa}(\nabla_\rhoa \rho_\rhoa\cdot\sigma_1,\nabla_\rhoa \rho_\rhoa\cdot\sigma_2)=\langle\fM(\rhoa)\sigma_1,\sigma_2\ranglen,
\end{align}
where $\sigma_1,\sigma_2\in T_{\rho_\rhoa}\PPG$, 
and the induced metric tensor is given by
\begin{equation}
\label{eq:mra}
\fM(\rhoa):=\int \nabla_\rhoa \rho_\rhoa(\theta) \bigl(M(\rho_a)\nabla_\rhoa \rho_\rhoa^T\bigr)(\theta)\mathrm{d}\theta.
\end{equation}

\subsubsection{Flow Equation}
\label{ssec:GF}
Given \cref{eq:mra} it is intuitive that the gradient flow in the parameter space
implied by the gradient flow in the manifold of Gaussians is given by
\begin{align}
\label{eq:G-GF}
   \frac{\partial \rhoa_t}{\partial t} = -\fM(\rhoa_t)^{-1}\left.\frac{\partial  \mathcal{E}(\rho_{\rhoa})}{\partial \rhoa } \right|_{\rhoa=\rhoa_t}.
\end{align}
We refer to \cref{eq:G-GF} as the Gaussian approximate gradient flow; it is 
formulated as an evolution equation in the parameter space. It is also
possible to write an evolution equation for $\rho_{a_t}$ in the space
of Gaussian probability densities $\mathcal{P}^G$. 

Our goal now is to show that \cref{eq:G-GF} may be derived by using any one of the following proximal, Riemannian, and moment closure perspectives. In particular, 
these perspectives justify that the gradient flow in the space of Gaussian 
densities $\mathcal{P}^G$ is a \textit{Gaussian approximation} of the gradient 
flow on the whole probability space. Such approximation can be interpreted either 
by constraining the minimization underlying the proximal perspective, by 
the projection of the flow field based on the Riemannian metric, or through 
a moment closure reduction of probability densities. The latter moment closure
approach is particularly expedient for determination of the form of the equation
\cref{eq:G-GF}.

\vspace{0.1in}
\paragraph{\bf Proximal Perspective}
\label{ssec:proximal}
Given the metric $g$ and the corresponding distance function $\calD$,  the proximal point method \cref{eq:prox0} can be restricted to the space of Gaussian
densities, leading to the iteration
\begin{align}
\label{eq:prox}
    \rrho_{{n+1}} = \argmin_{\rho \in \PPG}\Bigl(\EE(\rho) + \frac{1}{2\Delta t}\calD(\rho, \rrho_{{n}})^2\Bigr)
\end{align}
to minimize the energy functional $\EE$ in Gaussian density function space $\PPG$. 
Since elements in $\PPG$ are uniquely defined via a point $a \in \bR^{N_\rhoa}$,
the map $\rrho_{{n}} \mapsto \rrho_{{n+1}}$ implicitly defines a map $a_n \mapsto a_{n+1}.$ Thus we write $\rrho_{{n}}=\rho_{a_n}$ and determine the update
equation for $a_n$. When $\Delta t$ is small
it is natural to seek $a_{n+1}=a_n+\Delta t\sigma_n$ and note that, invoking the
approximations implied by \cref{e:Tidentify} and \cref{eq:geo},
\begin{align*}
    \sigma_{{n}} &\approx \argmin_{\sigma \in \bR^{N_\rhoa}} \Bigl(\EE(\rho_{a_n}+\Delta t
    \nabla_\rhoa \rho_{\rhoa_n} \cdot \sigma) + \frac{1}{2}{\Delta t}\langle M(\rho_{a_n})\nabla_\rhoa \rho_{\rhoa_n} \cdot \sigma, \nabla_\rhoa \rho_{\rhoa_n} \cdot \sigma \rangle \Bigr),\\
    &= \argmin_{\sigma \in \bR^{N_\rhoa}} \Bigl(\EE(\rho_{a_n}+\Delta t
    \nabla_\rhoa \rho_{\rhoa_n} \cdot \sigma) + \frac{1}{2}{\Delta t}\langle  \fM(a_n)\sigma, \sigma \ranglen \Bigr).
\end{align*}
To leading order in $\Delta t$, this expression is minimized by choosing
$$\sigma_n = -\fM(a_n)^{-1} \Bigl\langle\frac{\delta \mathcal{E}}{\delta \rho}\Bigr|_{\rho = \rho_{\rhoa_n}},\nabla_\rhoa \rho_{\rhoa_n}\Bigr\rangle=- \fM(a_n)^{-1} \frac{\partial \mathcal{E}(\rho_{\rhoa})}{\partial a}\Bigr|_{\rhoa=\rhoa_n}.$$ Letting
$a_n \approx a_{n\Delta t}$ shows that the formal continuous time limit of the proximal
algorithm leads to the corresponding gradient flow~\cref{eq:G-GF}.

\vspace{0.1in}
\paragraph{\bf Riemannian Perspective}
\label{ssec:Riemannian}

We start by defining the projection $P^G: T_{\rho}\PP \to T_{\rho}\PPG$ as follows:
for any $\psi \in T_{\rho}\PP$ we define $P^G$ by requiring that
\begin{equation}
\label{eq:def-Riemannian}
    g_{\rho}(\psi, \sigma) = g_{\rho}(P^G\psi, \sigma), \quad \forall \sigma \in T_{\rho}\PPG.
\end{equation}
The well-posedness of projection $P^G$ stems from the fact $T_{\rho}\PP^G$ is a finite dimensional Hilbert space when endowed with the inner product $g_{\rho}$.
Now consider the gradient flow 
\begin{subequations}
\label{eq:gftp}
\begin{align}
\frac{\partial \rho_t}{\partial t} &= \sigma_t \in T_{\rho_t}\PP,\\
\sigma_t&=-M(\rho_t)^{-1} \frac{\delta \mathcal{E}}{\delta \rho}\Bigr|_{\rho=\rho_t}
\end{align}
\end{subequations}
designed to decrease the functional $\mathcal{E}$ under the metric $g$. 
We note that, by virtue of \cref{eq:vare}, $\sigma_t=\sigma_t(\theta, \rho_t).$

We may now consider the restriction of the gradient flow to variations in
the manifold of Gaussian densities, leading to equation for $\rho_{\rhoa_t} \in \PPG \subseteq \PP$, defined through the corresponding gradient flow
\begin{align}
\label{eq:G-GF-Riemannian}
    \frac{\partial \rho_{\rhoa_t}}{\partial t} = P^G\sigma_t \in T_{\rho_{\rhoa_t}}\PPG.
\end{align}
The proof of the following proposition may be found in \cref{proof:Riemannian}.

\begin{proposition}
\label{prop:Riemannian-G}
The flow of the parameter $a_t$ implied by the evolution equation
\cref{eq:G-GF-Riemannian} for the density $\rho_{a_t}$ is the Gaussian 
approximate gradient flow \cref{eq:G-GF}.
\end{proposition}

\vspace{0.1in}
\paragraph{\bf Moment Closure Perspective}
\label{ssec:moment}
For any gradient flow \cref{eq:gftp} designed to
decrease the functional $\mathcal{E}$ under the metric $g$, we consider the following moment closure approach to obtain a Gaussian approximation. First, we write
evolution equations for the mean and covariance under \cref{eq:gftp} noting that they
satisfy the following identities:
\begin{equation}
\begin{aligned}
\label{eq:mC-Momentum}
&\frac{\mathrm{d} m_t}{\mathrm{d}t} = \frac{\mathrm{d}}{\mathrm{d}t}\int  \rho_t(\theta) \theta \mathrm{d}\theta = \int 
\sigma_t(\theta, \rho_t) \theta \mathrm{d}\theta, \\
&\frac{\mathrm{d} C_t}{\mathrm{d}t} = \frac{\mathrm{d}}{\mathrm{d}t}\int  \rho_t(\theta) (\theta - m_t)(\theta - m_t)^T \mathrm{d}\theta = \int  \sigma_t(\theta, \rho_t)  (\theta - m_t)(\theta - m_t)^T \mathrm{d}\theta.
\end{aligned}
\end{equation}
This is not, in general, a closed system for the
mean and covariance; this is because $\rho_t$ is not, in general,
determined by only the first and second moments. To close the system, we replace $\sigma_t(\theta, \rho_t)$ by $\sigma_t(\theta, \rho_{\rhoa_t})$, where $\rho_{\rhoa_t} = \N(m_t, C_t)$. We obtain the following closed system for the evolution of
$(m_t,C_t)$:
\begin{equation}
\begin{aligned}
\label{eq:mC-Momentum2}
&\frac{\mathrm{d} m_t}{\mathrm{d}t} = \int 
\sigma_t(\theta, \rho_{\rhoa_t}) \theta \mathrm{d}\theta, \\
&\frac{\mathrm{d} C_t}{\mathrm{d}t} = \int  \sigma_t(\theta, \rho_{\rhoa_t})  (\theta - m_t)(\theta - m_t)^T \mathrm{d}\theta. 
\end{aligned}
\end{equation}

The proof of the following proposition, which shows that this moment closure approach delivers the mean and covariance evolution equation of the Gaussian approximate gradient flow~\cref{eq:G-GF}, may be found in~\cref{proof:Riemannian}.

\begin{proposition}
\label{prop:1}
Suppose the following condition holds:
\begin{equation}
    \label{eq:mcp-cond}
    \M(\rho_\rhoa) T_{\rho_\rhoa}\PPG = {\rm span}\{\theta_i, \theta_i\theta_j, 1\leq i,j\leq N_{\theta}\}.
\end{equation}
Here $\M(\rho_\rhoa)$ is the metric tensor and $T_{\rho_\rhoa}\PPG$ is the tangent space. Moreover, $\theta_i, \theta_i\theta_j$ are understood as functions of $\theta$. Then the mean and covariance evolution equations \cref{eq:mC-Momentum2} are equivalent to the Gaussian approximate gradient flow~\cref{eq:G-GF}.
\end{proposition}

Furthermore, \cref{proof:Riemannian} also contains proof of the following \cref{prop:2} indicating
that several of the metrics considered later in this \Cref{sec:GGF} do indeed satisfy the assumption~\cref{eq:mcp-cond}
sufficient for \cref{prop:1} to hold.

\begin{lemma}
\label{prop:2}
Assumption~\cref{eq:mcp-cond} holds for the Fisher-Rao metric, the affine invariant Wasserstein metric with preconditioner $\Prec$ independent of $\theta$, and affine invariant Stein metric with preconditioner $\Prec$ independent of $\theta$ and with a bilinear kernel $\kappa(\theta, \theta',\rho) = (\theta - m)^TA(\rho)(\theta' - m) + b(\rho)$ ($b\neq0$, and $A$ nonsingular).
\end{lemma}

\begin{newremark}
The moment closure perspective was used in \cite{sarkka2007unscented} as a 
heuristic approach to state estimation in the context of the unscented Kalman 
filter. A connection between the heuristics and gradient flow on the 
Bures–Wasserstein space of Gaussian distributions was established in \cite{lambert2022variational}. The latter is equivalent to the Gaussian approximate gradient flow under the Wasserstein metric, also called the Gaussian approximate Wasserstein gradient flow in this paper; see also the discussion in \Cref{sssec: Gaussian-Wasserstein metric}.
\end{newremark}

\subsubsection{Affine Invariance}
\label{ssec:GA}
We now study the affine invariance concept in the setting of Gaussian approximate gradient
flows. Let $\varphi:\theta \rightarrow \tilde\theta$ denote an invertible affine transformation in $\R^{N_{\theta}}$, where $\varphi(\theta) = A\theta + b$ with $A \in \R^{N_{\theta}}\times \R^{N_{\theta}}$, $b\in \R^{N_{\theta}}$, and $A$ invertible. 

We define the push forward operator for various objects. 
\begin{itemize}
    \item For a parametrically-defined density $\rho_\rhoa$, we write $\rho_{\tilde\rhoa} = \varphi \# \rho_\rhoa$, so that $\rho_{\tilde\rhoa}(\tilde\theta) = \rho_\rhoa(\varphi^{-1}(\tilde{\theta}))|A^{-1}|$.
    Specifically, for Gaussian density space where $\rhoa = (m,C)$, 
we have an invertible affine transformation in $\R^{N_\rhoa}$, such that
    $\tilde{\rhoa} = A^G \rhoa + b^G$, where $A^G$ and $b^G$ depend only on $A$ and $b$ and are defined by the identities  $\tilde{m} = Am +b$ and $\tilde{C} = ACA^T.$
    \item For a tangent vector $\sigma \in \R^{N_\rhoa}$ corresponding to $\nabla_{\rhoa}\rho_{\rhoa}(\theta)\cdot\sigma$ in $T_{\rho_\rhoa} \PPG$, we have $\tilde{\sigma} = A^G\sigma \in \R^{N_\rhoa}$ corresponding to $\nabla_{\tilde\rhoa}\rho_{\tilde\rhoa}(\tilde\theta)\cdot \tilde{\sigma}$ in $ T_{\rho_{\tilde\rhoa}} \PPG$, and note that this satisfies $\nabla_{\tilde\rhoa}\rho_{\tilde\rhoa}(\tilde\theta)\cdot \tilde{\sigma} = \varphi \# (\nabla_{\rhoa}\rho_{\rhoa}(\theta)\cdot\sigma)  =\nabla_{\rhoa}\rho_{\rhoa}(\varphi^{-1}(\tilde{\theta})) 
    \cdot\sigma|A^{-1}|$.
    \item For a functional $\mathcal{E}$ on $\PPG$, we define $\tilde{\EE} = \varphi\#\EE$ via $\tilde{\EE}(\rho_{\tilde\rhoa}) = \EE(\varphi^{-1}\# \rho_{\tilde\rhoa})$.
\end{itemize}
With the above, we can make a precise definition of affine invariance for the Gaussian approximate gradient flows. The definition is similar to \cref{def: Affine Invariant Gradient Flow}.
\begin{definition}[Affine Invariant Gaussian Approximate Gradient Flow]
The Gaussian approximate gradient flow~\cref{eq:G-GF} is called 
\emph{affine invariant} if, under any invertible affine transformation $\tilde{a}_t = \varphi(a_t)$, 
the dynamics of $\tilde{\rhoa}_t$ is itself a gradient flow of $\tilde{\mathcal{E}}$, in the sense that
\begin{align}
\label{eqn: Gaussian approximate gradient flow}
   \frac{\partial \tilde{\rhoa}_t}{\partial t} = -\fM(\tilde{\rhoa}_t)^{-1}
   \frac{\partial  \tilde{\mathcal{E}}(\rho_{\tilde{\rhoa}})}{\partial \tilde\rhoa } 
   \Big|_{\tilde{\rhoa}=\tilde{\rhoa}_t}.
\end{align}
\end{definition}

Naturally, if the gradient flow in probability space is affine invariant, then the Gaussian approximate flow has the same property; see the following proposition.
\begin{proposition}
\label{prop: affine invariant metric tensor G}
For any affine invariant metric $g$ defined via \cref{def: Affine Invariant Metric}, the Gaussian approximate gradient flow under the corresponding metric $\fg$ is affine invariant for any $\mathcal{E}$.
\end{proposition}
We provide the proof for this proposition in \cref{proof: prop affine invariant metric tensor G}.


\subsubsection{Mean-Field Dynamics}
\label{ssec:GMD}
The Gaussian approximate gradient flow can also be realized as a mean-field ordinary differential equation. 
Recall that for Gaussian evolutions the mean-field models are
evolution equations for mean and covariance defined via mean-field expectations with
respect to the Gaussian with this mean and covariance.
We have the following lemma.

\begin{lemma}
\label{lemma:Gaussian-interacting-particle}
Consider the mean-field equation 
\begin{equation}
\label{eq:MFG}
\frac{\mathrm{d} \theta_t}{\mathrm{d}t} = \mA(\rho_t,\rho_{\rm post}) (\theta_t - m_t) + \mb(\rho_t,\rho_{\rm post}),
\end{equation}
where $\mA: \PP\times \PP \rightarrow \R^{N_{\theta}\times N_{\theta}}$ and
$\mb: \PP\times\PP \rightarrow \R^{N_{\theta}}$, $\rho_t$ is the law of $\theta_t$ and $m_t$ is the mean under $\rho_t$.
If the law of $\theta_0$ is a Gaussian distribution, then $\theta_t$ 
solving \eqref{eq:MFG} is also Gaussian distributed for any $t>0$; 
thus we can write $\rho_t = \rho_{a_t}$ where $a_t = (m_t, C_t)$ denotes the mean and covariance of the distribution. The evolution of $m_t$ and $C_t$ is given by
\begin{equation}
\label{eq:MFG-2}
\begin{split}
\frac{\mathrm{d} m_t}{\mathrm{d}t} &= \mb(\rho_{\rhoa_t},\rho_{\rm post}),\\
\frac{\mathrm{d} C_t}{\mathrm{d}t} &= \mA(\rho_{\rhoa_t},\rho_{\rm post})C_t + C_t\mA(\rho_{\rhoa_t},\rho_{\rm post})^T.
\end{split}
\end{equation}
\end{lemma}

We provide a proof of this lemma in \cref{proof:Gaussian-interacting-particle}. The lemma allows us to identify the corresponding mean-field dynamics \cref{eq:MFG} of the Gaussian approximate gradient flow \cref{eqn: Gaussian approximate gradient flow}.
Furthermore the evolution equation \eqref{eq:MFG-2} for the mean and covariance is
defined via
vector field for the evolution defined by expectations under the Gaussian with this
mean and covariance. We will elaborate on this identification 
in detail for specific metric tensors in later subsections.
Regarding the affine invariance property of the mean-field equation, we have the following proposition:
    \begin{proposition}
\label{prop: affine invariant mean-field G}
Suppose the mean and covariance evolution equation \cref{eq:MFG-2} is affine invariant, in the sense that under the invertible affine transformation $\tilde{\theta}= \varphi(\theta) = A\theta + b$ and correspondingly $\tilde{m}_t = Am_t, \tilde{C}_t = AC_tA^T, \tilde{\rho}_{a_t} = \varphi \# \rho_{a_t}, \tilde{\rho}_{\rm post} = \varphi \# \rho_{\rm post}$, it holds that
\begin{equation}
\begin{split}
\frac{\mathrm{d} \tilde{m}_t}{\mathrm{d}t} &= \mb(\tilde{\rho}_{\rhoa_t},\tilde{\rho}_{\rm post}),\\
\frac{\mathrm{d} \tilde{C}_t}{\mathrm{d}t} &= \mA(\tilde{\rho}_{\rhoa_t},\tilde{\rho}_{\rm post})\tilde{C}_t + \tilde{C}_t\mA(\tilde{\rho}_{\rhoa_t},\tilde{\rho}_{\rm post})^T.
\end{split}
\end{equation}
Then, the corresponding mean-field equation~\cref{eq:MFG} is also affine invariant.
\end{proposition}
We provide a proof of this proposition in \cref{proof: affine invariant mean-field G}.

%
\subsection{Gaussian Approximate Fisher-Rao Gradient Flow} 
\label{ssec:Gaussian-Fisher-Rao}
\subsubsection{Metric}
In the Gaussian density space, where $\rho$ is parameterized by $\rhoa = [m, C] \in \R^{N_{\rhoa}}$, the induced Fisher-Rao metric tensor $\fM(\rhoa) \in \R^{N_{\rhoa}\times N_{\rhoa}}$ has entries
\begin{align}
\label{eq:Fisher-Rao-param}
    \fM(\rhoa)_{jk} 
    &= \int \frac{\partial \log \rho_{\rhoa}(\theta)}{\partial \rhoa_j}  \frac{\partial \log \rho_{\rhoa}(\theta)}{\partial \rhoa_k} \rho_{\rhoa}(\theta)\mathrm{d}\theta .
\end{align}
This is also the Fisher information matrix, which has an explicit formula in the Gaussian space (e.g., see \cite{malago2015information}):
\begin{align}
\label{eq:Fisher-Rao-param-2}
    \fM(\rhoa)_{jk} 
    &= \frac{\partial m}{\partial \rhoa_j}^TC^{-1}\frac{\partial m}{\partial \rhoa_k} + \frac{1}{2}{\rm tr}\Bigl(C^{-1}\frac{\partial C}{\partial \rhoa_j}C^{-1}\frac{\partial C}{\partial \rhoa_k}\Bigr).
\end{align}

\subsubsection{Flow Equation}
The moment closure approach in \Cref{ssec:moment} delivers the evolution of
the mean and covariance by virtue of \cref{prop:2}. Applying the moment closure approach to \eqref{eq:mean-field-Fisher-Rao} leads to the following equations:
\begin{equation*}
\begin{aligned}
\frac{\mathrm{d} m_t}{\mathrm{d} t} & = \int \theta \left(\rho_{a_t} \bigl(\log \rho_{\rm post} - \log \rho_{a_t}\bigr) - \rho_{a_t}\E_{\rho_{a_t}}[\log \rho_{\rm post} - \log \rho_{a_t}] \right) {\rm d}\theta \\
&= \Cov_{\rho_{\rhoa_t}}[\theta ,\, \log \rho_{\rm post} - \log \rho_{\rhoa_t} ]\\ & = C_t\E_{\rho_{\rhoa_t}}[\nabla_\theta (\log \rho_{\rm post} - \log \rho_{\rhoa_t}) ],\\
\frac{\mathrm{d} C_t}{\mathrm{d} t} &= \E_{\rho_{\rhoa_t}}[(\theta -m_t)(\theta -m_t)^T(\log \rho_{\rm post} - \log \rho_{\rhoa_t} - \E_{\rho_{\rhoa_t}}[\log \rho_{\rm post} - \log \rho_{\rhoa_t}])]\\
&= C_t \E_{\rho_{\rhoa_t}}[\Hess(\log \rho_{\rm post} - \log \rho_{\rhoa_t})] C_t,
\end{aligned}
\end{equation*}
where $\rho_{\rhoa_t} \sim \N(m_t, C_t)$, and we have used the Stein's lemma (\cref{lemma: stein equality}) in the above derivation. Furthermore noting that $\E_{\rho_{a_t}}[\nabla \log \rho_{a_t}] = 0$, we obtain
\begin{equation}
\begin{aligned}
\label{eq:Gaussian Fisher-Rao}
\frac{\mathrm{d} m_t}{\mathrm{d} t}  & =  C_t\E_{\rho_{\rhoa_t}}[\nabla_\theta \log \rho_{\rm post} ],\\
\frac{\mathrm{d} C_t}{\mathrm{d} t}
&= C_t + C_t \E_{\rho_{\rhoa_t}}[\Hess \log \rho_{\rm post}]C_t.
\end{aligned}
\end{equation}

\begin{newremark}
Returning to \Cref{rem:homotopy} and assuming a quadratic negative log likelihood
\begin{align}
    \VL (\theta) = \frac{1}{2}(H\theta - y)^{\rm T}R^{-1}(H\theta-y)
\end{align}
and a Gaussian prior distribution, the functional (\ref{eq:Bayesian_U}) leads to the following Fisher-Rao gradient flow equations
\begin{equation}
\begin{aligned}
\label{eq:Gaussian Fisher-Rao-Kalman}
\frac{\mathrm{d} m_t}{\mathrm{d} t}  & =  -C_tH^{\rm T} R^{-1} (Hm_t-y),\\
\frac{\mathrm{d} C_t}{\mathrm{d} t}
&= -C_tH^{\rm T} R^{-1} H C_t,
\end{aligned}
\end{equation}
for the mean $m_t$ and covariance matrix $C_t$. These define the well-known 
Kalman-Bucy filter \cite{kalman1961new} from linear state estimation; see the text \cite{jazwinski2007stochastic}
for further details. Their extension to general log-likelihood functions $\VL$ under Gaussian approximation is discussed in \cite{pidstrigach2021affine}. 

\Cref{eq:Gaussian Fisher-Rao} also corresponds to the gradient flow under the finite dimensional Fisher-Rao metric 
in the parameter space~\cite{opper2009variational,khan2018fast}; in this context,
it goes by the nomenclature \emph{natural gradient flow}~\cite{amari1998natural,martens2020new,zhang2019fast}. The connection between Fisher-Rao natural gradient methods and Kalman filters has been studied in  \cite{ollivier2018online, ollivier2019extended}.
\end{newremark}

\subsubsection{Affine Invariance}
\Cref{eq:Gaussian Fisher-Rao} is affine invariant.  

\subsubsection{Mean-Field Dynamics}
Using~\cref{lemma:Gaussian-interacting-particle} we can read off a
choice of the pair $\mA,\mb$ defining the mean-field equation for the Gaussian approximate 
Fisher-Rao gradient flow (\ref{eq:Gaussian Fisher-Rao}); we obtain
\begin{equation}
\label{eq:Gaussian Fisher-Rao MD}
    \frac{\mathrm{d}\theta_t}{\mathrm{d}t} = \frac{1}{2}\Bigl[\I + C_t \E_{\rho_{\rhoa_t}}[\Hess \log \rho_{\rm post}]\Bigr](\theta_t - m_t) + C_t\E_{\rho_{\rhoa_t}}[\nabla_\theta \log \rho_{\rm post} ].
\end{equation}
Following~\cref{prop: affine invariant mean-field G}, the mean-field equation \cref{eq:Gaussian Fisher-Rao MD} is also affine invariant.

%
\subsection{Gaussian Approximate Wasserstein Gradient Flow}
\label{ssec:Gaussian-Wasserstein}
\subsubsection{Metric}
\label{sssec: Gaussian-Wasserstein metric}
Recall the preconditioner 
$\Prec: \R^{N_\theta}\times \PP \to \R^{N_\theta\times N_\theta}_{>0},$
where the output space is the cone of positive-definite symmetric matrices.
In the Gaussian density space,  where $\rho$ is parameterized by $\rhoa \in \R^{N_{\rhoa}}$, the preconditioned Wasserstein metric tensor $\fM(\rhoa) \in \R^{N_{\rhoa}\times N_{\rhoa}}$ has entries 
\begin{align}
\label{eq:wass-param}
    \fM(\rhoa)_{jk} = \int \psi_j(\theta) \frac{\partial \rho_{\rhoa}(\theta)}{\partial \rhoa_k} \mathrm{d}\theta 
    \quad\textrm{where}\quad -\nabla_{\theta}\cdot\Bigl(\rho_{\rhoa}(\theta)\Prec(\theta, \rho_{\rhoa})\nabla_{\theta} \psi_j(\theta)\Bigr) = \frac{\partial \rho_{\rhoa} }{\partial \rhoa_j}(\theta).
\end{align}
When $P(\theta,\rho_a)$ is the identity operator, the metric tensor $\fM(\rhoa)$ has an explicit formula \cite{chen2018natural, takatsu2011wasserstein, Luigi2018, bhatia2019bures, li2023}, and the corresponding Gaussian density space is called the Bures–Wasserstein space \cite{villani2021topics}.
 
\subsubsection{Flow Equation}
Again we can use the moment closure approach from \Cref{ssec:moment}, which is shown
to apply here in \cref{prop:2}. By applying the moment closure approach to \eqref{eq:mean-field-Wasserstein}, we get the mean and covariance evolution equations 
for the Gaussian approximate Wasserstein gradient flow with 
$\theta-$independent $\Prec$ as follows:
\begin{equation}
\label{eq:Wasserstein-mC}
\begin{split}
      \frac{\mathrm{d} m_t}{\mathrm{d}t} &= \int\Bigl[\rho_{\rhoa_t} \Prec(\rho_{\rhoa_t}) \nabla_{\theta}(\log \rho_{\rm post} - \log \rho_{\rhoa_t}) \Bigr] \mathrm{d}\theta =  \Prec(\rho_{\rhoa_t}) \E_{\rho_{\rhoa_t}}[\nabla_\theta \log \rho_{\rm post} ], \\
    \frac{\mathrm{d} C_t}{\mathrm{d}t} &= \int\nabla_{\theta} \cdot \Bigl[\rho_{\rhoa_t} \Prec(\rho_{\rhoa_t}) \nabla_{\theta} (\log \rho_{\rhoa_t}  - \log \rho_{\rm post})\Bigr] (\theta - m_t)(\theta - m_t)^T\mathrm{d}\theta \\
    &=  2\Prec(\rho_{\rhoa_t}) +  \Prec(\rho_{\rhoa_t}) \E_{\rho_{\rhoa_t}}[\Hess\log\rho_{\rm post}]C_t +  C_t\E_{\rho_{\rhoa_t}}[\Hess\log\rho_{\rm post}] \Prec(\rho_{\rhoa_t}), 
\end{split}
\end{equation}
where $\rho_{\rhoa_t} \sim \N(m_t, C_t)$, and we have used integration by parts and Stein's lemma (\cref{lemma: stein equality}), and the fact $\E_{\rho_{a_t}}[\nabla \log \rho_{a_t}] = 0$ in the above derivation.

When we set $\Prec(\rho) \equiv \I$, the 
evolution equation becomes
\begin{equation}
\begin{aligned}   
\label{eq:Gaussian-Wasserstein}
&\frac{\mathrm{d}m_t}{\mathrm{d}t} &&= \E_{\rho_{\rhoa_t}} \bigl[ \nabla_{\theta}  \log \rho_{\rm post}(\theta_t)  \bigr],\\
&\frac{\mathrm{d}C_t}{\mathrm{d}t} &&= 2\I + \E_{\rho_{\rhoa_t}}\bigl[\Hess\log \rho_{\rm post}(\theta_t) \bigr]C_t + C_t\E_{\rho_{\rhoa_t}}\bigl[\Hess\log \rho_{\rm post}(\theta_t) \bigr].
\end{aligned}
\end{equation}
This corresponds to the gradient flow under the constrained Wasserstein metric in Gaussian density space~\cite{sarkka2007unscented,julier1995new,lambert2022variational}.

\subsubsection{Affine Invariance}
We now allow the preconditioner to depend on
$\rho_{a_t}$ and set $\Prec(\rho_{a_t}) = C_t$. This choice
satisfies the affine invariant 
condition~\cref{proposition:Wasserstein-affine-invariant}. The 
resulting evolution equations for the corresponding mean and covariance are
\begin{equation}
\label{eq:Gaussian-AI-Wasserstein}
\begin{aligned}    
&\frac{\mathrm{d}m_t}{\mathrm{d}t} &&= C_t\E_{\rho_{\rhoa_t}} \bigl[ \nabla_{\theta}  \log \rho_{\rm post}(\theta_t)  \bigr],\\
&\frac{\mathrm{d}C_t}{\mathrm{d}t} &&= 2C_t + 2C_t \E_{\rho_{\rhoa_t}}\bigl[\Hess \log \rho_{\rm post}(\theta_t)\bigr]C_t.
\end{aligned}
\end{equation}
Equation \cref{eq:Gaussian-AI-Wasserstein} is similar to the Gaussian approximate Fisher-Rao gradient flow \cref{eq:Gaussian Fisher-Rao}, but with scaling factor $2$ in the covariance evolution. 

\subsubsection{Mean-Field Dynamics}
Employing~\cref{lemma:Gaussian-interacting-particle} we can again identify 
a pair $\mA,\mb$ leading to a mean-field equation for  the Gaussian approximate Wasserstein gradient flow (\ref{eq:Wasserstein-mC}) with $\theta$-independent $P$:
\begin{equation}
\label{eq:Gaussian-AI-Wasserstein MD}
    \frac{\mathrm{d}\theta_t}{\mathrm{d}t} = \Bigl[C_t^{-1} + \Prec(\rho_{\rhoa_t}) \E_{\rho_{\rhoa_t}}[\Hess\log\rho_{\rm post}]\Bigr](\theta_t - m_t) + P(\rho_{\rhoa_t})\E_{\rho_{\rhoa_t}}[\nabla_\theta \log \rho_{\rm post} ].
\end{equation}
From~\cref{prop: affine invariant mean-field G}, its corresponding mean-field equation \cref{eq:Gaussian-AI-Wasserstein MD} with $P(\rho_{\rhoa_t}) = C_t$ is also affine invariant.

\subsection{Gaussian Approximate Stein Gradient Flow}
\label{ssec:Gaussian-Stein}
\subsubsection{Metric}
We work in the general preconditioned setting, as in the Wasserstein case.
In the Gaussian density space,   where $\rho$ is parameterized by $\rhoa \in \R^{N_{\rhoa}}$, the Stein metric tensor $\fM(\rhoa) \in \R^{N_{\rhoa}\times N_{\rhoa}}$ is
\begin{equation}
\begin{aligned}
\label{eq:Stein-param}
    &\fM(\rhoa)_{jk} = \int \psi_j(\theta) \frac{\partial \rho_{\rhoa}(\theta)}{\partial \rhoa_k} \mathrm{d}\theta, \quad \textrm{where}
    \\
    &- \nabla_{\theta}\cdot\Bigl(\rho_{\rhoa}(\theta)\int \kappa(\theta,\theta',\rho_\rhoa)\rho_{\rhoa}(\theta')\Prec( \theta, \theta',\rho_{\rhoa}(\theta), \rho_{\rhoa}(\theta'))\nabla_{\theta} \psi_j(\theta')\mathrm{d}\theta' \Bigr) = \frac{\partial \rho_{\rhoa} }{\partial \rhoa_j}(\theta). 
\end{aligned}
\end{equation}

\subsubsection{Flow Equation}
We consider the setting in which $\Prec=\Prec(\rho_{\rhoa})$ is 
independent of $\theta$, and we choose  bilinear kernel
\begin{equation}
\label{eq:bil}
\kappa(\theta, \theta', \rho) = (\theta - m)^TA(\rho)(\theta' - m) + b(\rho).
\end{equation}
where $A : \PP \rightarrow \R^{N_{\theta}\times N_{\theta}}$, $b: \PP \rightarrow 
\R$ and $m$ is the mean under $\rho.$

In the following let $P_t:=P(\rho_{\rhoa_t})$, evaluate $A$ and $b$
at $\rho=\rho_{\rhoa_t}$, writing the resulting time-dependent
matrix- and vector-valued functions as  $A_t$ and $b_t$
where $A_\cdot : \R \rightarrow \R^{N_{\theta}\times N_{\theta}}$ and $b_\cdot : \R \rightarrow \R^{N_{\theta}}$,
and let $m_t$ denote mean under $\rho_{\rhoa_t}$ so that $m_\cdot: \R \rightarrow
\R^{N_{\theta}}$.
We apply the moment closure approach from \Cref{ssec:moment} to \eqref{eq:Stein-GF}. The mean and covariance evolution equations of the preconditioned Stein gradient flow with bilinear kernel \cref{eq:bil}
are
\begin{equation}
\begin{aligned}   
\label{eq:Gaussian-Stein-CTL0}
&\frac{\mathrm{d}m_t}{\mathrm{d}t} = -\int\Bigl(\rho_{\rhoa_t}(\theta)\int \kappa(\theta,\theta',\rho_{\rhoa_t})\rho_{\rhoa_t}(\theta')\Prec_t \nabla_{\theta'} \bigl(\log\rho_{\rhoa_t}(\theta') - \log \rho_{\rm post}(\theta') \bigr)\mathrm{d}\theta' \Bigr) \mathrm{d}\theta, \\
&\frac{\mathrm{d}C_t}{\mathrm{d}t} =\\
&-\int \Bigl(\rho_{\rhoa_t}(\theta)\int \kappa(\theta,\theta',\rho_{\rhoa_t})\rho_{\rhoa_t}(\theta')\Prec_t \nabla_{\theta'} \bigl(\log\rho_{\rhoa_t}(\theta') - \log \rho_{\rm post}(\theta') \bigr)\mathrm{d}\theta' \Bigr) (\theta - m_t)^T \\
&+ (\theta - m_t)\Bigl(\rho_{\rhoa_t}(\theta)\int \kappa(\theta,\theta',\rho_{\rhoa_t})\rho_{\rhoa_t}(\theta')\Prec_t \nabla_{\theta'} \bigl(\log\rho_{\rhoa_t}(\theta') - \log \rho_{\rm post}(\theta') \bigr)\mathrm{d}\theta' \Bigr)^T \mathrm{d}\theta. \\
\end{aligned}
\end{equation}
where $\rho_{\rhoa_t} \sim \N(m_t, C_t)$, and we have used integration by parts in the above derivation. Imposing the form of the bilinear
kernel \cref{eq:bil} and using the Stein's lemma (\cref{lemma: stein equality}), and the fact $\E_{\rho_{a_t}}[\nabla \log \rho_{a_t}] = 0$, we obtain
\begin{equation}
\begin{aligned}   
\label{eq:Gaussian-Stein-CTL}
\frac{\mathrm{d}m_t}{\mathrm{d}t} 
=&b_tP_t \E_{\rho_{\rhoa_t}}[\nabla_{\theta}\log\rho_{\rm post}]
,\\
\frac{\mathrm{d}C_t}{\mathrm{d}t} 
=& P_tA_tC_t + C_tA_tP_t\\
&\quad\quad\quad+ P_t\E_{\rho_{\rhoa_t}}[\Hess\log\rho_{\rm post}]C_tA_tC_t +
C_tA_tC_t\E_{\rho_{\rhoa_t}}[\Hess\log\rho_{\rm post}]P_t.
\end{aligned}
\end{equation}

\begin{newremark}
Different choices of the preconditioner $\Prec$ and the bilinear kernel $\kappa$ 
allows us to recover different Gaussian variational inference methods
appearing in the literature.
Choosing the preconditioner $\Prec_t = \I$ and bilinear kernel \cref{eq:bil} 
with $A_t=C_t^{-1}$ and $b_t=1$ recovers the Gaussian approximate Wasserstein gradient flow~\cref{eq:Gaussian-Wasserstein}.
Setting the preconditioner $\Prec_t = \I$ and bilinear kernel \cref{eq:bil} 
with $A_t=I$ and $b_t=1$
recovers the Gaussian sampling approach introduced in \cite{galy2021flexible}:
\begin{equation}
\label{eq:galy2021flexible}
\begin{aligned}
\frac{\mathrm{d}m_t}{\mathrm{d}t}
=&\E_{\rho_{\rhoa_t}}[\nabla_{\theta}\log\rho_{\rm post}]
,\\
\frac{\mathrm{d}C_t}{\mathrm{d}t}
=& 2C_t + \E_{\rho_{\rhoa_t}}[\Hess\log\rho_{\rm post}]C_t^2 +
C_t^2\E_{\rho_{\rhoa_t}}[\Hess\log\rho_{\rm post}].
\end{aligned}
\end{equation}
\end{newremark}

\subsubsection{Affine Invariance}
Recalling that $\rho_{\rhoa_t} \sim \N(m_t, C_t)$
and setting $\Prec_t=\Prec(\rho_{a_t}) = C_t$
and choosing bilinear kernel \cref{eq:bil} with $A_t=\frac{1}{2}C_t^{-1}$ and $b_t=1$, which satisfies the affine invariant condition~\cref{proposition:Stein-affine-invariant}, leads to the Gaussian approximate Fisher-Rao gradient flow~\cref{eq:Gaussian Fisher-Rao}. 

\subsubsection{Mean-Field Dynamics}
Using~\cref{lemma:Gaussian-interacting-particle} we can deduce that 
the Gaussian approximate Stein gradient flow (\ref{eq:Gaussian-Stein-CTL0}) with $\theta$-independent $P$ has the following mean-field equation:
\begin{equation*}
    \frac{\mathrm{d}\theta_t}{\mathrm{d}t} = \Bigl[P_tA_t + P_t\E_{\rho_{\rhoa_t}}[\Hess\log\rho_{\rm post}]C_t A_t\Bigr](\theta_t - m_t) + b_tP_t\E_{\rho_{\rhoa_t}}[\nabla_\theta \log \rho_{\rm post} ].
\end{equation*}
Here $P_t = P(\rho_{\rhoa_t})$.
From~\cref{prop: affine invariant mean-field G}, we know that this
corresponding mean-field equation is also affine invariant.

\subsection{Convergence to Steady State}
\label{ssec:GGF-converge}

Recall that the objective of Gaussian variational inference is to solve the
minimization problem \cref{eq:Gaussian-KL0}. Furthermore all critical points 
satisfy \cref{eq:Gaussian-KL}. Regular gradient descent, in metric defined
via the Euclidean inner-product in $\R^{N_a}$ (i.e., setting $\fM(\rhoa_t)=I$ in \eqref{eq:G-GF}) will give rise to the dynamical system
\begin{equation}
\label{eq:g-gf}
\begin{aligned}
\frac{\mathrm{d}m_t}{\mathrm{d}t} &= \E_{\rho_{\rhoa_t}}\bigl[\nabla_\theta  \log \rho_{\rm post}(\theta_t)   \bigr],  \\
\frac{\mathrm{d}C_t}{\mathrm{d}t} &= \frac{1}{2}C_t^{-1} + \frac{1}{2}\E_{\rho_{\rhoa_t}}\bigl[\Hess\log \rho_{\rm post}(\theta_t) \bigr].
\end{aligned}
\end{equation}
Note that steady states of this dynamical system necessarily 
satisfy \cref{eq:Gaussian-KL}. In the preceding subsections we
have derived a number of different gradient flows in the manifold
of Gaussian densities, including the Gaussian approximate Fisher-Rao gradient flow~\cref{eq:Gaussian Fisher-Rao} and the Gaussian approximate Wasserstein 
gradient flow~\cref{eq:Gaussian-Wasserstein}; note that both of these
dynamical systems also necessarily satisfy \cref{eq:Gaussian-KL} in steady  
state. The convergence properties of \cref{eq:Gaussian-AI-Wasserstein} obtained by the affine-invariant Wasserstein gradient flow are similar to those of the Gaussian approximate Fisher-Rao gradient flow~\cref{eq:Gaussian Fisher-Rao}. We omit 
detailed discussion from the paper to avoid redundant discussions.
In this subsection, we survey and study the convergence of these aforementioned Gaussian approximate gradient flows in three settings: the Gaussian posterior case; logconcave posterior case; and the general posterior case.

\subsubsection{Gaussian Posterior Case} 
\label{ssec:gaussian}
Assume the posterior distribution \cref{eq:posterior} is Gaussian so that
$\rho_{\rm post}(\theta)\sim \N(m_{\star}, C_{\star})$ where
\begin{align}
\label{e:Gpost}    
\V(\theta) = \frac{1}{2}(x - m_{\star})^TC_{\star}^{-1}(x - m_{\star}).
\end{align}
\begin{proposition}
\label{proposition:analytical-solution-ngd}
Consider the posterior distribution \cref{eq:posterior}
under assumption \cref{e:Gpost} so that the posterior is Gaussian. Then the Gaussian approximate Fisher-Rao gradient flow~\cref{eq:Gaussian Fisher-Rao} has the analytical solution
\begin{subequations}
\begin{align} 
&m_t = m_{\star} + e^{-t}\Bigl((1-e^{-t})C_{\star}^{-1} + e^{-t}C_0^{-1}\Bigr)^{-1} C_0^{-1} \bigl(m_0 - m_{\star}\bigr),  \label{e:mtCt}   \\
&C_t^{-1} = C_{\star}^{-1} + e^{-t}\bigl(C_0^{-1}  -  C_{\star}^{-1}\bigr).
\end{align}
\end{subequations}
\end{proposition}
The proof is in~\cref{proof:analytical-solution-ngd}. We remark that both mean and covariance converge exponentially fast to $m_{\star}$ and $C_{\star}$ with convergence rate $\bigO(e^{-t})$. This rate is independent of $C_{\star}$. The uniform convergence rate $\bigO(e^{-t})$ of the Gaussian approximate affine-invariant Wasserstein gradient flow~\cref{eq:Gaussian-AI-Wasserstein} is obtained in~\cite[Lemma 3.2]{garbuno2020interacting} for Gaussian initial data, and extended to general initial data in \cite{carrillo2021wasserstein}.

For the Gaussian approximate gradient flow~\cref{eq:g-gf} and the Gaussian approximate Wasserstein gradient flow~\cref{eq:Gaussian-Wasserstein}, if the norm of $C_{\star}$ is large, their convergence rate is much slower than the Gaussian approximate Fisher-Rao gradient flow. Indeed, we have the following convergence result:
\begin{proposition}\label{p:rate}
Consider the posterior distribution \cref{eq:posterior}
under assumption \cref{e:Gpost} so that the posterior is Gaussian; this
posterior is the unique minimizer of the Gaussian variational inference problem~\cref{eq:Gaussian-KL0}.
Denote the largest eigenvalue of $C_{\star}$ by $\la$. 
For gradient flows with initialization $C_0=\lambda_0\I$, the following hold:
\begin{enumerate}
\item for the Gaussian approximate gradient flow~\cref{eq:g-gf}: \[\|m_t-m_{\star}\|_2=\bigO(e^{-t/\la}), \|C_t-C_{\star}\|_2=\bigO(e^{-t/(2\la^2)});\]
\item for the Gaussian approximate Fisher-Rao gradient flow \cref{eq:Gaussian Fisher-Rao}: 
\[\|m_t-m_{\star}\|_2=\bigO(e^{-t}), \|C_t-C_{\star}\|_2=\bigO(e^{-t});\]
\item for the Gaussian approximate Wasserstein gradient flow \cref{eq:Gaussian-Wasserstein}: 
\[\|m_t-m_{\star}\|_2=\bigO(e^{-t/\la}), \|C_t-C_{\star}\|_2=\bigO(e^{-2t/\la}),\]
\end{enumerate}
where the implicit constants depend on $m_{\star}$, $C_{\star}$ and $\lambda_0$.
\end{proposition}
The proof is in~\cref{proof:p:rate}.

\subsubsection{Logconcave Posterior Case} 

\label{ssec:logconcave}
In this subsection, we consider the case that the posterior distribution $\rho_{\rm post}(\theta)$ given by \cref{eq:posterior} is strongly log-concave.

\begin{proposition}\label{p:logconcave}
Assume that the posterior distribution $\rho_{\rm post}(\theta)$ is $\alpha$-strongly logconcave~(\cref{assumption:logconcave}) and that $-\Hess \log \rho_{\rm post} \preceq \beta \I$.
Assume further that the initial covariance matrix  satisfies $\lambda_{0,\min}\I\preceq C_0\preceq \lambda_{0,\max}\I$. Then for the dynamics \cref{eq:g-gf}, the Gaussian approximate Fisher-Rao gradient flow \cref{eq:Gaussian Fisher-Rao}, and the Gaussian approximate Wasserstein gradient flow \cref{eq:Gaussian-Wasserstein}, we have 
\begin{align}\label{e:KLconv}
    {\mathrm{KL}}\Bigl[\rho_{\rhoa_t}\Big\Vert  \rho_{\rm post}\Bigr]
    \leq e^{-K t}{\mathrm{KL}}\Bigl[\rho_{\rhoa_0}\Big\Vert  \rho_{\rm post}\Bigr]+(1-e^{-K t})
     {\mathrm{KL}}\Bigl[\rho_{\rhoa_\star}\Big\Vert  \rho_{\rm post}\Bigr],
\end{align}
where $\rho_{\rhoa_0}\sim \N(m_{0}, C_{0})$ is the initial condition and $\rho_{\rhoa_\star}\sim \N(m_{\star}, C_{\star})$ is the unique global minimizer of 
\cref{eq:Gaussian-KL0}. The rate constant $K$ depends on $\alpha, \beta, \lambda_{0,\min}, \lambda_{0,\max}$.
Specifically, we have: 
\begin{itemize}
    \item $K=2\alpha/\max\{1, 4/\alpha, 4\lambda_{0,\max}\} $ for the Gaussian approximate gradient flow \cref{eq:g-gf};
    \item $K=\alpha\min\{1/\beta, \lambda_{0,\min}\}$ for the Gaussian approximate Fisher-Rao gradient flow \cref{eq:Gaussian Fisher-Rao};
    \item $K = 2\alpha$ for the Gaussian approximate Wasserstein gradient flow \cref{eq:Gaussian-Wasserstein}.
\end{itemize}
In addition, we also have that $\rho_{\rhoa_t}$ converges to $\rho_{\rhoa_\star}$ exponentially fast in terms of the Wasserstein metric:
\begin{align}\label{e:W2distance}
    W^2_2(\rho_{\rhoa_t}, \rho_{\rhoa_\star})\leq  \frac{2e^{-Kt}}{\alpha}\left({\mathrm{KL}}[\rho_{\rhoa_0} \Vert  \rho_{\rm post}]-{\mathrm{KL}}[\rho_{\rhoa_\star} \Vert  \rho_{\rm post}]\right).
\end{align}

\end{proposition}
The proof is in~\cref{proof:a:rate}, and is inspired by the work \cite{lambert2022variational}.

\begin{remark}
In the above result, the rate constants $K$ for the Gaussian approximate gradient flow \cref{eq:g-gf} and the Gaussian approximate Fisher-Rao gradient flow~\cref{eq:Gaussian Fisher-Rao} depend on the eigenvalues of the initial covariance matrix. We conjecture that this dependence \sr{is caused by our proof techniques} and may be removed with a better proof strategy.  Furthermore, if we can initialize $C_0$ such that $\frac{1}{\beta} \preceq C_0 \preceq \frac{1}{\alpha}$, such dependence can be directly eliminated in the above bounds. 

We also observe that the bound on the rate constant $K$ for the Gaussian approximate Fisher-Rao gradient flow depends on $\beta/\alpha$.
In some cases, affine transformation may be applied to reduce $\beta/\alpha$, since the Gaussian approximate Fisher-Rao gradient flow~\cref{eq:Gaussian Fisher-Rao} is affine invariant. As an example, in the case of Gaussian posteriors, the rate constant can be reduced to $1$, as in \Cref{p:rate}.
\end{remark}

On the other hand, we have the following proposition about the lower bound of the local convergence rate of the Gaussian approximate Fisher-Rao gradient flow \cref{eq:Gaussian Fisher-Rao}.

\begin{proposition}\label{p:logconcave-local}
Assume the posterior distribution $\rho_{\rm post}(\theta)$ is $\alpha$-strongly logconcave~(\cref{assumption:logconcave}) and that $-\Hess \log \rho_{\rm post} \preceq \beta \I$.
Denote the unique minimizer of the Gaussian variational inference problem~\cref{eq:Gaussian-KL0} by $\rho_{\rhoa_\star} :=  \N(m_{\star}, C_{\star})$.
For $N_\theta = 1$, let $\lambda_{\star,\max} < 0$ denote the largest eigenvalue of the linearized Jacobian matrix of the Gaussian approximate Fisher-Rao gradient flow \cref{eq:Gaussian Fisher-Rao} around $m_\star$ and $C_\star$; this number
determines the local convergence rate of the Gaussian approximate Fisher-Rao 
gradient flow \cref{eq:Gaussian Fisher-Rao}. 
Then we have
\begin{align}
\label{eq:lambda_max}
- \lambda_{\star,\max} \geq \frac{1}{(7+\frac{4}{\sqrt{\pi}})\bigl(1 + \log(\frac{\beta}{\alpha})\bigr)}. 
\end{align}
Moreover, the bound is sharp: it is possible to construct a sequence of triplets $\rho_{{\rm post},n}$, $\alpha_n$ and $\beta_n$, satisfying $\lim_{n \to \infty} \frac{\beta_n}{\alpha_n} = \infty$, such that, if we let $\lambda_{\star,\max,n}$ denote the corresponding largest eigenvalues of the linearized Jacobian matrix for 
the $n$-th triple, then it holds that
\[-\lambda_{\star,\max,n} = \bigO\left(1/\log \frac{\beta_n}{\alpha_n}\right).\]
\end{proposition}

The proof can be found in~\cref{proof:proof-counter-example-Gaussian-Fisher-Rao}. For the counterexample, the constructed posterior $\rho_{{\rm post},n}$ is built by designing  $-\nabla_{\theta}\nabla_{\theta}\rho_{{\rm post},n}$ to be a bump function between $\alpha$ and $\beta$ with width gradually approaching $0$ with $n$.

\subsubsection{General Posterior Case}
\label{ssec:general}
The previous sections consider Gaussian and then logconcave posteriors; for 
all three Gaussian approximate gradient flows we demonstrate exponential convergence, and for the Fisher-Rao based methodology we have some invariance of the rates of convergence with respect to the conditioning of the problem. In this subsection, however, we construct counterexamples showing that, for general posteriors, the convergence of all three Gaussian approximate gradient flows to a stationary point can be arbitrarily slow. 

\begin{proposition}
\label{proposition:counter-example-slow}
For any $K > 0$ there exist a target $\rho_{\rm post}$ such that, for the three Gaussian approximate gradient flows~\cref{eq:Gaussian Fisher-Rao}, 
\cref{eq:Gaussian-Wasserstein} and \cref{eq:g-gf}, the convergence to 
their stationary points can be as slow as $\bigO(t^{-\frac{1}{2K}})$. 
\end{proposition}

The proof is in~\cref{proof:counter-example-slow}.

%
%

\section{Numerical Experiments}
\label{sec:Numerics}

In this section, we perform numerical experiments to study the behavior of the aforementioned gradient flows for sampling, which complements our theoretical study. We observe the following:

\begin{itemize}
    \item In the probability density space, affine invariant gradient flows  outperform their non-affine invariant counterparts, for the Gaussian posterior case~(\cref{fig:Gaussian_gd_particle_converge}), the logconcave posterior case~(\cref{fig:Logconcave_gd_particle_converge}), and the general posterior case~(\cref{fig:Rosenbrock_gd_particle_converge}).
    \item In the restricted Gaussian density space, affine invariant gradient flows  outperform their non-affine invariant counterparts, for the Gaussian posterior case~(\cref{fig:Gaussian_gd_converge}), the logconcave posterior case~(\cref{fig:Logconcave_gd_converge}), and the general posterior case~(\cref{fig:Rosenbrock_gd_converge}).
    \item For general non-Gaussian posteriors, the convergence rates of all gradient flows deteriorate when the posterior becomes more anisotropic~(\cref{fig:Rosenbrock_gd_particle_converge}); 
consequently accurately estimating the summary statistics of these posteriors 
is challenging.

    \item The convergence curves from the use of affine invariant Wasserstein gradient flow, implemented with Langevin dynamics, oscillate slightly due to the added noise; those obtained from affine invariant Stein gradient flow, implemented by Stein variational gradient descent, are smooth~(See \cref{fig:Gaussian_gd_particle_converge,fig:Logconcave_gd_particle_converge,fig:Rosenbrock_gd_particle_converge}). However, the added noise helps for sampling non-Gaussian, highly anisotropic posteriors in comparison with the affine invariant Stein variational gradient descent~(See \cref{fig:Rosenbrock_gd_particle_density,fig:Rosenbrock_gd_particle_converge}).
\end{itemize}

\vspace{0.1in}

In the following subsections, we first introduce all test problems in \Cref{ssec:test_problems} and the setup for the numerical methods in \Cref{ssec:method-setup}. Then we present numerical results for the Gaussian posterior case 
in \Cref{ssec:Gaussian-numerics}, the logconcave posterior case in \Cref{ssec:Logconcave-numerics}, and the general posterior case in \Cref{ssec:General-numerics}. Our code is accessible online:
\begin{center}
  \url{https://github.com/Zhengyu-Huang/InverseProblems.jl}
\end{center}

\subsection{Overview of Test Problems}
\label{ssec:test_problems}
We focus our experiments on three two-dimensional  posteriors. In defining
them we use the notation $\theta = [\theta^{(1)},\theta^{(2)}]^T$.
\begin{enumerate}
    \item \emph{Gaussian Posterior.} 
        $$
    \V(\theta) = \frac{1}{2}\theta^{T}
    \begin{bmatrix}
      1 & 0\\
      0 & \lambda
    \end{bmatrix}
    \theta
    \quad \textrm{with} \quad \lambda = 0.01,\,0.1,\,1.
    $$
    We initialize the gradient flows from 
    $$
    \theta_0 \sim \N\Bigl(
    \begin{bmatrix}
  10\\
  10
    \end{bmatrix}
, 
\begin{bmatrix}
  \frac{1}{2} & 0\\
  0 & 2
\end{bmatrix}\Bigr).
$$
    \item \emph{Logconcave Posterior}. 
    $$
\V(\theta) = \frac{(\sqrt{\lambda}\theta^{(1)} - \theta^{(2)})^2}{20} +\frac{(\theta^{(2)})^4}{20} \quad \textrm{with} \quad \lambda = 0.01,\,0.1,\,1.
$$
We initialize the gradient flows from 
$
\theta_0 \sim \N\Bigl(
\begin{bmatrix}
  10\\
  10
\end{bmatrix}
, 
\begin{bmatrix}
  4 & 0\\
  0 & 4
\end{bmatrix}\Bigr)
$.

    \item \emph{General Posterior.} 
$$
\V(\theta) = \frac{\lambda( \theta^{(2)} - (\theta^{(1)})^2)^2}{20} + \frac{(1 - \theta^{(1)})^2}{20}  \quad \textrm{with} \quad \lambda = 0.01,\,0.1,\,1.
$$
This example is known as the Rosenbrock function~\cite{goodman2010ensemble}. We initialize the gradient flows from 
$$
\theta_0 \sim \N\Bigl(
\begin{bmatrix}
  0\\
  0
\end{bmatrix}
, 
\begin{bmatrix}
  4 & 0\\
  0 & 4
\end{bmatrix}\Bigr).
$$
\end{enumerate}

The summary statistics that we use to compare the resulting solution with the ground truth are the expectation $\E[\theta]$, the covariance $\Cov[\theta]$, and $\E[\cos(\omega^T \theta + b)]$; in the latter case we randomly draw $\omega\sim\N(0, I)$ and $b\sim \textrm{Uniform}(0, 2\pi)$ and report the average MSE over 20 random draws of $\omega$ and $b$. The ground truths of these summary statistics are evaluated by integrating $\rho_{\rm post}$ numerically (See \cref{sec:integration} for details). 

\subsection{Numerical Method Setup}
\label{ssec:method-setup}
The gradient flows in the nonparametric density space studied in \Cref{sec:GF and AI} are implemented by interacting particle systems with $J =100$ particles. 
    We note that straightforward particle implementations of the Fisher-Rao gradient flow suffer from the immobility of the support. One approach to address this challenge is by adding a transport step to the gradient flow, leading to the Wasserstein-Fisher-Rao gradient flow~\cite{lu2019accelerating,lu2022birth}. In our study, we focus on It\^o type mean field models; thus particle implementations of the Fisher-Rao gradient flow are not included. We will compare the performance of the following nonparametric gradient flows (GFs) in our experiments. 
\begin{itemize}
    \item Wasserstein GF: The Wasserstein gradient flow with $P(\theta,\rho) = I$, which is implemented as stochastic particle dynamics~\cref{eq:particle-Wasserstein} with $h(\theta, \rho) = \sqrt{2}\I$. 
    \item Affine-invariant Wasserstein GF: The affine-invariant Wasserstein gradient flow with $P(\theta,\rho) = C(\rho)$, which is implemented as stochastic particle dynamics \cref{eq:AI-Wasserstein-MD-Ct2}. 
    \item Stein GF: The Stein gradient flow with 
$$P = I, \quad \kappa(\theta, \theta', \rho) = (1 + 4\log(J+1)/N_{\theta})^{N_{\theta}/2}\exp(-\frac{1}{h}\lVert\theta - \theta'\rVert^2),$$
which is implemented as deterministic particle dynamics~\cref{eq:particle-Stein}.
Here $h = \textrm{med}^2/\log(J+1)$ and $\textrm{med}^2$ is the squared median of the pairwise Euclidean distance between the current particles, following~\cite{liu2016stein}.
    \item Affine-invariant Stein GF: The affine-invariant Stein gradient flow with 
$$P = C(\rho), \quad \kappa(\theta, \theta',\rho) =(1 + 2/N_\theta)^{N_{\theta}/2}\exp(-\frac{1}{2N_\theta}(\theta - \theta')^TC(\rho)^{-1}(\theta - \theta')),$$
which is implemented as deterministic particle dynamics~\cref{eq:AI-Stein-MD}.
\end{itemize}
Here the scaling constants\footnote{In the case of the Stein GF, there is no analytical formula for the integral~\cref{eq:kernel-const}. To deal with this, we estimate the scaling constant by replacing 
$\textrm{med}^2\I$ with $N_\theta C(\rho)$, for which we can analytically compute~\cref{eq:kernel-const}.} in the above definition of kernel functions are chosen such that
\begin{align}
\label{eq:kernel-const}
    \int \int \kappa(\theta, \theta', \rho) \N(\theta, m, C) \N(\theta', m, C) \mathrm{d}\theta \mathrm{d}\theta' = 1.
\end{align} 
This choice makes the Stein gradient flows comparable with the Wasserstein 
gradient flows in terms of implementation cost per step.
Since we cannot analytically compute the integral~\cref{eq:kernel-const} 
for the kernel of Stein GF, we estimate its scaling constant by replacing 
$\textrm{med}^2\I$ with $N_\theta C$.

For the gradient flows in the Gaussian density space we consider the three mean and covariant dynamics given in equations \cref{eq:Gaussian Fisher-Rao}, \cref{eq:Gaussian-Wasserstein} and \cref{eq:g-gf}.
The expectations in the evolution equations are calculated by the unscented transform~\cite{julier1997new} with $J =2N_\theta + 1 = 5$ quadrature points. Therefore, 
the Gaussian approximation has considerable speedup in comparison with the 
previously mentioned particle-based sampling approaches, where $J=100$.

\subsection{Gaussian Posterior Case}
\label{ssec:Gaussian-numerics}

The convergence of different gradient flows, according to the three summary statistics, are presented in~\cref{fig:Gaussian_gd_particle_converge} (nonparameteric density space) and in~\cref{fig:Gaussian_gd_converge} (Gaussian density space). In both nonparametric and Gaussian density spaces, the imposition of the affine invariance property makes the convergence rate independent of the anisotropy $\lambda$ and accelerates the sampling for badly scaled Gaussian ($\lambda = 0.01$).
However, all these gradient flows in the nonparametric density space 
do not converge within machine precision because of the limited number of particles.
The convergence rates of Gaussian approximate gradient flows match well with the
predictions of~\cref{p:rate}.

\begin{figure}[ht]
\centering
    \includegraphics[width=0.8\textwidth]{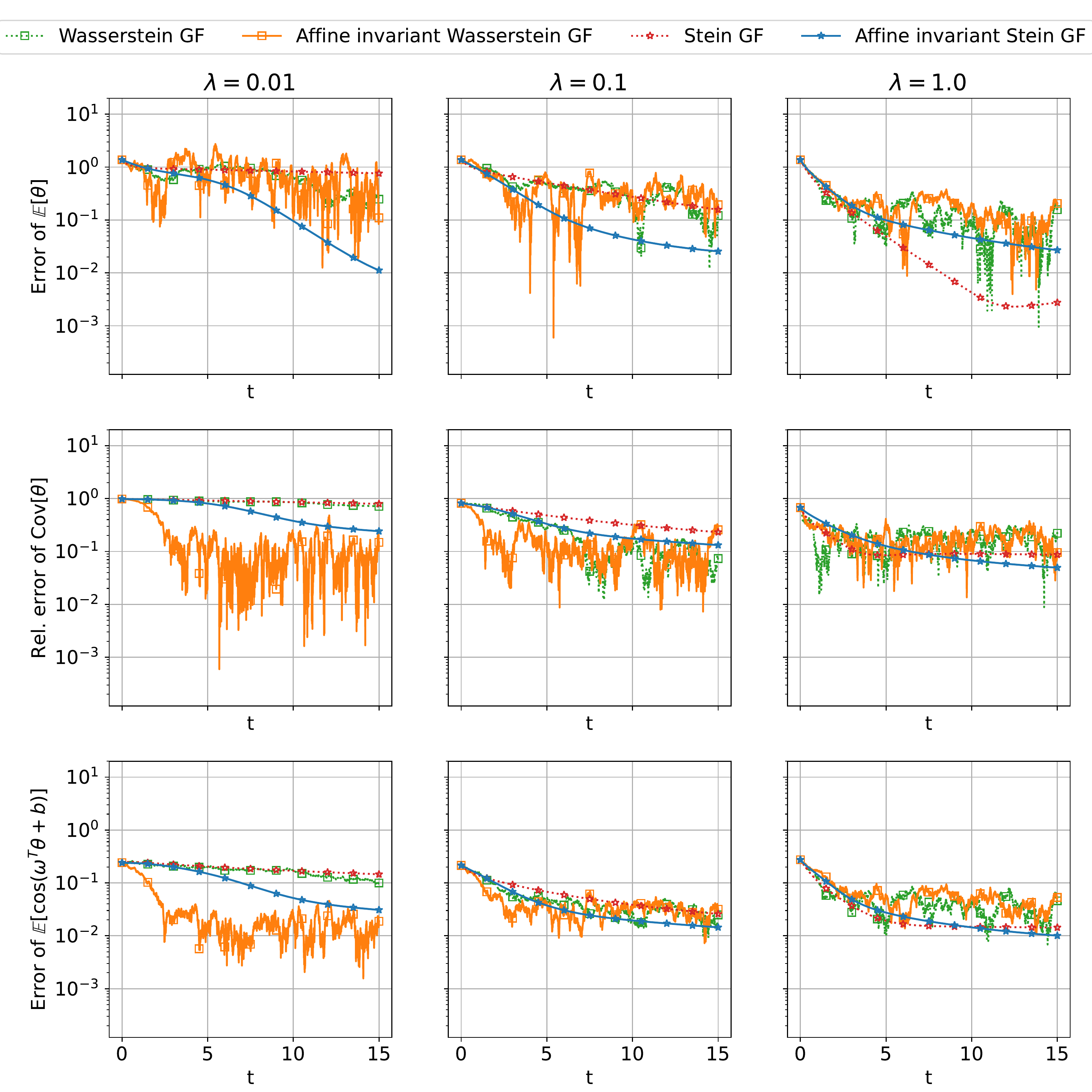}
    \caption{Gaussian posterior case: convergence of different gradient flows in terms of the $L_2$ error of $\E[\theta_t]$, the relative Frobenius norm error of the covariance $ \frac{\lVert \Cov[\theta_t] - \Cov[\theta_{{\rm true}}]\rVert_F}{\lVert \Cov[\theta_{\rm true}]\rVert_F}$, and the error of $\E[\cos(\omega^T \theta_t + b)]$. }
    \label{fig:Gaussian_gd_particle_converge}
\end{figure}

\begin{figure}[ht]
\centering
    \includegraphics[width=0.8\textwidth]{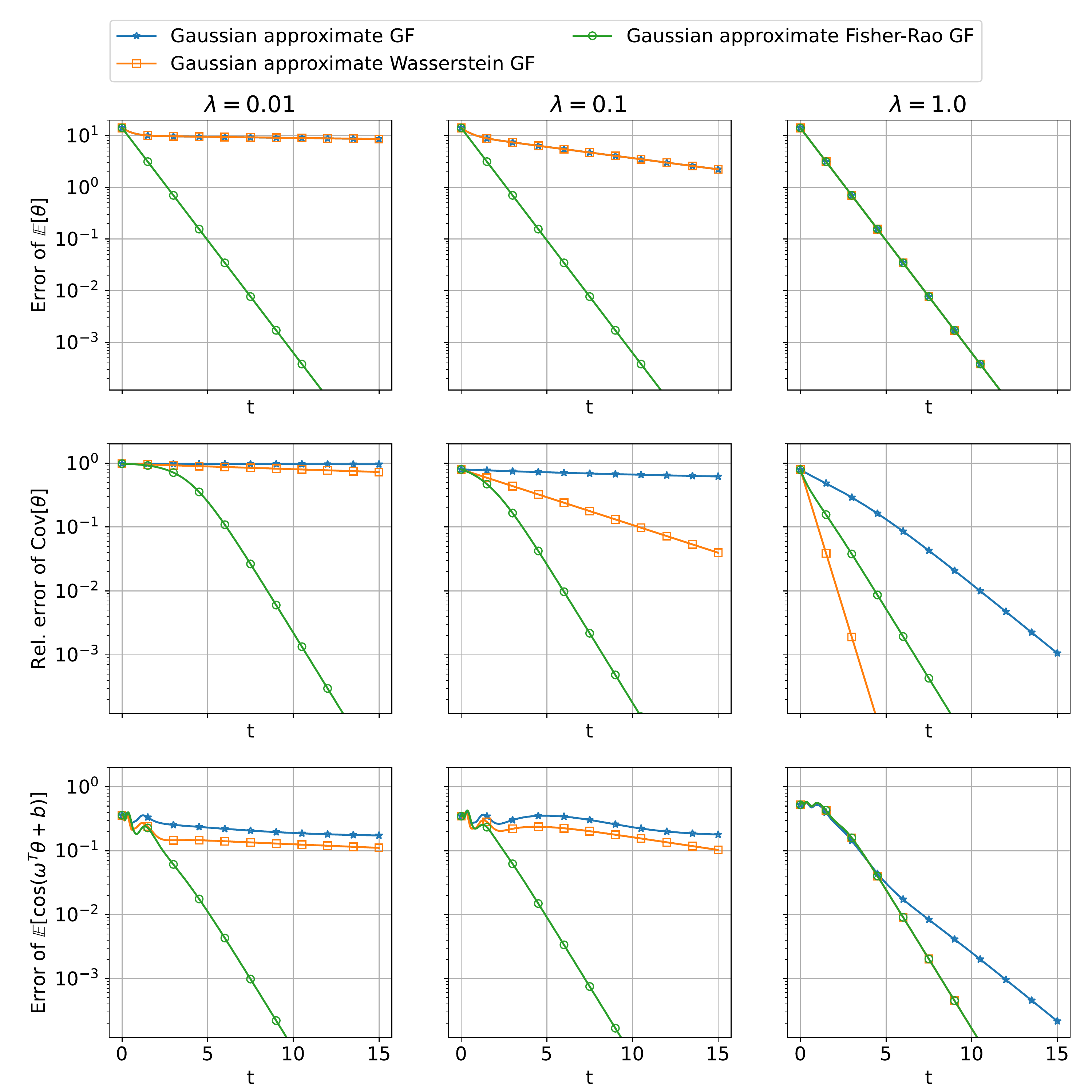}
    \caption{Gaussian posterior case: convergence of different dynamics in terms of $L_2$ error of $\E[\theta_t]$, the relative Frobenius norm error of the covariance $ \frac{\lVert \Cov[\theta_t] - \Cov[\theta_{{\rm true}}]\rVert_F}{\lVert \Cov[\theta_{\rm true}]\rVert_F}$, and the error of $\E[\cos(\omega^T \theta_t + b)]$.}
    \label{fig:Gaussian_gd_converge}
\end{figure}

\subsection{Logconcave Posterior Case}
\label{ssec:Logconcave-numerics}

The convergence of different gradient flows, according to the three summary statistics, are presented in~\cref{fig:Logconcave_gd_particle_converge} (nonparameteric density space) and in~\cref{fig:Logconcave_gd_converge} (Gaussian density space). In both nonparametric and Gaussian density spaces, the imposition of the affine invariance property makes the convergence rate independent of the anisotropy $\lambda$ and accelerates the sampling in the highly anisotropic case ($\lambda = 0.01$). We observe that the corresponding Gaussian approximate gradient flows can reach lower errors for this case with the present numerical method setup defined in \Cref{ssec:method-setup}. 
We also observe that the convergence rate of the Gaussian approximate Fisher-Rao gradient flow does not deteriorate with increased anisotropy constant $\lambda$; this 
indicates that the convergence rate in~\cref{p:logconcave}, for this gradient
flow, may not be tight.


\begin{figure}[ht]
\centering
    \includegraphics[width=0.8\textwidth]{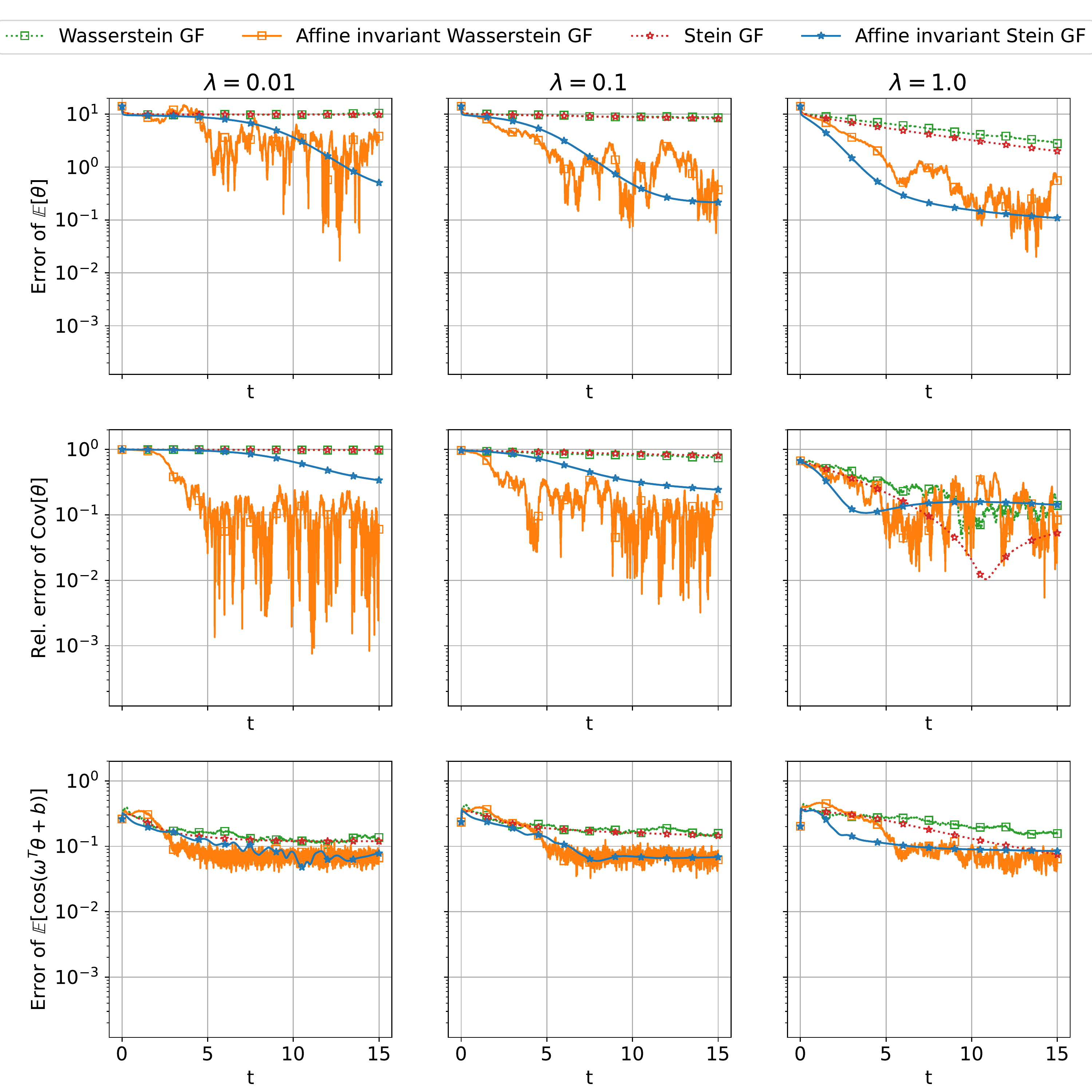}
    \caption{Logconcave posterior case: convergence of different gradient flows in terms of the $L_2$ error of $\E[\theta_t]$, the relative Frobenius norm error of the covariance $ \frac{\lVert \Cov[\theta_t] - \Cov[\theta_{{\rm true}}]\rVert_F}{\lVert \Cov[\theta_{\rm true}]\rVert_F}$, and the error of $\E[\cos(\omega^T \theta_t + b)]$.}
    \label{fig:Logconcave_gd_particle_converge}
\end{figure}


\begin{figure}[ht]
\centering
    \includegraphics[width=0.8\textwidth]{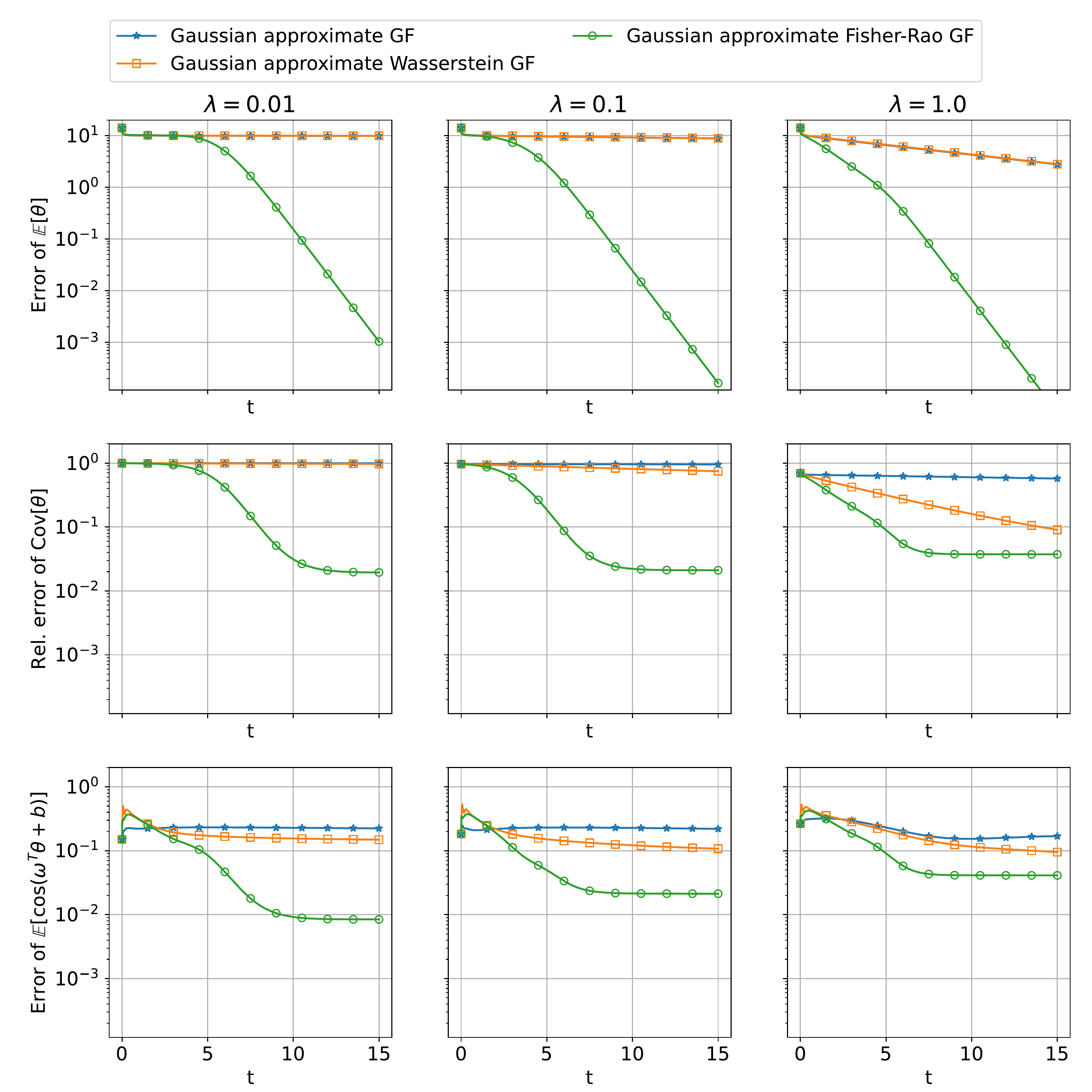}
    \caption{ Logconcave posterior case: convergence of different dynamics in terms of $L_2$ error of $\E[\theta_t]$, the relative Frobenius norm error of the covariance $ \frac{\lVert \Cov[\theta_t] - \Cov[\theta_{{\rm true}}]\rVert_F}{\lVert \Cov[\theta_{\rm true}]\rVert_F}$, and the error of $\E[\cos(\omega^T \theta_t + b)]$}
    \label{fig:Logconcave_gd_converge}
\end{figure}

\subsection{General Posterior Case}
\label{ssec:General-numerics}
We note that the Rosenbrock function is a non-convex function. Although its minimizer is at $[1, 1]$, the expectation and covariance of the posterior density function is (See \cref{sec:integration}) 
$$
\E[\theta] = \begin{bmatrix}
  1\\
  11
\end{bmatrix} \qquad 
\Cov[\theta] = \begin{bmatrix}
  10& 20\\
  20& \frac{10}{\lambda} + 240
\end{bmatrix}.
$$

The particles obtained by different nonparameteric gradient flows at $t=15$ are depicted in~\cref{fig:Rosenbrock_gd_particle_density}, and their convergences according to the three summary statistics are depicted in~\cref{fig:Rosenbrock_gd_particle_converge}. 
Estimated posterior densities (3 standard deviations) obtained by different Gaussian approximate gradient flows are presented in~\cref{fig:Rosenbrock_gd_density}, and their convergences according to the three summary statistics are depicted in~\cref{fig:Rosenbrock_gd_converge}.
For small $\lambda$ (e.g., $\lambda = 0.01$), ${\theta^{(2)}}$ is the stretch direction, and therefore the imposition of the affine invariance property makes the convergence faster.  
However, when $\lambda$ increases, the posterior density concentrates on
a manifold with significant curvature (See \cref{fig:Rosenbrock_gd_particle_density}). 
Although the particle positions match well with the density contours, the convergence of different gradient flows significantly deteriorates; the imposition of 
affine invariance does not relieve the situation.
Furthermore, the Gaussian approximation cannot represent the posterior 
distribution at all well because the posterior is far from Gaussian.

\begin{figure}[ht]
\centering
    \includegraphics[width=0.8\textwidth]{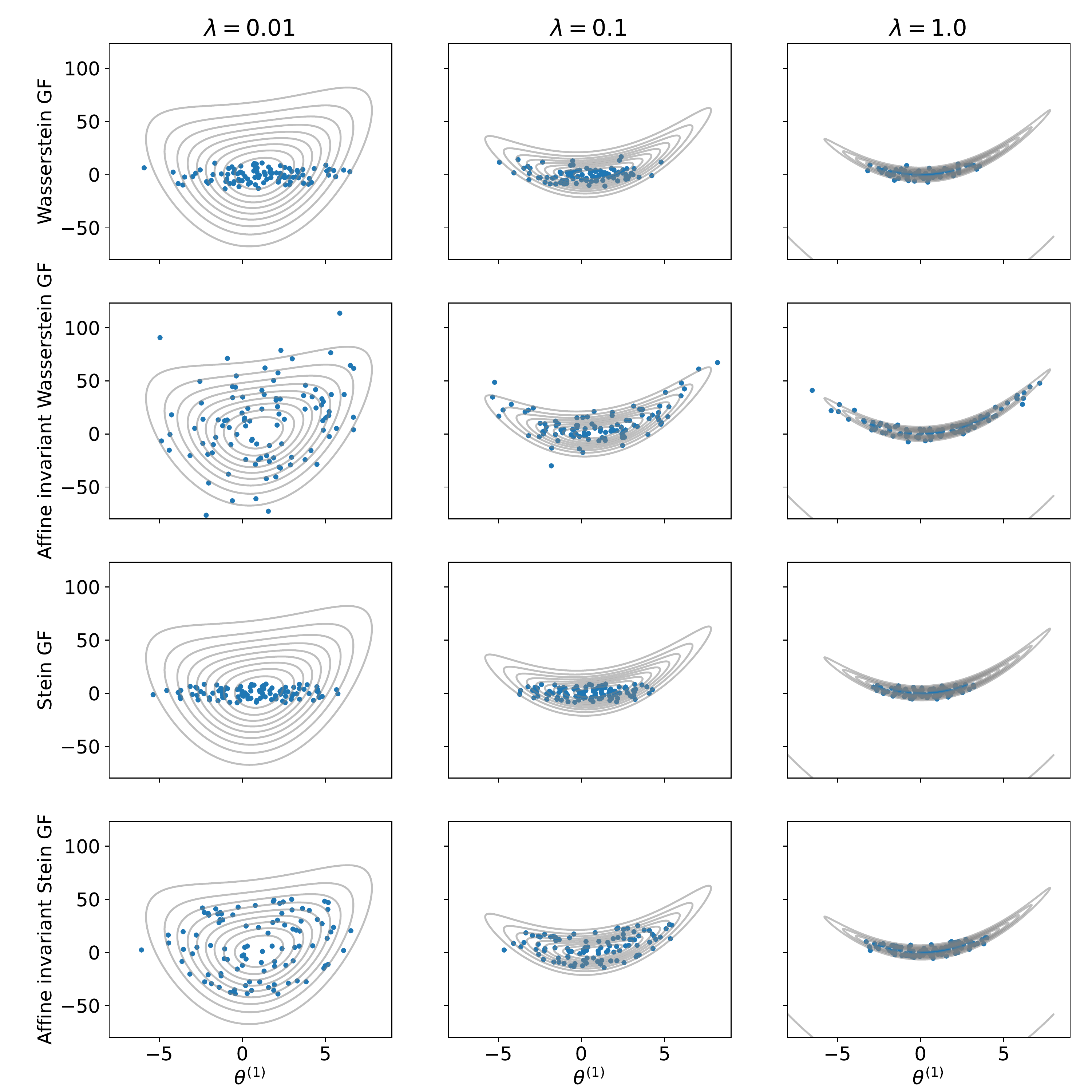}
    \caption{General posterior case: particles obtained by different gradient flows at $t=15$. Grey lines represent the contour of the true posterior.}
    \label{fig:Rosenbrock_gd_particle_density}
\end{figure}

\begin{figure}[ht]
\centering
    \includegraphics[width=0.8\textwidth]{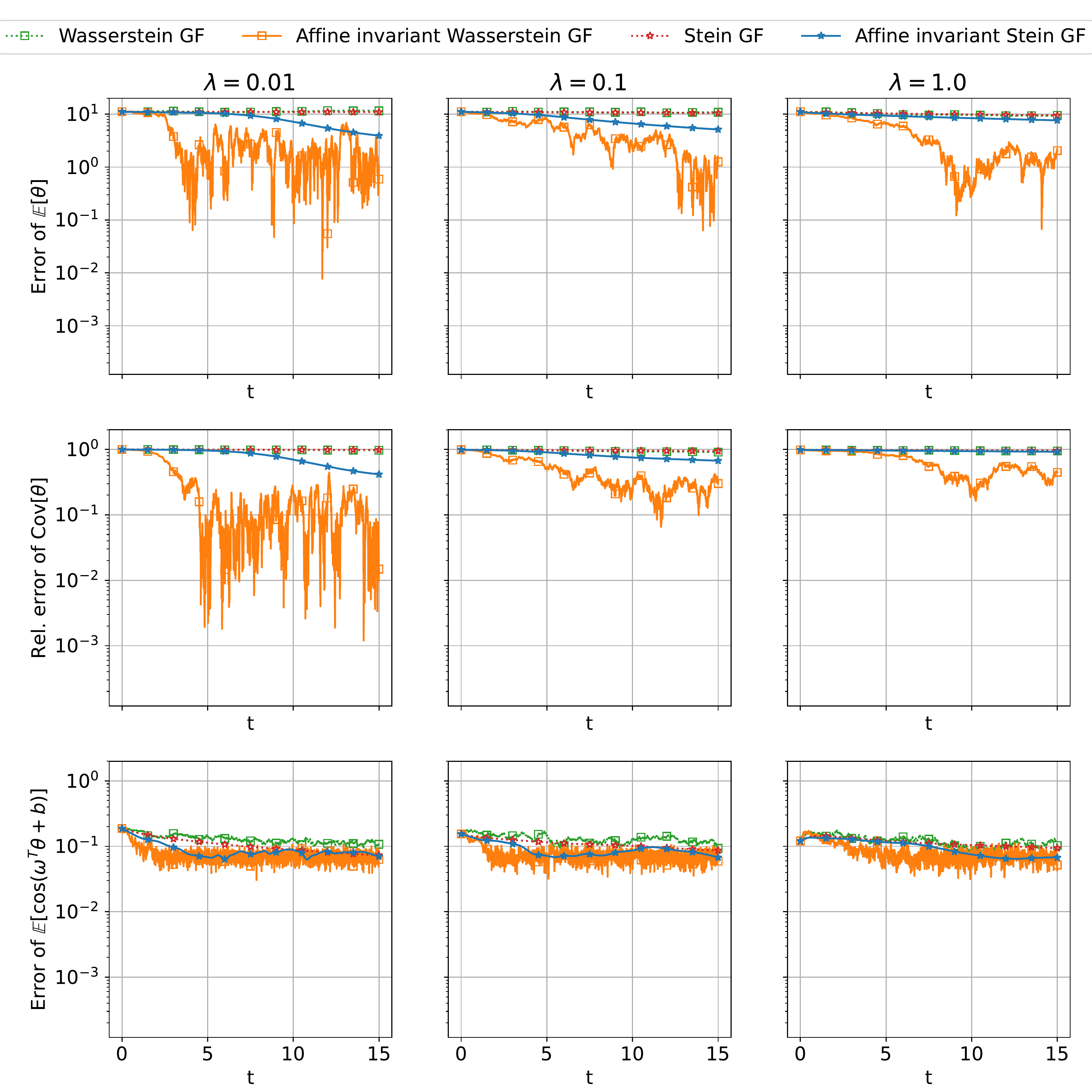}
    \caption{General posterior case: convergence of different gradient flows in terms of the $L_2$ error of $\E[\theta_t]$, the relative Frobenius norm error of the covariance $ \frac{\lVert \Cov[\theta_t] - \Cov[\theta_{{\rm true}}]\rVert_F}{\lVert \Cov[\theta_{\rm true}]\rVert_F}$, and the error of $\E[\cos(\omega^T \theta_t + b)]$.}
    \label{fig:Rosenbrock_gd_particle_converge}
\end{figure}

\begin{figure}[ht]
\centering
    \includegraphics[width=0.8\textwidth]{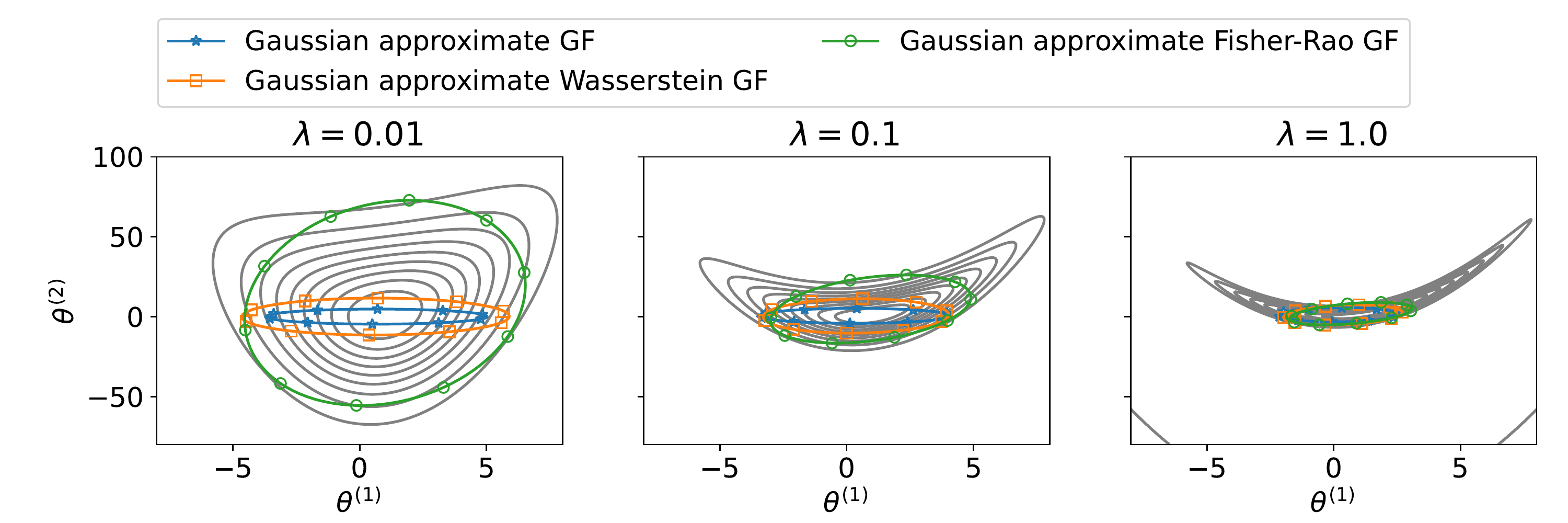}
    \caption{General posterior case: density functions (3 standard deviations) obtained by different dynamics at $t=15$. Grey lines represent the contour of the true posterior.
    }
    \label{fig:Rosenbrock_gd_density}
\end{figure}

\begin{figure}[ht]
\centering
    \includegraphics[width=0.8\textwidth]{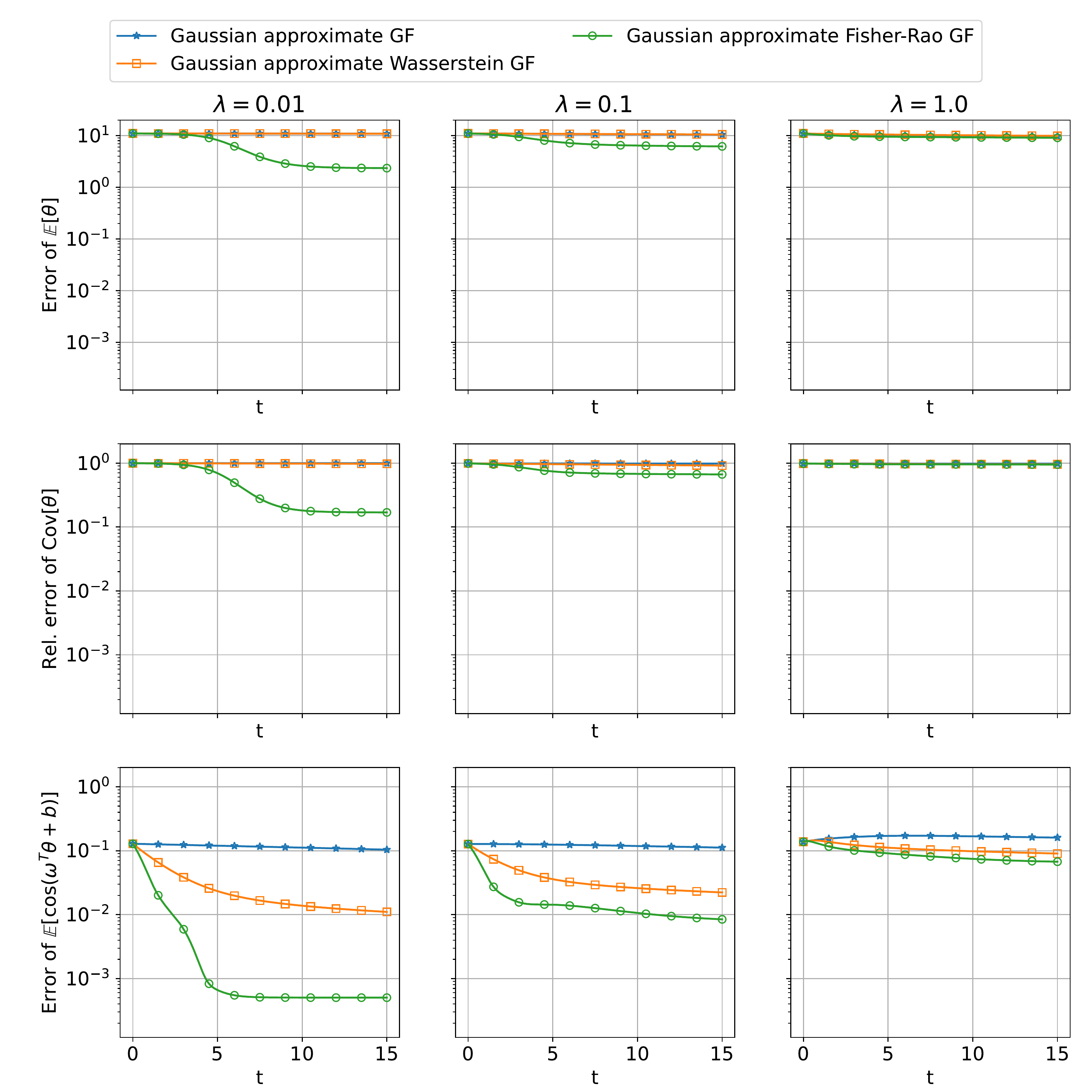}
    \caption{General posterior case: convergence of different dynamics in terms of $L_2$ error of $\E[\theta_t]$, the relative Frobenius norm error of the covariance $ \frac{\lVert \Cov[\theta_t] - \Cov[\theta_{{\rm true}}]\rVert_F}{\lVert \Cov[\theta_{\rm true}]\rVert_F}$, and the error of $\E[\cos(\omega^T \theta_t + b)]$.}
    \label{fig:Rosenbrock_gd_converge}
\end{figure}

\section{Conclusions}
\label{sec:conclusion}

\subsection{Summary}
In this work, we have studied various gradient flows in both nonparametric and Gaussian density spaces for sampling distributions with unknown normalization constants, focusing on the affine invariance property. We introduce the concept of affine 
invariant metric and use it to develop general affine invariant Wasserstein, Stein, and Fisher-Rao gradient flows in both nonparametric and Gaussian density spaces. 
We provide a theoretical analysis of these gradient flows and demonstrate that affine-invariance effectively improves the efficiency of sampling log-concave distributions.
Numerically, we demonstrate that the affine-invariance property can accelerate the convergence of gradient flows for highly anisotropic distributions.

Nevertheless, these strategies may still not perform well for general posterior 
distributions. In particular for multimodal distributions, or distributions 
which concentrate on manifolds with significant curvature such as the
Rosenbrock function used here. In a companion work~\cite{chen2022derivative}, 
we are exploring the direction of approximating the Fisher-Rao Gradient flow 
using Gaussian mixtures; this has the potential to capture multiple dominant modes efficiently. We are also interested in exploring other invariant properties and approximations that could deal with these more complex distributions. In addition, for high dimensional problems, it is of interest to develop a systematic study of model reduction of these gradient flows; see for example the projected Stein and Wasserstein gradient flows studied in \cite{chen2020projected, wang2022projected}.

Finally, we highlight that, for Bayesian inverse problems \cite{sanz2018inverse}
the methods developed here do not exploit the structure of the forward problem.
This is in contrast to ensemble Kalman based methods which have demonstrable
performance advantages for some problems in this class \cite{calvello22}; notably these ensemble methods have derivative free implementations which are favorable in large scale inverse problems. 
Developing the analysis of these methods, following the approach in this paper, 
constitutes an interesting direction for future study.

\subsection{Open Technical Problems}
Several open technical problems remain unsolved 
in the work we have presented; we list them here.
\begin{enumerate}
    \item 
    A key determining factor in the exponential rates \cref{proposition:W-convergence} of the 
affine invariant Wasserstein gradient flow with preconditioner $\Prec_t$ is the lower-bound on the preconditioner. 
Determining such a lower bound, or determiner whether the preconditioner can become singular under the flow,
is an interesting question for further study, in particular under a logconcavity assumption such as
\cref{assumption:logconcave}.
    
    \item For the affine invariant Stein gradient flow, the study of convergence 
and convergence rates for general kernel functions is of interest.
    
    \item For the Gaussian approximate Fisher-Rao gradient 
flow~\cref{eq:Gaussian Fisher-Rao}, 
in particular under the logconcavity assumption 
\cref{assumption:logconcave},
we prove a global convergence rate of at least \[e^{-\alpha\min\{\frac{1}{\beta}, \lambda_{0,\min}\}t}\] in \cref{p:logconcave}.  However, the numerical study showed
far superior behavior on one problem class. A sharp bound is of interest. We study this problem from the perspective of local convergence rate in \cref{p:logconcave-local} for $N_\theta = 1$. Studying this problem for $N_\theta > 1$ is of future interest.

    \item For the Gaussian approximate Fisher-Rao gradient flow~\cref{eq:Gaussian Fisher-Rao}, 
under the logconcave assumption such as
\cref{assumption:logconcave}, we prove that $C_t$ is bounded below and
above in \Cref{p:logconcave}. 
Furthermore, studying whether $C_t$ can become singular or infinity for the
general posterior case is of future interest.

\end{enumerate}

\vspace{0.5in}
\noindent {\bf Acknowledgments} YC acknowledges the support from the Air Force Office of Scientific Research under MURI award number FA9550-20-1-0358 (Machine Learning and Physics-Based Modeling and Simulation). DZH and AMS are supported by NSF award AGS1835860 and by the generosity of Eric and Wendy Schmidt by recommendation of the Schmidt Futures program; 
DZH is also supported by High-performance Computing Platform of Peking University;
AMS is also
supported by the Office of Naval Research (ONR) through grant N00014-17-1-2079
and by a Department of Defense Vannevar Bush Faculty Fellowship. SR is supported by Deutsche Forschungsgemeinschaft (DFG) - Project-ID 318763901 - SFB1294.
JH is supported by NSF grant DMS-2054835.

\bibliographystyle{siamplain}
\bibliography{references}

\begin{thebibliography}{100}

\bibitem{amari1998natural}
{\sc S.-I. Amari}, {\em Natural gradient works efficiently in learning}, Neural
  computation, 10 (1998), pp.~251--276.

\bibitem{amari2016information}
{\sc S.-i. Amari}, {\em Information Geometry and its Applications}, vol.~194,
  Springer, 2016.

\bibitem{ambrosio2005gradient}
{\sc L.~Ambrosio, N.~Gigli, and G.~Savar{\'e}}, {\em Gradient flows: in Metric
  Spaces and in the Space of Probability Measures}, Springer Science \&
  Business Media, 2005.

\bibitem{ambrosio2006gradient}
{\sc L.~Ambrosio, G.~Savar{\'e}, et~al.}, {\em Gradient flows of probability
  measures}, Handbook of differential equations: evolutionary equations, 3
  (2006), pp.~1--136.

\bibitem{applebaum2009levy}
{\sc D.~Applebaum}, {\em L{\'e}vy processes and stochastic calculus}, Cambridge
  university press, 2009.

\bibitem{ay2015information}
{\sc N.~Ay, J.~Jost, H.~V. L{\^e}, and L.~Schwachh{\"o}fer}, {\em Information
  geometry and sufficient statistics}, Probability Theory and Related Fields,
  162 (2015), pp.~327--364.

\bibitem{ay2017information}
{\sc N.~Ay, J.~Jost, H.~V.~L{\^e}, and L.~Schwachh{\"o}fer}, {\em Information
  Geometry}, vol.~64, Springer, 2017.

\bibitem{bakry1985diffusions}
{\sc D.~Bakry and M.~{\'E}mery}, {\em Diffusions hypercontractives}, in
  Seminaire de probabilit{\'e}s XIX 1983/84, Springer, 1985, pp.~177--206.

\bibitem{bakry2014analysis}
{\sc D.~Bakry, I.~Gentil, and M.~Ledoux}, {\em Analysis and geometry of Markov
  diffusion operators}, vol.~103, Springer, 2014.

\bibitem{bauer2016uniqueness}
{\sc M.~Bauer, M.~Bruveris, and P.~W. Michor}, {\em Uniqueness of the
  {Fisher--Rao} metric on the space of smooth densities}, Bulletin of the
  London Mathematical Society, 48 (2016), pp.~499--506.

\bibitem{bell1994iterated}
{\sc B.~M. Bell}, {\em The iterated {Kalman smoother as a Gauss--Newton}
  method}, SIAM Journal on Optimization, 4 (1994), pp.~626--636.

\bibitem{bell1993iterated}
{\sc B.~M. Bell and F.~W. Cathey}, {\em The iterated {Kalman filter update as a
  Gauss-Newton method}}, IEEE Transactions on Automatic Control, 38 (1993),
  pp.~294--297.

\bibitem{benamou2000computational}
{\sc J.-D. Benamou and Y.~Brenier}, {\em A computational fluid mechanics
  solution to the monge-kantorovich mass transfer problem}, Numerische
  Mathematik, 84 (2000), pp.~375--393.

\bibitem{bertoin1996levy}
{\sc J.~Bertoin}, {\em L{\'e}vy processes}, vol.~121, Cambridge university
  press Cambridge, 1996.

\bibitem{bhatia2019bures}
{\sc R.~Bhatia, T.~Jain, and Y.~Lim}, {\em On the bures--wasserstein distance
  between positive definite matrices}, Expositiones Mathematicae, 37 (2019),
  pp.~165--191.

\bibitem{blei2017variational}
{\sc D.~M. Blei, A.~Kucukelbir, and J.~D. McAuliffe}, {\em Variational
  inference: A review for statisticians}, Journal of the American statistical
  Association, 112 (2017), pp.~859--877.

\bibitem{blomker2019well}
{\sc D.~Bl{\"o}mker, C.~Schillings, P.~Wacker, and S.~Weissmann}, {\em Well
  posedness and convergence analysis of the ensemble {Kalman} inversion},
  Inverse Problems, 35 (2019), p.~085007.

\bibitem{boffi2022probability}
{\sc N.~M. Boffi and E.~Vanden-Eijnden}, {\em Probability flow solution of the
  fokker–planck equation}, Machine Learning: Science and Technology, 4
  (2023), p.~035012, \url{https://doi.org/10.1088/2632-2153/ace2aa},
  \url{https://dx.doi.org/10.1088/2632-2153/ace2aa}.

\bibitem{brooks2011handbook}
{\sc S.~Brooks, A.~Gelman, G.~Jones, and X.-L. Meng}, {\em Handbook of Markov
  chain Monte Carlo}, CRC press, 2011.

\bibitem{burger2023covariance}
{\sc M.~Burger, M.~Erbar, F.~Hoffmann, D.~Matthes, and A.~Schlichting}, {\em
  Covariance-modulated optimal transport and gradient flows}, arXiv preprint
  arXiv:2302.07773,  (2023).

\bibitem{calvello22}
{\sc E.~Calvello, S.~Reich, and A.~M. Stuart}, {\em Ensemble {Kalman} methods:
  a mean field perspective}, arXiv preprint arXiv:2209.11371,  (2022).

\bibitem{cao2022bayesian}
{\sc S.~Cao and D.~Z. Huang}, {\em Bayesian calibration for large-scale fluid
  structure interaction problems under embedded/immersed boundary framework},
  International Journal for Numerical Methods in Engineering,  (2022).

\bibitem{carrillo2021wasserstein}
{\sc J.~Carrillo and U.~Vaes}, {\em Wasserstein stability estimates for
  covariance-preconditioned fokker--planck equations}, Nonlinearity, 34 (2021),
  p.~2275.

\bibitem{carrillo2018analytical}
{\sc J.~A. Carrillo, Y.-P. Choi, C.~Totzeck, and O.~Tse}, {\em An analytical
  framework for consensus-based global optimization method}, Mathematical
  Models and Methods in Applied Sciences, 28 (2018), pp.~1037--1066.

\bibitem{cencov2000statistical}
{\sc N.~N. Cencov}, {\em Statistical decision rules and optimal inference},
  American Mathematical Soc., 2000.

\bibitem{chada2022convergence}
{\sc N.~Chada and X.~Tong}, {\em Convergence acceleration of ensemble {Kalman}
  inversion in nonlinear settings}, Mathematics of Computation, 91 (2022),
  pp.~1247--1280.

\bibitem{chada2020iterative}
{\sc N.~K. Chada, Y.~Chen, and D.~Sanz-Alonso}, {\em Iterative ensemble
  {Kalman} methods: a unified perspective with some new variants}, arXiv
  preprint arXiv:2010.13299,  (2020).

\bibitem{chaintron2021propagation}
{\sc L.-P. Chaintron and A.~Diez}, {\em Propagation of chaos: a review of
  models, methods and applications}, arXiv preprint arXiv:2106.14812,  (2021).

\bibitem{chen2020projected}
{\sc P.~Chen and O.~Ghattas}, {\em Projected {Stein} variational gradient
  descent}, Advances in Neural Information Processing Systems, 33 (2020),
  pp.~1947--1958.

\bibitem{chen2021solving}
{\sc Y.~Chen, B.~Hosseini, H.~Owhadi, and A.~M. Stuart}, {\em Solving and
  learning nonlinear {PDEs with Gaussian processes}}, Journal of Computational
  Physics, 447 (2021), p.~110668.

\bibitem{chen2022derivative}
{\sc Y.~Chen, D.~Z. Huang, J.~Huang, S.~Reich, and A.~M. Stuart}, {\em
  Derivative-free multimodal sampling: {Fisher-Rao Gradient Flow and Gaussian}
  mixture approximation}, in preparation,  (2022).

\bibitem{chen2018natural}
{\sc Y.~Chen and W.~Li}, {\em Optimal transport natural gradient for
  statistical manifolds with continuous sample space}, Information Geometry, 3
  (2020), pp.~1--32.

\bibitem{chen2012ensemble}
{\sc Y.~Chen and D.~S. Oliver}, {\em Ensemble randomized maximum likelihood
  method as an iterative ensemble smoother}, Mathematical Geosciences, 44
  (2012), pp.~1--26.

\bibitem{chen2021consistency}
{\sc Y.~Chen, H.~Owhadi, and A.~Stuart}, {\em Consistency of empirical {Bayes}
  and kernel flow for hierarchical parameter estimation}, Mathematics of
  Computation, 90 (2021), pp.~2527--2578.

\bibitem{chewi2020svgd}
{\sc S.~Chewi, T.~Le~Gouic, C.~Lu, T.~Maunu, and P.~Rigollet}, {\em {SVGD} as a
  kernelized {Wasserstein} gradient flow of the chi-squared divergence},
  Advances in Neural Information Processing Systems, 33 (2020), pp.~2098--2109.

\bibitem{chewi2020exponential}
{\sc S.~Chewi, T.~Le~Gouic, C.~Lu, T.~Maunu, P.~Rigollet, and A.~Stromme}, {\em
  Exponential ergodicity of mirror-langevin diffusions}, Advances in Neural
  Information Processing Systems, 33 (2020), pp.~19573--19585.

\bibitem{chizat2018global}
{\sc L.~Chizat and F.~Bach}, {\em On the global convergence of gradient descent
  for over-parameterized models using optimal transport}, Advances in neural
  information processing systems, 31 (2018).

\bibitem{chopin2020introduction}
{\sc N.~Chopin, O.~Papaspiliopoulos, et~al.}, {\em An introduction to
  sequential Monte Carlo}, Springer, 2020.

\bibitem{christen2010general}
{\sc J.~A. Christen and C.~Fox}, {\em A general purpose sampling algorithm for
  continuous distributions (the t-walk)},  (2010).

\bibitem{cotter2013mcmc}
{\sc S.~L. Cotter, G.~O. Roberts, A.~M. Stuart, and D.~White}, {\em {MCMC}
  methods for functions: modifying old algorithms to make them faster},
  Statistical Science, 28 (2013), pp.~424--446.

\bibitem{coullon2021ensemble}
{\sc J.~Coullon and R.~J. Webber}, {\em Ensemble sampler for
  infinite-dimensional inverse problems}, Statistics and Computing, 31 (2021),
  pp.~1--9.

\bibitem{cui2016dimension}
{\sc T.~Cui, K.~J. Law, and Y.~M. Marzouk}, {\em Dimension-independent
  likelihood-informed mcmc}, Journal of Computational Physics, 304 (2016),
  pp.~109--137.

\bibitem{daumetal2010}
{\sc F.~Daum, J.~Huang, and A.~Noushin}, {\em {Exact particle flow for
  nonlinear filters}}, in Signal Processing, Sensor Fusion, and Target
  Recognition XIX, I.~Kadar, ed., vol.~7697, International Society for Optics
  and Photonics, SPIE, 2010, pp.~92 -- 110,
  \url{https://doi.org/10.1117/12.839590}.

\bibitem{del2006sequential}
{\sc P.~Del~Moral, A.~Doucet, and A.~Jasra}, {\em Sequential {M}onte {C}arlo
  samplers}, Journal of the Royal Statistical Society: Series B (Statistical
  Methodology), 68 (2006), pp.~411--436.

\bibitem{detommaso2018stein}
{\sc G.~Detommaso, T.~Cui, Y.~Marzouk, A.~Spantini, and R.~Scheichl}, {\em A
  {Stein variational Newton} method}, Advances in Neural Information Processing
  Systems, 31 (2018).

\bibitem{ding2021ensemble}
{\sc Z.~Ding and Q.~Li}, {\em Ensemble {Kalman} inversion: mean-field limit and
  convergence analysis}, Statistics and Computing, 31 (2021), pp.~1--21.

\bibitem{do1992riemannian}
{\sc M.~P. Do~Carmo and J.~Flaherty~Francis}, {\em Riemannian Geometry},
  vol.~6, Springer, 1992.

\bibitem{doucet2009tutorial}
{\sc A.~Doucet, A.~M. Johansen, et~al.}, {\em A tutorial on particle filtering
  and smoothing: Fifteen years later}, Handbook of Nonlinear Filtering, 12
  (2009), p.~3.

\bibitem{duncan2019geometry}
{\sc A.~Duncan, N.~N{\"u}sken, and L.~Szpruch}, {\em On the geometry of {Stein}
  variational gradient descent}, Journal of Machine Learning Research, 24
  (2023), pp.~1--39.

\bibitem{dunlop2022gradient}
{\sc M.~M. Dunlop and G.~Stadler}, {\em A gradient-free subspace-adjusting
  ensemble sampler for infinite-dimensional bayesian inverse problems}, arXiv
  preprint arXiv:2202.11088,  (2022).

\bibitem{emerick2013investigation}
{\sc A.~A. Emerick and A.~C. Reynolds}, {\em Investigation of the sampling
  performance of ensemble-based methods with a simple reservoir model},
  Computational Geosciences, 17 (2013), pp.~325--350.

\bibitem{evensen1994sequential}
{\sc G.~Evensen}, {\em Sequential data assimilation with a nonlinear
  quasi-geostrophic model using {M}onte {C}arlo methods to forecast error
  statistics}, Journal of Geophysical Research: Oceans, 99 (1994),
  pp.~10143--10162.

\bibitem{foreman2013emcee}
{\sc D.~Foreman-Mackey, D.~W. Hogg, D.~Lang, and J.~Goodman}, {\em {EMCEE}:
  {T}he {MCMC} hammer}, Publications of the Astronomical Society of the
  Pacific, 125 (2013), p.~306.

\bibitem{friedman1962functional}
{\sc D.~Friedman}, {\em The functional equation f (x+ y)= g (x)+ h (y)}, The
  American Mathematical Monthly, 69 (1962), pp.~769--772.

\bibitem{FRmetricInfDim1991}
{\sc T.~Friedrich}, {\em Die fisher-information und symplektische strukturen},
  Mathematische Nachrichten, 153 (1991), pp.~273--296,
  \url{https://doi.org/https://doi.org/10.1002/mana.19911530125},
  \url{https://onlinelibrary.wiley.com/doi/abs/10.1002/mana.19911530125}.

\bibitem{galy2021particle}
{\sc T.~Galy-Fajou, V.~Perrone, and M.~Opper}, {\em Flexible and efficient
  inference with particles for the variational {G}aussian approximation},
  Entropy, 23 (2021), p.~990, \url{https://doi.org/10.3390/e23080990}.

\bibitem{galy2021flexible}
{\sc T.~Galy-Fajou, V.~Perrone, and M.~Opper}, {\em Flexible and efficient
  inference with particles for the variational gaussian approximation},
  Entropy, 23 (2021), p.~990.

\bibitem{garbuno2020interacting}
{\sc A.~Garbuno-Inigo, F.~Hoffmann, W.~Li, and A.~M. Stuart}, {\em Interacting
  {L}angevin diffusions: {G}radient structure and ensemble {K}alman sampler},
  SIAM Journal on Applied Dynamical Systems, 19 (2020), pp.~412--441.

\bibitem{garbuno2020affine}
{\sc A.~Garbuno-Inigo, N.~N{\"u}sken, and S.~Reich}, {\em Affine invariant
  interacting {L}angevin dynamics for {B}ayesian inference}, SIAM Journal on
  Applied Dynamical Systems, 19 (2020), pp.~1633--1658.

\bibitem{TrillosNoticeAMS}
{\sc N.~Garcia~Trillos, B.~Hosseini, and D.~Sanz-Alonso}, {\em From
  optimization to sampling through gradient flows}, Notices of the American
  Mathematical Society, 70 (2023).

\bibitem{garcia2020bayesian}
{\sc N.~Garcia~Trillos and D.~Sanz-Alonso}, {\em The {B}ayesian update:
  Variational formulations and gradient flows}, Bayesian Analysis, 15 (2020),
  pp.~29--56.

\bibitem{gibbs2002choosing}
{\sc A.~L. Gibbs and F.~E. Su}, {\em On choosing and bounding probability
  metrics}, International Statistical Review, 70 (2002), pp.~419--435.

\bibitem{gonzalez2008first}
{\sc O.~Gonzalez and A.~M. Stuart}, {\em A First Course in Continuum
  Mechanics}, vol.~42, Cambridge University Press, 2008.

\bibitem{goodman2010ensemble}
{\sc J.~Goodman and J.~Weare}, {\em Ensemble samplers with affine invariance},
  Communications in applied mathematics and computational science, 5 (2010),
  pp.~65--80.

\bibitem{gross1975logarithmic}
{\sc L.~Gross}, {\em Logarithmic sobolev inequalities}, American Journal of
  Mathematics, 97 (1975), pp.~1061--1083.

\bibitem{halder2017gradient}
{\sc A.~Halder and T.~T. Georgiou}, {\em Gradient flows in uncertainty
  propagation and filtering of linear gaussian systems}, in 2017 IEEE 56th
  Annual Conference on Decision and Control (CDC), IEEE, 2017, pp.~3081--3088.

\bibitem{halder2018gradient}
{\sc A.~Halder and T.~T. Georgiou}, {\em Gradient flows in filtering and
  {Fisher-Rao} geometry}, in 2018 Annual American Control Conference (ACC),
  IEEE, 2018, pp.~4281--4286.

\bibitem{he2022regularized}
{\sc Y.~He, K.~Balasubramanian, B.~K. Sriperumbudur, and J.~Lu}, {\em
  Regularized stein variational gradient flow}, arXiv preprint
  arXiv:2211.07861,  (2022).

\bibitem{hsieh2018mirrored}
{\sc Y.-P. Hsieh, A.~Kavis, P.~Rolland, and V.~Cevher}, {\em Mirrored langevin
  dynamics}, Advances in Neural Information Processing Systems, 31 (2018).

\bibitem{huang2022efficient}
{\sc D.~Z. Huang, J.~Huang, S.~Reich, and A.~M. Stuart}, {\em Efficient
  derivative-free {B}ayesian inference for large-scale inverse problems}, arXiv
  preprint arXiv:2204.04386,  (2022).

\bibitem{huang2022iterated}
{\sc D.~Z. Huang, T.~Schneider, and A.~M. Stuart}, {\em Iterated {Kalman}
  methodology for inverse problems}, Journal of Computational Physics, 463
  (2022), p.~111262.

\bibitem{iglesias2013ensemble}
{\sc M.~A. Iglesias, K.~J. Law, and A.~M. Stuart}, {\em Ensemble {K}alman
  methods for inverse problems}, Inverse Problems, 29 (2013), p.~045001.

\bibitem{isaac2015scalable}
{\sc T.~Isaac, N.~Petra, G.~Stadler, and O.~Ghattas}, {\em Scalable and
  efficient algorithms for the propagation of uncertainty from data through
  inference to prediction for large-scale problems, with application to flow of
  the antarctic ice sheet}, Journal of Computational Physics, 296 (2015),
  pp.~348--368.

\bibitem{ito1996diffusion}
{\sc K.~It{\^o}, P.~Henry~Jr, et~al.}, {\em Diffusion processes and their
  sample paths: Reprint of the 1974 edition}, Springer Science \& Business
  Media, 1996.

\bibitem{jazwinski2007stochastic}
{\sc A.~H. Jazwinski}, {\em Stochastic processes and filtering theory}, Courier
  Corporation, 2007.

\bibitem{jordan1998variational}
{\sc R.~Jordan, D.~Kinderlehrer, and F.~Otto}, {\em The variational formulation
  of the {Fokker--Planck} equation}, SIAM journal on mathematical analysis, 29
  (1998), pp.~1--17.

\bibitem{julier1997new}
{\sc S.~J. Julier and J.~K. Uhlmann}, {\em New extension of the {K}alman filter
  to nonlinear systems}, in Signal processing, sensor fusion, and target
  recognition VI, vol.~3068, International Society for Optics and Photonics,
  1997, pp.~182--193.

\bibitem{julier1995new}
{\sc S.~J. Julier, J.~K. Uhlmann, and H.~F. Durrant-Whyte}, {\em A new approach
  for filtering nonlinear systems}, in Proceedings of 1995 American Control
  Conference-ACC'95, vol.~3, IEEE, 1995, pp.~1628--1632.

\bibitem{kalman1960new}
{\sc R.~E. {K}alman}, {\em A new approach to linear filtering and prediction
  problems}, J. Basic Eng. Mar, 82 (1960), pp.~35--45.

\bibitem{kalman1961new}
{\sc R.~E. Kalman and R.~S. Bucy}, {\em {New Results in Linear Filtering and
  Prediction Theory}}, Journal of Basic Engineering, 83 (1961), pp.~95--108,
  \url{https://doi.org/10.1115/1.3658902},
  \url{https://doi.org/10.1115/1.3658902},
  \url{https://arxiv.org/abs/https://asmedigitalcollection.asme.org/fluidsengineering/article-pdf/83/1/95/5503549/95\_1.pdf}.

\bibitem{karlin1957classification}
{\sc S.~Karlin and J.~McGregor}, {\em The classification of birth and death
  processes}, Transactions of the American Mathematical Society, 86 (1957),
  pp.~366--400.

\bibitem{khan2017conjugate}
{\sc M.~Khan and W.~Lin}, {\em Conjugate-computation variational inference:
  Converting variational inference in non-conjugate models to inferences in
  conjugate models}, in Artificial Intelligence and Statistics, PMLR, 2017,
  pp.~878--887.

\bibitem{khan2018fast}
{\sc M.~Khan, D.~Nielsen, V.~Tangkaratt, W.~Lin, Y.~Gal, and A.~Srivastava},
  {\em Fast and scalable bayesian deep learning by weight-perturbation in
  adam}, in International Conference on Machine Learning, PMLR, 2018,
  pp.~2611--2620.

\bibitem{kim2022hierarchical}
{\sc H.~Kim, D.~Sanz-Alonso, and A.~Strang}, {\em Hierarchical ensemble
  {Kalman} methods with sparsity-promoting generalized gamma hyperpriors},
  arXiv preprint arXiv:2205.09322,  (2022).

\bibitem{korba2020non}
{\sc A.~Korba, A.~Salim, M.~Arbel, G.~Luise, and A.~Gretton}, {\em A
  non-asymptotic analysis for {Stein} variational gradient descent}, Advances
  in Neural Information Processing Systems, 33 (2020), pp.~4672--4682.

\bibitem{lambert2022continuous}
{\sc M.~Lambert, S.~Bonnabel, and F.~Bach}, {\em The continuous-discrete
  variational kalman filter ({CD-VKF})}, in 2022 IEEE 61st Conference on
  Decision and Control (CDC), 2022, pp.~6632--6639,
  \url{https://doi.org/10.1109/CDC51059.2022.9992993}.

\bibitem{lambert2022recursive}
{\sc M.~Lambert, S.~Bonnabel, and F.~Bach}, {\em The recursive variational
  gaussian approximation (r-vga)}, Statistics and Computing, 32 (2022),
  pp.~1--24.

\bibitem{lambert2022variational}
{\sc M.~Lambert, S.~Chewi, F.~Bach, S.~Bonnabel, and P.~Rigollet}, {\em
  Variational inference via {Wasserstein} gradient flows}, arXiv preprint
  arXiv:2205.15902,  (2022).

\bibitem{laschos2019geometric}
{\sc V.~Laschos and A.~Mielke}, {\em Geometric properties of cones with
  applications on the hellinger--kantorovich space, and a new distance on the
  space of probability measures}, Journal of Functional Analysis, 276 (2019),
  pp.~3529--3576.

\bibitem{laugesen2015poisson}
{\sc R.~S. Laugesen, P.~G. Mehta, S.~P. Meyn, and M.~Raginsky}, {\em Poisson's
  equation in nonlinear filtering}, SIAM Journal on Control and Optimization,
  53 (2015), pp.~501--525.

\bibitem{leimkuhler2018ensemble}
{\sc B.~Leimkuhler, C.~Matthews, and J.~Weare}, {\em Ensemble preconditioning
  for markov chain monte carlo simulation}, Statistics and Computing, 28
  (2018), pp.~277--290.

\bibitem{li2023}
{\sc W.~Li and J.~Zhao}, {\em Wasserstein information matrix}, Information
  Geometry,  (2023), \url{https://doi.org/10.1007/s41884-023-00099-9},
  \url{https://doi.org/10.1007/s41884-023-00099-9}.

\bibitem{liese2006divergences}
{\sc F.~Liese and I.~Vajda}, {\em On divergences and informations in statistics
  and information theory}, IEEE Transactions on Information Theory, 52 (2006),
  pp.~4394--4412.

\bibitem{lin2019fast}
{\sc W.~Lin, M.~E. Khan, and M.~Schmidt}, {\em Fast and simple natural-gradient
  variational inference with mixture of exponential-family approximations}, in
  International Conference on Machine Learning, PMLR, 2019, pp.~3992--4002.

\bibitem{lindsey2022ensemble}
{\sc M.~Lindsey, J.~Weare, and A.~Zhang}, {\em Ensemble markov chain monte
  carlo with teleporting walkers}, SIAM/ASA Journal on Uncertainty
  Quantification, 10 (2022), pp.~860--885.

\bibitem{liu2017stein}
{\sc Q.~Liu}, {\em Stein variational gradient descent as gradient flow},
  Advances in neural information processing systems, 30 (2017).

\bibitem{liu2016stein}
{\sc Q.~Liu and D.~Wang}, {\em Stein variational gradient descent: A general
  purpose {Bayesian} inference algorithm}, Advances in neural information
  processing systems, 29 (2016).

\bibitem{liu2022second}
{\sc Z.~Liu, A.~M. Stuart, and Y.~Wang}, {\em Second order ensemble langevin
  method for sampling and inverse problems}, arXiv preprint arXiv:2208.04506,
  (2022).

\bibitem{lopez2022training}
{\sc I.~Lopez-Gomez, C.~Christopoulos, H.~L. Langeland~Ervik, O.~R. Dunbar,
  Y.~Cohen, and T.~Schneider}, {\em Training physics-based machine-learning
  parameterizations with gradient-free ensemble {Kalman} methods}, Journal of
  Advances in Modeling Earth Systems, 14 (2022), p.~e2022MS003105.

\bibitem{Lott08}
{\sc J.~Lott}, {\em Some geometric calculations on {Wasserstein} space},
  Communications in Mathematical Physics, 277 (2008), pp.~423--437,
  \url{https://doi.org/10.1007/s00220-007-0367-3},
  \url{https://doi.org/10.1007/s00220-007-0367-3}.

\bibitem{lu2019scaling}
{\sc J.~Lu, Y.~Lu, and J.~Nolen}, {\em Scaling limit of the {Stein} variational
  gradient descent: The mean field regime}, SIAM Journal on Mathematical
  Analysis, 51 (2019), pp.~648--671.

\bibitem{lu2019accelerating}
{\sc Y.~Lu, J.~Lu, and J.~Nolen}, {\em Accelerating {Langevin} sampling with
  birth-death}, arXiv preprint arXiv:1905.09863,  (2019).

\bibitem{lu2022birth}
{\sc Y.~Lu, D.~Slep{\v{c}}ev, and L.~Wang}, {\em Birth-death dynamics for
  sampling: Global convergence, approximations and their asymptotics}, arXiv
  preprint arXiv:2211.00450,  (2022).

\bibitem{Luigi2018}
{\sc L.~Malag{\`o}, L.~Montrucchio, and G.~Pistone}, {\em Wasserstein
  {R}iemannian geometry of positive definite matrices}, Information Geometry, 1
  (2018), pp.~137--179, \url{https://doi.org/10.1007/s41884-018-0014-4}.

\bibitem{malago2015information}
{\sc L.~Malag{\`o} and G.~Pistone}, {\em Information geometry of the gaussian
  distribution in view of stochastic optimization}, in Proceedings of the 2015
  ACM Conference on Foundations of Genetic Algorithms XIII, 2015, pp.~150--162.

\bibitem{maoutsa2020interacting}
{\sc D.~Maoutsa, S.~Reich, and M.~Opper}, {\em Interacting particle solutions
  of fokker--planck equations through gradient--log--density estimation},
  Entropy, 22 (2020), p.~802.

\bibitem{martens2020new}
{\sc J.~Martens}, {\em New insights and perspectives on the natural gradient
  method}, The Journal of Machine Learning Research, 21 (2020), pp.~5776--5851.

\bibitem{martin2012stochastic}
{\sc J.~Martin, L.~C. Wilcox, C.~Burstedde, and O.~Ghattas}, {\em A stochastic
  {N}ewton {MCMC} method for large-scale statistical inverse problems with
  application to seismic inversion}, SIAM Journal on Scientific Computing, 34
  (2012), pp.~A1460--A1487.

\bibitem{mckean1967propagation}
{\sc H.~P. McKean}, {\em Propagation of chaos for a class of non-linear
  parabolic equations}, Stochastic Differential Equations (Lecture Series in
  Differential Equations, Session 7, Catholic Univ., 1967),  (1967),
  pp.~41--57.

\bibitem{mielke2016generalization}
{\sc A.~Mielke, D.~M. Renger, and M.~A. Peletier}, {\em A generalization of
  onsager’s reciprocity relations to gradient flows with nonlinear mobility},
  Journal of Non-Equilibrium Thermodynamics, 41 (2016), pp.~141--149.

\bibitem{murphy2012machine}
{\sc K.~P. Murphy}, {\em Machine learning: a probabilistic perspective}, MIT
  press, 2012.

\bibitem{nelder1965simplex}
{\sc J.~A. Nelder and R.~Mead}, {\em A simplex method for function
  minimization}, The computer journal, 7 (1965), pp.~308--313.

\bibitem{ollivier2018online}
{\sc Y.~Ollivier}, {\em Online natural gradient as a {Kalman} filter},
  Electronic Journal of Statistics, 12 (2018), pp.~2930--2961.

\bibitem{ollivier2019extended}
{\sc Y.~Ollivier}, {\em The extended {Kalman} filter is a natural gradient
  descent in trajectory space}, arXiv preprint arXiv:1901.00696,  (2019).

\bibitem{onsager1931reciprocal}
{\sc L.~Onsager}, {\em Reciprocal relations in irreversible processes. i.},
  Physical review, 37 (1931), p.~405.

\bibitem{onsager1931reciprocal-II}
{\sc L.~Onsager}, {\em Reciprocal relations in irreversible processes. ii.},
  Physical review, 38 (1931), p.~2265.

\bibitem{opper2009variational}
{\sc M.~Opper and C.~Archambeau}, {\em The variational {G}aussian approximation
  revisited}, Neural computation, 21 (2009), pp.~786--792.

\bibitem{otto2001geometry}
{\sc F.~Otto}, {\em The geometry of dissipative evolution equations: The porous
  medium equation}, Communications in Partial Differential Equations, 26
  (2001), pp.~101--174, \url{https://doi.org/10.1081/PDE-100002243},
  \url{https://doi.org/10.1081/PDE-100002243},
  \url{https://arxiv.org/abs/https://doi.org/10.1081/PDE-100002243}.

\bibitem{pavliotis2014stochastic}
{\sc G.~A. Pavliotis}, {\em Stochastic processes and applications: diffusion
  processes, the Fokker-Planck and Langevin equations}, vol.~60, Springer,
  2014.

\bibitem{pidstrigach2021affine}
{\sc J.~Pidstrigach and S.~Reich}, {\em Affine-invariant ensemble transform
  methods for logistic regression}, arXiv preprint arXiv:2104.08061,  (2021).

\bibitem{pidstrigach2023affine}
{\sc J.~Pidstrigach and S.~Reich}, {\em Affine-invariant ensemble transform
  methods for logistic regression}, Foundations of Computational Mathematics,
  23 (2023), pp.~675--708.

\bibitem{poincare1890equations}
{\sc H.~Poincar{\'e}}, {\em Sur les {\'e}quations aux d{\'e}riv{\'e}es
  partielles de la physique math{\'e}matique}, American Journal of Mathematics,
   (1890), pp.~211--294.

\bibitem{quiroz2018gaussian}
{\sc M.~Quiroz, D.~J. Nott, and R.~Kohn}, {\em Gaussian variational
  approximation for high-dimensional state space models}, arXiv preprint
  arXiv:1801.07873,  (2018).

\bibitem{rao1945information}
{\sc C.~R. Rao}, {\em Information and the accuracy attainable in the estimation
  of statistical parameters}, Reson. J. Sci. Educ, 20 (1945), pp.~78--90.

\bibitem{rasmussen2003gaussian}
{\sc C.~E. Rasmussen}, {\em Gaussian processes in machine learning}, in Summer
  school on machine learning, Springer, 2003, pp.~63--71.

\bibitem{reich2011dynamical}
{\sc S.~Reich}, {\em A dynamical systems framework for intermittent data
  assimilation}, BIT Numerical Mathematics, 51 (2011), pp.~235--249.

\bibitem{rellich1969perturbation}
{\sc F.~Rellich and J.~Berkowitz}, {\em Perturbation theory of eigenvalue
  problems}, CRC Press, 1969.

\bibitem{salimans2018improving}
{\sc T.~Salimans, H.~Zhang, A.~Radford, and D.~Metaxas}, {\em Improving {GANs}
  using optimal transport}, arXiv preprint arXiv:1803.05573,  (2018).

\bibitem{santambrogio2017euclidean}
{\sc F.~Santambrogio}, {\em $\{$Euclidean, metric, and {Wasserstein}$\}$
  gradient flows: an overview}, Bulletin of Mathematical Sciences, 7 (2017),
  pp.~87--154.

\bibitem{sanz2018inverse}
{\sc D.~Sanz-Alonso, A.~M. Stuart, and A.~Taeb}, {\em Inverse problems and data
  assimilation}, arXiv preprint arXiv:1810.06191,  (2018).

\bibitem{sarkka2007unscented}
{\sc S.~S\"arkk\"a}, {\em On unscented {Kalman} filtering for state estimation
  of continuous-time nonlinear systems}, IEEE Transactions on automatic
  control, 52 (2007), pp.~1631--1641.

\bibitem{schneider2017earth}
{\sc T.~Schneider, S.~Lan, A.~Stuart, and J.~Teixeira}, {\em Earth system
  modeling 2.0: {A} blueprint for models that learn from observations and
  targeted high-resolution simulations}, Geophysical Research Letters, 44
  (2017), pp.~12--396.

\bibitem{semmesintroduction}
{\sc S.~Semmes}, {\em An introduction to some aspects of functional analysis},
  https://math.rice.edu/~semmes/fun5.pdf.

\bibitem{shen2022self}
{\sc Z.~Shen, Z.~Wang, S.~Kale, A.~Ribeiro, A.~Karbasi, and H.~Hassani}, {\em
  Self-consistency of the fokker-planck equation}, arXiv preprint
  arXiv:2206.00860,  (2022).

\bibitem{song2020sliced}
{\sc Y.~Song, S.~Garg, J.~Shi, and S.~Ermon}, {\em Sliced score matching: A
  scalable approach to density and score estimation}, in Uncertainty in
  Artificial Intelligence, PMLR, 2020, pp.~574--584.

\bibitem{sorenson1985kalman}
{\sc H.~W. Sorenson}, {\em {K}alman Filtering: {T}heory and Application}, IEEE,
  1985.

\bibitem{srivastava2007riemannian}
{\sc A.~Srivastava, I.~Jermyn, and S.~Joshi}, {\em Riemannian analysis of
  probability density functions with applications in vision}, in 2007 IEEE
  Conference on Computer Vision and Pattern Recognition, IEEE, 2007, pp.~1--8.

\bibitem{subrahmanya2021ensemble}
{\sc A.~N. Subrahmanya, A.~A. Popov, and A.~Sandu}, {\em An ensemble
  variational {Fokker-Planck} method for data assimilation}, arXiv preprint
  arXiv:2111.13926,  (2021).

\bibitem{sznitman1991topics}
{\sc A.-S. Sznitman}, {\em Topics in propagation of chaos}, in Ecole
  d'{\'e}t{\'e} de probabilit{\'e}s de Saint-Flour XIX—1989, Springer, 1991,
  pp.~165--251.

\bibitem{takatsu2011wasserstein}
{\sc A.~Takatsu}, {\em Wasserstein geometry of gaussian measures}, Osaka
  Journal of Mathematics, 48 (2011), pp.~1005--1026.

\bibitem{villani2009optimal}
{\sc C.~Villani}, {\em Optimal transport: old and new}, vol.~338, Springer,
  2009.

\bibitem{villani2021topics}
{\sc C.~Villani}, {\em Topics in optimal transportation}, vol.~58, American
  Mathematical Soc., 2021.

\bibitem{wainwright2008graphical}
{\sc M.~J. Wainwright, M.~I. Jordan, et~al.}, {\em Graphical models,
  exponential families, and variational inference}, Foundations and
  Trends{\textregistered} in Machine Learning, 1 (2008), pp.~1--305.

\bibitem{wan2000unscented}
{\sc E.~A. Wan and R.~Van Der~Merwe}, {\em The unscented {K}alman filter for
  nonlinear estimation}, in Proceedings of the IEEE 2000 Adaptive Systems for
  Signal Processing, Communications, and Control Symposium (Cat. No. 00EX373),
  Ieee, 2000, pp.~153--158.

\bibitem{wang2022projected}
{\sc Y.~Wang, P.~Chen, and W.~Li}, {\em Projected wasserstein gradient descent
  for high-dimensional bayesian inference}, SIAM/ASA Journal on Uncertainty
  Quantification, 10 (2022), pp.~1513--1532.

\bibitem{wang2022optimal}
{\sc Y.~Wang, P.~Chen, M.~Pilanci, and W.~Li}, {\em Optimal neural network
  approximation of wasserstein gradient direction via convex optimization},
  arXiv preprint arXiv:2205.13098,  (2022).

\bibitem{wang2020information}
{\sc Y.~Wang and W.~Li}, {\em Information {Newton's} flow: second-order
  optimization method in probability space}, arXiv preprint arXiv:2001.04341,
  (2020).

\bibitem{wang2022accelerated}
{\sc Y.~Wang and W.~Li}, {\em Accelerated information gradient flow}, Journal
  of Scientific Computing, 90 (2022), pp.~1--47.

\bibitem{weissmann2021adaptive}
{\sc S.~Weissmann, N.~K. Chada, C.~Schillings, and X.~T. Tong}, {\em Adaptive
  {T}ikhonov strategies for stochastic ensemble {K}alman inversion}, arXiv
  preprint arXiv:2110.09142,  (2021).

\bibitem{yuen2010bayesian}
{\sc K.-V. Yuen}, {\em Bayesian methods for structural dynamics and civil
  engineering}, John Wiley \& Sons, 2010.

\bibitem{yumei2022variational}
{\sc H.~Yumei, W.~Xuezhi, P.~Quan, H.~Zhentao, and B.~Moran}, {\em Variational
  bayesian {Kalman} filter using natural gradient}, Chinese Journal of
  Aeronautics, 35 (2022), pp.~1--10.

\bibitem{zhang2019fast}
{\sc G.~Zhang, J.~Martens, and R.~B. Grosse}, {\em Fast convergence of natural
  gradient descent for over-parameterized neural networks}, Advances in Neural
  Information Processing Systems, 32 (2019).

\bibitem{zhang2020wasserstein}
{\sc K.~S. Zhang, G.~Peyr{\'e}, J.~Fadili, and M.~Pereyra}, {\em Wasserstein
  control of mirror langevin monte carlo}, in Conference on Learning Theory,
  PMLR, 2020, pp.~3814--3841.

\end{thebibliography}

%
%

\appendix
\newpage
\section{Unique Property of the KL Divergence Energy}
\subsection{Proof of \Cref{prop: KL unique f divergence}}
\label{sec: Proofs of prop: KL unique f divergence}
\begin{proof}
We first note that the KL divergence satisfies the desired property:
for any $c \in (0,\infty)$ and $\rho_{\rm post} \in \PP$ it holds that
$$\mathrm{KL}[\rho \Vert  c\rho_{\rm post}]-\mathrm{KL}[\rho \Vert  \rho_{\rm post}] = -\log c.$$ 
Now we establish uniqueness. For any $f-$divergence with property that
$D_f[\rho \Vert c\rho_{\rm post}]-D_f[\rho \Vert \rho_{\rm post}]$ is independent of $\rho$, we have
\begin{equation}
    \lim_{t \to 0} \frac{(D_f[\rho + t\sigma \Vert c\rho_{\rm post}]-D_f[\rho + t\sigma \Vert \rho_{\rm post}]) - (D_f[\rho \Vert c\rho_{\rm post}]-D_f[\rho \Vert \rho_{\rm post}]) }{t} = 0,
\end{equation}
for any bounded, smooth function $\sigma$ supported in $B^d(0,R)$ that integrates to zero. Here $R > 0$ is a finite parameter that we will choose later, and $d=N_{\theta}$ is the dimension of $\theta$. In the above formula we have used the fact that for sufficiently small $t$, one has $\rho+t\sigma \in \PP$ since $R$ is finite. By direct calculations, we get
\begin{equation}
    \int_{B^d(0,R)} \left(f'(\frac{\rho}{c\rho_{\rm post}}) -  f'(\frac{\rho}{\rho_{\rm post}})\right)\sigma {\rm d}\theta= 0.
\end{equation}
Since $\sigma$ is arbitrary, $f'(\frac{\rho}{c\rho_{\rm post}}) -  f'(\frac{\rho}{\rho_{\rm post}})$ must be a constant function in $B^d(0,R)$. 



Because $\rho$ and $\rho_{\rm post}$ integrate to $1$ and they are continuous, there exists $\theta^\dagger$ such that $\rho (\theta^\dagger)/\rho_{\rm post}(\theta^\dagger) = 1$. Choose $R$ sufficiently large such that $\theta^\dagger \in B^d(0,R)$. Then, we obtain
\begin{equation}
    \label{eq:A1}
f'\Bigl(\frac{\rho(\theta)}{\rho_{\rm post}(\theta)}\Bigr) - 
f'\Bigl(\frac{\rho(\theta)}{c\rho_{\rm post}(\theta)}\Bigr)
=f'\Bigl(\frac{\rho(\theta^\dagger)}{\rho_{\rm post}(\theta^\dagger)}\Bigr) - 
f'\Bigl(\frac{\rho(\theta^\dagger)}{c\rho_{\rm post}(\theta^\dagger)}\Bigr) = f'(1) - f'(1/c),
\end{equation}
for any $\theta \in B^d(0,R)$. As $R$ can be arbitrarily large, the above identity holds for all $\theta \in \R^d$. Let $g(1/c):=f'(1) - f'(1/c)$. We have obtained
\begin{equation}
\label{eqn A1}
    f'\Bigl(\frac{\rho(\theta)}{\rho_{\rm post}(\theta)}\Bigr) - f'\Bigl(\frac{\rho(\theta)}{c\rho_{\rm post}(\theta)}\Bigr) = g(1/c),
\end{equation}
where $c$ is an arbitrary positive number. Note, furthermore, that $g(\cdot)$ 
is continuous since $f$ is continuously differentiable. 
Since $\rho$ and $\rho_{\rm post}$ are arbitrary, 
we can write \eqref{eqn A1} equivalently as
\begin{equation}
\label{eqn: KL unique f divergence, function equality}
    f'(y)-f'(cy) = g(c),
\end{equation}
for any $y, c \in \mathbb{R}_{+}$. Let $h: \mathbb{R} \to \mathbb{R}$ such that $h(z) = f'(\exp(z))$. Then, we can equivalently formulate \eqref{eqn: KL unique f divergence, function equality} as
\begin{equation}
\label{eq:a3}
    h(z_1) - h(z_2) = r(z_1 - z_2),
\end{equation}
for any $z_1, z_2\in \mathbb{R}$ and $r:\mathbb{R} \to \mathbb{R}$ such that $r(t) = g(\exp(-t))$. 

We can show $r$ is linear function. Setting $z_1=z_2$ in \eqref{eq:a3}
shows that $r(0)=0.$ Note also that, again by \eqref{eq:a3},
$$r(z_1-z_2)+r(z_2-z_3)=h(z_1)-h(z_3)=r(z_1-z_3).$$
Hence, since $z_1,z_2$ and $z_3$ are arbitrary, we deduce that
for any $x, y \in \mathbb{R}$, it holds that
\begin{equation}
\label{eqn A4}
    r(x) + r(y) = r(x+y).
\end{equation}
Furthermore $r$ is continuous since $f$ is assumed continuously
differentiable. With the above conditions, it is a standard result in functional equations that $r(x)$ is linear. Indeed, as a sketch of proof, by \eqref{eqn A4} we can first deduce $r(n) = n r(1)$ for $n \in \mathbb{Z}$. Then, by setting $x,y$ to be dyadic rationals, we can deduce $r(\frac{i}{2^k}) = \frac{i}{2^k} r(1)$ for any $i, k \in \mathbb{Z}$. Finally using the continuity of the function $r$, we get $r(x) = xr(1)$ for any $x \in \mathbb{R}$. For more details see \cite{friedman1962functional}.

Using the fact that $r$ is linear and the equation \eqref{eq:a3}, we know that $h$ is an affine function and thus $f'(\exp(z)) = az+b$ for some $a, b \in \mathbb{R}$. Equivalently, $f'(y) = a \log(y) + b$. Using the condition $f(1) = 0$, we get $f(y) = a y \log(y) + (b-a)(y-1)$. Plugging this $f$ into the formula for $D_f$, we get
\[D_f[\rho||\rho_{\rm post}] = a \text{KL}[\rho||\rho_{\rm post}],\]
noting that the affine term in $f(y)$ has zero contributions in the formula for $D_f$.
The proof is complete.
\end{proof}

\newpage
\section{Proofs for Affine Invariant Gradient Flows}
\label{appendix:affine-invariant}

\subsection{Preliminaries}
In this section, we present some lemmas that are useful for our proofs in later sections. The first lemma concerns the change-of-variable formula under 
the gradient and divergence operators.
\begin{lemma}
\label{lemma:change_variable}
Consider any invertible affine mapping from $\theta \in \R^{N_{\theta}}$ to  $\tilde\theta \in \R^{N_{\theta}}$ defined by $\tilde\theta =\varphi(\theta)= A \theta + b$. 
For any differentiable scalar field $f: \R^{N_{\theta}} \rightarrow \R$ and 
vector field $g: \R^{N_{\theta}} \rightarrow \R^{N_{\theta}}$, we have
$$\nabla_{\theta} f(\theta) = A^{T} \nabla_{\tilde{\theta}} \tilde{f}(\tilde\theta) \quad \text{and} \quad \nabla_{\theta} \cdot g(\theta) =  \nabla_{\tilde{\theta}} \cdot ( A \tilde{g} (\tilde\theta) ), $$
where $\tilde{f}(\tilde{\theta}) := f(A^{-1}(\tilde{\theta} - b))$ and $\tilde{g}(\tilde{\theta}) := g(A^{-1}(\tilde{\theta} - b))$.
\end{lemma}
\begin{proof} Note that $\tilde{f}(\tilde{\theta}) = f(\theta)$ and $\tilde{g}(\tilde{\theta}) = g(\theta)$. By direct calculations, we get
\begin{align*}
&[\nabla_{\theta} f(\theta)]_i = \frac{\partial f(\theta)}{\partial \theta_i} 
= 
\sum_j\frac{\partial f(\theta)}{\partial \tilde{\theta}_j}\frac{\partial \tilde{\theta}_j }{\partial \theta_i} 
= 
\sum_j\frac{\partial \tilde{f}(\tilde{\theta})}{\partial \tilde{\theta}_j}A_{ji} = [A^T\nabla_{\tilde{\theta}} \tilde{f}(\tilde{\theta})]_i,
\\
    &\nabla_{\theta} \cdot g(\theta) = \sum_i \frac{\partial g_i(\theta)}{\partial\theta_i} = \sum_i \sum_j \frac{\partial g_i(\theta)}{\partial\tilde{\theta}_j}\frac{\partial \tilde{\theta}_j}{\partial \theta_i} =  \sum_i \sum_j \frac{\partial \tilde{g}_i(\tilde{\theta})}{\partial\tilde{\theta}_j}A_{ji} = \nabla_{\tilde{\theta}} \cdot ( A \tilde{g} (\tilde\theta) ).
\end{align*}
This completes the proof.
\end{proof}

\begin{newremark}
\label{remark: change of variable grad and div}
    Since $\tilde{f}(\tilde{\theta}) = f(\theta)$ and $\tilde{g}(\tilde{\theta}) = g(\theta)$, we can also summarize the result in \cref{lemma:change_variable} as $\nabla_{\theta} f = A^T\nabla_{\tilde{\theta}}f$ and $\nabla_{\theta}\cdot g = \nabla_{\tilde{\theta}}\cdot (Ag)$.
\end{newremark}

The second lemma concerns the change-of-variable formula 
for the first variation of functional.
\begin{lemma}
\label{lemma:first_variation_change_variable}
Consider any invertible mapping from $\theta \in \R^{N_{\theta}}$ to  $\tilde\theta \in \R^{N_{\theta}}$ defined by $\tilde\theta = \varphi(\theta)$. 
Then, for any differentiable functional $\mathcal{E}$ in probability space, it holds that 
\[\Bigl\langle \frac{\delta \tilde{\mathcal{E}}}{\delta \tilde\rho}, \tilde{\sigma} \Bigr\rangle=\Bigl\langle \frac{\delta {\mathcal{E}}}{\delta \rho} , {\sigma} \Bigr\rangle, \]
for every $\sigma \in T_{\rho}\PP$ and $\tilde{\sigma} = \varphi \# \sigma \in T_{\tilde{\rho}}\PP$. Here, $\tilde{\mathcal{E}}$ and $\tilde\rho$ are defined in \cref{def:pushforward}. Consequently, it follows that
\[\frac{\delta \tilde{\mathcal{E}}}{\delta \tilde\rho}(\tilde{\theta}) = \frac{\delta \mathcal{E}}{\delta \rho}(\theta). \]

\end{lemma}

\begin{proof}
By the definition of the first variation, for any $\tilde{\sigma} \in T_{\tilde{\rho}}\PP$, we have that
\begin{align*}
\Bigl\langle \frac{\delta \tilde{\mathcal{E}}}{\delta \tilde\rho}, \tilde{\sigma} \Bigr\rangle
&= \lim_{\epsilon \rightarrow 0}\frac{\tilde{\mathcal{E}}(\tilde\rho + \epsilon \tilde\sigma) - \tilde{\mathcal{E}}(\tilde\rho + \epsilon \tilde\sigma)}{\epsilon} \\
&= \lim_{\epsilon \rightarrow 0}\frac{{\mathcal{E}}(\varphi^{-1}\#\tilde\rho + \epsilon \varphi^{-1}\#\tilde\sigma) - {\mathcal{E}}(\varphi^{-1}\#\tilde\rho + \epsilon \varphi^{-1}\#\tilde\sigma)}{\epsilon}\\
&= \lim_{\epsilon \rightarrow 0}\frac{{\mathcal{E}}(\rho + \epsilon \sigma) - {\mathcal{E}}(\rho + \epsilon \sigma)}{\epsilon} \\
&= \Bigl\langle \frac{\delta {\mathcal{E}}}{\delta \rho} , {\sigma} \Bigr\rangle.
\end{align*}
Here in the second equality, we have used the fact that $\tilde{\EE}(\tilde{\rho}) = \EE(\varphi^{-1}\# \tilde{\rho})$ as in \cref{def:pushforward}; moreover, $\tilde{\sigma} = \varphi \# \sigma$ by definition.

Now, we write out the integral explicitly:
\begin{equation*}
\begin{aligned}    
\Bigl\langle \frac{\delta \tilde{\mathcal{E}}}{\delta \tilde\rho}, \tilde{\sigma} \Bigr\rangle = \int \frac{\delta \tilde\EE}{\delta \tilde\rho}(\tilde \theta) \tilde\sigma(\tilde{\theta}) \mathrm{d}\tilde{\theta} = \int \frac{\delta \tilde\EE}{\delta \tilde\rho}(\varphi(\theta)) \sigma(\theta) \mathrm{d}\theta,
\end{aligned}
\end{equation*}
where in the last equality, we used the change-of-variable $\tilde{\theta}=\varphi(\theta)$ and the fact that $\tilde{\sigma}(\tilde\theta) = \sigma(\varphi^{-1}(\tilde{\theta}))|\nabla_{\tilde{\theta}} \varphi^{-1}(\tilde{\theta})|$. Combining the above relations, we obtain
\[\int \frac{\delta \tilde\EE}{\delta \tilde\rho}(\varphi(\theta)) \sigma(\theta) \mathrm{d}\theta = \Bigl\langle \frac{\delta {\mathcal{E}}}{\delta \rho} , {\sigma} \Bigr\rangle = \int \frac{\delta {\mathcal{E}}}{\delta \rho}(\theta)\sigma(\theta)\mathrm{d}\theta,\]
which holds for any $\sigma \in T_{\rho}\PP$. We deduce that 
\[\frac{\delta \tilde{\mathcal{E}}}{\delta \tilde\rho}(\varphi(\theta)) = \frac{\delta \mathcal{E}}{\delta \rho}(\theta).\]
Here we used the fact that the integration of any element in $T_{\rho}^{*}\PP$ is 0 so the above identity holds exactly.
The proof is complete noticing that $\tilde{\theta}=\varphi(\theta)$.
\end{proof}

\subsection{Proof of \Cref{prop: affine invariant metric tensor}} 
\label{proof: prop affine invariant metric tensor}

\begin{proof}
From the definition, it suffices to show the equivalence of 
\begin{itemize}
    \item[(a)] $\varphi \# (\nabla \mathcal{E}(\rho)) = \nabla \tilde{\mathcal{E}}(\tilde{\rho})$ for any $\mathcal{E}$;
    \item[(b)] $\varphi^{\#}g = g$.
\end{itemize}

From (b) to (a): For any $\sigma \in T_{\rho}\PP$, by the definition of the pushforward and pull-back operators, we have
\[(\varphi^{\#} g)_{\rho}(\nabla \mathcal{E}(\rho),\sigma) = g_{\tilde{\rho}}(\varphi \# (\nabla \mathcal{E}(\rho)), \tilde{\sigma}), \]
where $\tilde{\sigma} = \varphi \# \sigma \in T_{\tilde{\rho}}\PP$. Thus, by (b) we get $g_{\rho}(\nabla \mathcal{E}(\rho),\sigma) = g_{\tilde{\rho}}(\varphi \# (\nabla \mathcal{E}(\rho)), \tilde{\sigma})$. 

Using the definition of the gradient operator and \cref{lemma:first_variation_change_variable}, we have
\[g_{\rho}(\nabla \mathcal{E}(\rho),\sigma) = \Bigl\langle\frac{\delta \mathcal{E}}{\delta \rho},\sigma\Bigr\rangle 
= \Bigl\langle\frac{\delta \tilde{\mathcal{E}}}{\delta \tilde{\rho}},\tilde{\sigma}\Bigr\rangle. \]
Combining the two relations above leads to
$\Bigl\langle
\frac{\delta \tilde{\mathcal{E}}}{\delta \tilde{\rho}},\tilde{\sigma} \Bigr\rangle = g_{\tilde{\rho}}(\varphi \# (\nabla \mathcal{E}(\rho)), \tilde{\sigma})
$.  Then, using the definition of the gradient operator, we get $\varphi \# (\nabla \mathcal{E}(\rho)) = \nabla \tilde{\mathcal{E}}(\tilde{\rho})$.

From (a) to (b): similar as above, we have the relation
\[g_{\rho}(\nabla \mathcal{E}(\rho),\sigma) = \Bigl\langle\frac{\delta \mathcal{E}}{\delta \rho}, \sigma\Bigr\rangle= \Bigl\langle\frac{\delta \tilde{\mathcal{E}}}{\delta \tilde{\rho}},\tilde{\sigma}\Bigr\rangle = g_{\tilde{\rho}}(\nabla\tilde{\mathcal{E}}(\tilde{\rho}), \tilde{\sigma}). \]
By (a), it holds that
\[g_{\tilde{\rho}}(\nabla\tilde{\mathcal{E}}(\tilde{\rho}), \tilde{\sigma}) = g_{\tilde{\rho}}(\varphi \# (\nabla\mathcal{E}(\rho)), \tilde{\sigma}) = (\varphi^{\#}g)_{\rho} (\nabla\mathcal{E}(\rho), \sigma), \]
where in the last equality we used the definition of the pull-back operator.

Combining the two relations above and noticing the fact that $\mathcal{E}$ and $\sigma$ are arbitrary, we get $\varphi^{\#}g = g$. This completes the proof.
\end{proof}

\subsection{Proof of \Cref{{prop: affine invariance of mean-field dynamics}}}
\label{appendix: Proof of prop: affine invariance of mean-field dynamics}
By definition, $f$ and $h$ satisfy the equation
\begin{equation}
\label{appendix eqn: mean-field equation PDE nontransformed}
    -\nabla_{\theta} \cdot (\rho f) + \frac{1}{2}\nabla_{\theta}\cdot (\nabla_{\theta} \cdot (hh^T\rho)) = -M(\rho)^{-1}\frac{\delta \mathcal{E}(\rho;\rho_{\rm post})}{\delta \rho}.
\end{equation}
For a fixed $\mathcal{E}$, one can write $f = f(\theta; \rho, \rho_{\rm post}), h = h(\theta; \rho, \rho_{\rm post})$. 

Consider an invertible affine transformation $\tilde{\theta} = \varphi(\theta) = A\theta + b$, and correspondingly $\tilde{\rho} = \varphi \# \rho, \tilde{\rho}_{\rm post} = \varphi \# \rho_{\rm post}$. Under this new coordinate, we obtain the corresponding equation to determine the drift $\tilde{f}$ and diffusion coefficient $\tilde{h}$: 
\begin{equation}
\label{appendix eqn: mean-field equation PDE transformed}
    -\nabla_{\tilde{\theta}} \cdot (\tilde{\rho} \tilde{f}) + \frac{1}{2}\nabla_{\tilde{\theta}}\cdot (\nabla_{\tilde{\theta}} \cdot (\tilde{h}\tilde{h}^T\tilde{\rho})) = -M(\tilde{\rho})^{-1}\frac{\delta \mathcal{E}(\tilde{\rho};\tilde{\rho}_{\rm post})}{\delta \tilde{\rho}}.
\end{equation}
This identity can be satisfied by many  
$\tilde{f}$, $\tilde{h}.$
Our goal is to show that the choice 
\begin{equation}
\label{appendix eqn: choice of drift and diffusion}
    \begin{aligned}
        \tilde{f}(\tilde{\theta}; \tilde{\rho}, \tilde{\rho}_{\rm post}) &= Af(\theta;\rho,\rho_{\rm post})\\
        \tilde{h}(\tilde{\theta}; \tilde{\rho}, \tilde{\rho}_{\rm post}) &= Ah(\theta;\rho,\rho_{\rm post})
    \end{aligned}
\end{equation}
satisfies the equation \cref{appendix eqn: mean-field equation PDE transformed}. 

First, by the change-of-variable formula, we have \begin{equation}
\label{appendix eqn change of variable for rho}
    \tilde{\rho}(\tilde{\theta}) = \rho(\varphi^{-1}(\tilde{\theta}))|\nabla_{\tilde{\theta}} \varphi^{-1}(\tilde{\theta})| = \rho(\theta)|A^{-1}|.
\end{equation}
For the $\tilde{f}$ in \cref{appendix eqn: choice of drift and diffusion}, we calculate
\begin{equation}
\label{appendix eqn: change of variable drift term PDE}
    \begin{aligned}
        -\nabla_{\tilde{\theta}} \cdot (\tilde{\rho}(\tilde{\theta}) \tilde{f}(\tilde{\theta}; \tilde{\rho}, \tilde{\rho}_{\rm post})) &= -\nabla_{\theta} \cdot (A^{-1}\tilde{\rho}(\tilde{\theta}) \tilde{f}(\tilde{\theta}; \tilde{\rho}, \tilde{\rho}_{\rm post}))\\
        & = -\nabla_{\theta} \cdot (\rho(\theta) f(\theta;\rho,\rho_{\rm post}))|A^{-1}|,
    \end{aligned}
\end{equation}
where in the first equality, we used the change-of-variable formula for the divergence operator (\cref{remark: change of variable grad and div}); in the second equality, we used \cref{appendix eqn: choice of drift and diffusion} and \cref{appendix eqn change of variable for rho}. Similarly, for the $\tilde{h}$ in \cref{appendix eqn: choice of drift and diffusion}, we can get
\begin{equation}
    \frac{1}{2}\nabla_{\tilde{\theta}}\cdot (\nabla_{\tilde{\theta}} \cdot (\tilde{h}\tilde{h}^T\tilde{\rho})) = \frac{1}{2}\nabla_{\theta}\cdot (\nabla_{\theta} \cdot (hh^T\rho)) |A^{-1}|.
\end{equation}
Here we note that the left hand side is evaluated at $\tilde{\theta}$ while the right hand side is evaluated at $\theta$, similar to \cref{appendix eqn: change of variable drift term PDE}.

Thus, for the choice of $\tilde{f}$ and $\tilde{h}$ in \cref{appendix eqn: choice of drift and diffusion}, it holds that
\begin{equation}
    -\nabla_{\tilde{\theta}} \cdot (\tilde{\rho} \tilde{f}) + \frac{1}{2}\nabla_{\tilde{\theta}}\cdot (\nabla_{\tilde{\theta}} \cdot (\tilde{h}\tilde{h}^T\tilde{\rho})) = -M(\rho)^{-1}\frac{\delta \mathcal{E}(\rho;\rho_{\rm post})}{\delta \rho} |A^{-1}|, 
\end{equation}
where we used the fact that $f$ and $h$ satisfy \cref{appendix eqn: mean-field equation PDE nontransformed}. The right hand side is evaluated at $\theta$. 

Now, in order for \cref{appendix eqn: mean-field equation PDE transformed} to hold, it suffices to show
\begin{equation}
\label{appendix eqn: mean-field equation affine invarince suffice to show}
    \left(M(\rho)^{-1}\frac{\delta \mathcal{E}(\rho;\rho_{\rm post})}{\delta \rho}\right)(\theta) |A^{-1}| = \left(M(\tilde{\rho})^{-1}\frac{\delta \mathcal{E}(\tilde{\rho};\tilde{\rho}_{\rm post})}{\delta \tilde{\rho}}\right)(\tilde{\theta}).
\end{equation}
Using the condition that the gradient flow in probability space is affine invariant, we have $\varphi \# (\nabla \mathcal{E}(\rho)) = \nabla \tilde{\mathcal{E}}(\tilde{\rho})$, or equivalently, 
\begin{equation}
\label{appendix eqn: affine invariance of gf}
    \left(M(\rho)^{-1}\frac{\delta \mathcal{E}(\rho;\rho_{\rm post})}{\delta \rho}\right)(\theta) |A^{-1}| = \left(M(\tilde{\rho})^{-1}\frac{\delta \tilde{\mathcal{E}}(\tilde{\rho};\rho_{\rm post})}{\delta \tilde{\rho}}\right)(\tilde{\theta}),
\end{equation}
where we used the definition of the gradient and the push forward operator. Here $\tilde{\mathcal{E}} = \varphi \# \mathcal{E}$. Using the assumption, we have \[\tilde{\mathcal{E}}(\tilde{\rho};\rho_{\rm post}) = \mathcal{E}(\varphi^{-1} \# \tilde{\rho};\rho_{\rm post}) = \mathcal{E}(\tilde{\rho};\varphi \# \rho_{\rm post}) = \mathcal{E}(\tilde{\rho};\tilde{\rho}_{\rm post}).\] Combining it with \cref{appendix eqn: affine invariance of gf}, we get the desired \cref{appendix eqn: mean-field equation affine invarince suffice to show}. The proof is complete.

\subsection{Proof of \Cref{proposition:AI-Fisher-Rao-MD}}
\label{proof:AI-Fisher-Rao-MD}
\begin{proof}
    Let $\phi$ be chosen so that 
    $$f(\theta; \rho, \rho_{\rm post}) := P(\theta,\rho)\nabla_{\theta}\phi(\theta; \rho, \rho_{\rm post})$$
    satisfies the identity
    \begin{equation} 
-\nabla_\theta \cdot (\rho f) = 
\rho \bigl(\log \rho_{\rm post} - \log \rho\bigr) - \rho\E_{\rho}[\log \rho_{\rm post} - \log \rho].
\end{equation}
Consider an invertible affine transformation $\tilde{\theta} = \varphi(\theta) = A\theta + b$, and correspondingly $\tilde{\rho} = \varphi \# \rho, \tilde{\rho}_{\rm post} = \varphi \# \rho_{\rm post}$. Using the proof of \cref{{prop: affine invariance of mean-field dynamics}}, which can be found in \Cref{appendix: Proof of prop: affine invariance of mean-field dynamics}, we know that $\tilde{f}(\tilde{\theta}; \tilde{\rho}, \tilde{\rho}_{\rm post}) = Af(\theta;\rho,\rho_{\rm post})$ satisfies the identity
    \begin{equation} 
-\nabla_{\tilde \theta} \cdot (\tilde{\rho} \tilde{f}) = 
\tilde{\rho} \bigl(\log \tilde{\rho}_{\rm post} - \log \tilde{\rho}\bigr) - \tilde{\rho}\E_{\tilde{\rho}}[\log \tilde{\rho}_{\rm post} - \log \tilde{\rho}].
\end{equation}
Notice that
\begin{equation}
\label{eqn: proof:AI-Fisher-Rao-MD, eqn a}
\begin{aligned}
    \tilde{f}(\tilde{\theta}; \tilde{\rho}, \tilde{\rho}_{\rm post}) &= AP(\theta,\rho)\nabla_{\theta}\phi(\theta; \rho, \rho_{\rm post})\\
    & = AP(\theta,\rho)A^T\nabla_{\tilde{\theta}}\phi(\theta; \rho, \rho_{\rm post})\\
    & = P(\tilde{\theta}, \tilde{\rho}) \nabla_{\tilde{\theta}}\phi(\theta; \rho, \rho_{\rm post}),
\end{aligned}
\end{equation}
where in the second equality, we used the change-of-variable formula from \cref{remark: change of variable grad and div}. In the last equality, we used the assumed condition on $P$. 

Now using the uniqueness of the solution to the equation \eqref{eqn: affine invariant Fisher Rao mean field dynamics}, we get 
\begin{equation}
\label{eqn: proof:AI-Fisher-Rao-MD, eqn b}
    \phi(\theta; \rho, \rho_{\rm post}) = \phi(\tilde{\theta}; \tilde{\rho}, \tilde{\rho}_{\rm post}) \quad \text{up to constants}.
\end{equation}
We can use the above result to prove that the mean-field dynamics is affine invariant. By definition, it suffices to show \[AP(\theta,\rho)\nabla_{\theta}\phi(\theta; \rho, \rho_{\rm post}) = P(\tilde{\theta}, \tilde{\rho}) \nabla_{\tilde{\theta}}\phi(\tilde{\theta}; \tilde{\rho}, \tilde{\rho}_{\rm post}).\]
This is readily true by combining the calculation in \eqref{eqn: proof:AI-Fisher-Rao-MD, eqn a} and the condition \eqref{eqn: proof:AI-Fisher-Rao-MD, eqn b}. The proof is complete.
\end{proof}

\subsection{Proof of \Cref{proposition:Wasserstein-affine-invariant}}
\label{proof:Wasserstein-affine-invariant}

\begin{proof}
By \cref{prop: affine invariant metric tensor}, it suffices to show that the gradient flows for any $\mathcal{E}$ is affine invariant. We write down the form of the corresponding gradient flow as follows:
\begin{align}
\label{appendix eqn: non-transformed Wasserstein GF}
\frac{\partial \rho_t(\theta)}{\partial t} = 
\nabla_{\theta} \cdot \Bigl[\rho_t \Prec(\theta, \rho_t) \nabla_{\theta} \frac{\delta \mathcal{E}}{\delta \rho}\Bigr|_{\rho=\rho_t}
\Bigr].
\end{align}
Consider $\tilde{\theta} = \varphi(\theta)=A\theta + b$ and $\tilde{\rho}_t = \varphi \# \rho_t$ for an invertible affine transformation $\varphi$. Then, it suffices to show that under the assumption $\Prec(\tilde\theta, \tilde\rho) = A \Prec(\theta, \rho) A^T$, one has
\begin{align}
\label{appendix eqn: transformed Wasserstein GF}
\frac{\partial \tilde{\rho}_t(\tilde{\theta})}{\partial t} 
=\nabla_{\tilde{\theta}} \cdot \Bigl[\tilde{\rho}_t \Prec(\tilde{\theta}, \tilde{\rho}_t) \nabla_{\tilde{\theta}} \frac{\delta \tilde{\mathcal{E}}}{\delta \tilde{{\rho}}}\Bigr|_{\tilde{\rho}=\tilde{\rho}_t}
\Bigr].
\end{align}
Here, the transformed energy functional $\tilde{\mathcal{E}}$ is defined via \cref{def:pushforward}.

First, for the left hand side, by definition, we have 
\begin{equation}
\label{appendix eqn: transform of rho}
    \tilde{\rho}_t(\tilde{\theta}) = \rho_t(\varphi^{-1}(\tilde{\theta}))|\nabla_{\tilde{\theta}} \varphi^{-1}(\tilde{\theta})| = \rho_t(\theta)|A^{-1}|.
\end{equation}
For the right hand side, by \cref{lemma:first_variation_change_variable}, it holds that
\begin{equation}
\label{appendix eqn: change of variable grad of first variation}
    \nabla_{\tilde{\theta}}\frac{\delta \tilde{\mathcal{E}}}{\delta \tilde\rho}(\tilde{\theta}) = \nabla_{\tilde{\theta}} \left(\frac{\delta \mathcal{E}}{\delta \rho}(\theta)\right)=A^{-T}\nabla_{\theta} \frac{\delta \mathcal{E}}{\delta \rho}(\theta),
\end{equation}
where in the last equality, we used the change-of-variable formula mentioned in \cref{remark: change of variable grad and div}. Therefore, we can write the right hand side of \cref{appendix eqn: transformed Wasserstein GF} as
\begin{equation}
\label{appendix: change of variables for rhs of Wasserstein gf}
\begin{aligned}
\nabla_{\tilde{\theta}} \cdot \Bigl[\tilde{\rho}_t \Prec(\tilde{\theta}, \tilde{\rho}_t) \nabla_{\tilde{\theta}} \frac{\delta \tilde{\mathcal{E}}}{\delta \tilde{{\rho}}}\Bigr|_{\tilde{\rho}=\tilde{\rho}_t}
\Bigr] &= \nabla_{\tilde{\theta}} \cdot \Bigl[\tilde{\rho}_t \Prec(\tilde{\theta}, \tilde{\rho}_t) A^{-T}\nabla_{\theta} \frac{\delta \mathcal{E}}{\delta {\rho}}\Bigr|_{\rho=\rho_t}
\Bigr]\\
&= \nabla_{\theta} \cdot \Bigl[\tilde{\rho}_t A^{-1}\Prec(\tilde{\theta}, \tilde{\rho}_t) A^{-T}\nabla_{\theta} \frac{\delta \mathcal{E}}{\delta {\rho}}\Bigr|_{\rho=\rho_t}
\Bigr]\\
& = \nabla_{\theta} \cdot \Bigl[\rho_t A^{-1}\Prec(\tilde{\theta}, \tilde{\rho}_t) A^{-T}\nabla_{\theta} \frac{\delta \mathcal{E}}{\delta {\rho}}\Bigr|_{\rho=\rho_t} 
\Bigr] \cdot |A^{-1}|,
\end{aligned}
\end{equation}
where in the second equality, we used the change-of-variable formula for the divergence operator (\cref{remark: change of variable grad and div}), and in the third equality, we used \cref{appendix eqn: transform of rho}.
Based on \cref{appendix eqn: transform of rho}, \cref{appendix: change of variables for rhs of Wasserstein gf} and \cref{appendix eqn: non-transformed Wasserstein GF}, a sufficient condition for \cref{appendix eqn: transformed Wasserstein GF} to hold is $A^{-1}\Prec(\tilde{\theta}, \tilde{\rho}) A^{-T} = P(\theta, \rho)$, or equivalently, $\Prec(\tilde\theta, \tilde\rho) = A \Prec(\theta, \rho) A^T$. This completes the proof.
\end{proof}

\subsection{Proof of \Cref{lem:AI-Wasserstein-MD}}
\label{proof:AI-Wasserstein-MD}


\begin{proof}
Consider the invertible affine transformation $\tilde\theta = \varphi(\theta) = A \theta + b$ and correspondingly $\tilde{\rho} = \varphi \# \rho, \tilde{\rho}_{\rm post} = \varphi \# \rho_{\rm post}$. Using \cref{eq:aMFD-AI-sigma}, we get
\begin{align*}
    D(\tilde{\theta},\tilde\rho) = \frac{1}{2}h(\tilde{\theta},\tilde\rho)h(\tilde{\theta},\tilde\rho)^T = \frac{1}{2}Ah(\theta,\rho)h(\theta,\rho)^TA^T = AD(\theta, \rho)A^T.
\end{align*}
Similarly, it holds that $d(\tilde{\theta},\tilde\rho) = A d(\theta,\rho)$.
Based on these relations, we can calculate as follows:
\begin{equation}
\begin{aligned}    
&Af(\theta, \rho, \rho_{\rm post})\\
&\quad = A\Prec(\theta, \rho)\nabla_{\theta}\log \rho_{\rm post}(\theta) + A(D(\theta,\rho)-P(\theta,\rho))\nabla_{\theta}\log \rho(\theta) - Ad(\theta,\rho)\\
&\quad= 
A\Prec(\theta, \rho)A^T\nabla_{\tilde{\theta}} \log \tilde{\rho}_{\rm post}(\tilde{\theta})  + \bigl(A D(\theta,\rho) - A\Prec(\theta, \rho) \bigr)A^T\nabla_{ \tilde{\theta}} \log \tilde{\rho}(\tilde{\theta}) - Ad(\theta,\rho)  
\\
&\quad= 
\Prec(\tilde{\theta}, \tilde{\rho})\nabla_{\tilde{\theta}} \log \tilde{\rho}_{\rm post}(\tilde{\theta})  + \bigl( D(\tilde{\theta},\tilde\rho) - \Prec(\tilde{\theta}, \tilde{\rho}) \bigr)\nabla_{ \tilde{\theta}} \log \tilde{\rho}(\tilde{\theta}) - d(\tilde\theta,\tilde\rho)  
\\
&\quad=f(\tilde\theta, \tilde\rho, \tilde\rho_{\rm post}).
\end{aligned}
\end{equation}
The first equality is by definition. In the second equality, we used $A^T\nabla_{\tilde{\theta}} \log \tilde{\rho}(\tilde{\theta}) = \nabla_{{\theta}} \log {\rho}({\theta})$, which is due to \cref{appendix eqn change of variable for rho} and \cref{remark: change of variable grad and div}. In the third equality, we used the condition in \cref{proposition:Wasserstein-affine-invariant}. With this result, the mean-field equation is affine invariant (\cref{def: affine invariant mean-field equation}). The proof is complete.
\end{proof}

\subsection{Proof of \Cref{proposition:Stein-affine-invariant}}
\label{proof:Stein-affine-invariant}

\begin{proof}
The proof is similar to that in \cref{proof:Wasserstein-affine-invariant}. For any energy functional $\mathcal{E}$, the gradient flow has the form
\begin{align}
\label{appendix eqn: nontransformed stein GF}
\frac{\partial \rho_t(\theta)}{\partial t} = \nabla_{\theta} \cdot \Bigl[\rho_t(\theta)\int \kappa(\theta,\theta',\rho_t)\rho_t(\theta')\Prec( \theta, \theta',\rho_t)\left(\nabla_{\theta} \frac{\delta \mathcal{E}}{\delta \rho}\Bigr|_{\rho=\rho_t}\right)(\theta')\mathrm{d}\theta' \Bigr]. 
\end{align}
Consider $\tilde{\theta} = \varphi(\theta)=A\theta +b$ and $\tilde{\rho}_t = \varphi \# \rho_t$ for an invertible affine transformation $\varphi$. Then, it suffices to show that under the assumed condition, one has
\begin{align}
\label{appendix eqn: transformed stein GF}
\frac{\partial \tilde{\rho}_t(\tilde{\theta})}{\partial t} 
&= \nabla_{\tilde\theta} \cdot \Bigl[\tilde\rho_t(\tilde\theta)\int \kappa(\tilde{\theta},\tilde{\theta}',\tilde\rho_t)\tilde{\rho}_t(\tilde\theta')\Prec( \tilde{\theta}, \tilde{\theta}',\tilde{\rho}_t) \left(\nabla_{\tilde{\theta}} \frac{\delta \tilde{\mathcal{E}}}{\delta \tilde \rho}\Bigr|_{\tilde{\rho} =\tilde{\rho}_t}\right)(\tilde{\theta}')\mathrm{d}\tilde\theta' \Bigr].
\end{align}
For the right hand side of \cref{appendix eqn: transformed stein GF}, we have
\begin{equation}
\label{appendix eqn: rhs of transformed stein flow}
    \begin{aligned}
        &\nabla_{\tilde\theta} \cdot \Bigl[\tilde\rho_t(\tilde\theta)\int \kappa(\tilde{\theta},\tilde{\theta}',\tilde\rho_t)\tilde{\rho}_t(\tilde\theta')\Prec( \tilde{\theta}, \tilde{\theta}',\tilde{\rho}_t) \left(\nabla_{\tilde{\theta}} \frac{\delta \tilde{\mathcal{E}}}{\delta \tilde \rho}\Bigr|_{\tilde{\rho} =\tilde{\rho}_t}\right)(\tilde{\theta}')\mathrm{d}\tilde\theta' \Bigr]\\
        =&\nabla_{\tilde\theta} \cdot \Bigl[\tilde\rho_t(\tilde\theta)\int \kappa(\tilde{\theta},\tilde{\theta}',\tilde\rho_t)\tilde{\rho}_t(\tilde\theta')\Prec( \tilde{\theta}, \tilde{\theta}',\tilde{\rho}_t) A^{-T}\left(\nabla_{\theta} \frac{\delta \mathcal{E}}{\delta \rho}\Bigr|_{\rho =\rho_t}\right)(\varphi^{-1}(\tilde\theta'))\mathrm{d}\tilde\theta' \Bigr]\\
        =&\nabla_{\tilde\theta} \cdot \Bigl[\tilde\rho_t(\tilde\theta)\int \kappa(\tilde{\theta},\varphi(\theta'),\tilde\rho_t){\rho}_t(\theta')\Prec( \tilde{\theta}, \varphi(\theta'),\tilde{\rho}_t) A^{-T}\left(\nabla_{\theta} \frac{\delta \mathcal{E}}{\delta \rho}\Bigr|_{\rho =\rho_t}\right)(\theta')\mathrm{d}\theta' \Bigr]\\
        =&\nabla_{\theta} \cdot \Bigl[\tilde\rho_t(\tilde\theta)\int \kappa(\tilde{\theta},\varphi(\theta'),\tilde\rho_t){\rho}_t(\theta')A^{-1}\Prec( \tilde{\theta}, \varphi(\theta'),\tilde{\rho}_t) A^{-T}\left(\nabla_{\theta} \frac{\delta \mathcal{E}}{\delta \rho}\Bigr|_{\rho =\rho_t}\right)(\theta')\mathrm{d}\theta' \Bigr]\\
        =&(\nabla_{\theta} \cdot \mathsf{f}) \cdot |A^{-1}|\\
        &\mathsf{f}=\rho_t(\theta)\int \kappa(\tilde{\theta},\varphi(\theta'),\tilde\rho_t){\rho}_t(\theta')A^{-1}\Prec( \tilde{\theta}, \varphi(\theta'),\tilde{\rho}_t) A^{-T} \left(\nabla_{\theta} \frac{\delta \mathcal{E}}{\delta \rho}\Bigr|_{\rho =\rho_t}\right)(\theta')\mathrm{d}\theta',
    \end{aligned}
\end{equation}
where in the first equality, we used \cref{appendix eqn: change of variable grad of first variation}; in the second equality, we changed of coordinates in the integral $\tilde{\theta}' = \varphi(\theta')$; in the third equality, we used the change-of-variable formula for the divergence operator (\cref{remark: change of variable grad and div}), and in the last equality, we used \cref{appendix eqn: transform of rho}.

By \cref{appendix eqn: transform of rho}, \cref{appendix eqn: rhs of transformed stein flow} and \cref{appendix eqn: nontransformed stein GF}, a sufficient condition for \cref{appendix eqn: transformed stein GF} to hold is
\[\kappa(\tilde{\theta},\varphi(\theta'),\tilde\rho)A^{-1}\Prec( \tilde{\theta}, \varphi(\theta'),\tilde{\rho}) A^{-T} = \kappa(\theta,\theta',\rho) \Prec(\theta, \theta', \rho), \]
or equivalently,
$$ \kappa(\tilde\theta, \tilde\theta',\tilde\rho)\Prec(\tilde \theta, \tilde \theta',  \tilde\rho) 
= \kappa(\theta, \theta', \rho) A \Prec(\theta, \theta', \rho) A^T,$$
where $\tilde{\theta}' = \varphi(\theta')$.
This completes the proof.
\end{proof}

\subsection{Proof of \Cref{lem:AI-Stein-MD}}
\label{proof:AI-Stein-MD}


\begin{proof}
Consider the invertible affine transformation $\tilde\theta = \varphi(\theta) = A \theta + b$ and correspondingly $\tilde{\rho} = \varphi \# \rho, \tilde{\rho}_{\rm post} = \varphi \# \rho_{\rm post}$.
By direct calculations, we get
\begin{equation}
\begin{aligned}    
& Af(\theta, \rho, \rho_{\rm post})\\
&\quad =\int \kappa(\theta,\theta',\rho)A\Prec( \theta, \theta',\rho)\nabla_{\theta'} \bigl(\log\rho(\theta') - \log \rho_{\rm post}(\theta') \bigr)\rho(\theta')\mathrm{d}\theta'\\
&\quad=\int \kappa(\theta,\theta',\rho) A\Prec( \theta, \theta', \rho ) A^T
\bigl( \nabla_{\tilde\theta} \log \tilde{\rho}_{\rm post}(\tilde{\theta}') -  \nabla_{\tilde\theta} \log \tilde{\rho}(\tilde{\theta}')
\bigr) \rho(\theta') \mathrm{d}\theta'
\\
&\quad =\int \kappa(\tilde\theta,\tilde\theta', \tilde\rho) \Prec( \tilde\theta, \tilde\theta', \tilde\rho )
\bigl( \nabla_{\tilde\theta} \log \tilde{\rho}_{\rm post}(\tilde{\theta}') -  \nabla_{\tilde\theta} \log \tilde{\rho}(\tilde{\theta}')
\bigr) \tilde\rho(\tilde\theta') \mathrm{d}\tilde\theta'
\\
&\quad =f(\tilde\theta, \tilde\rho, \tilde\rho_{\rm post}).
\end{aligned}
\end{equation}
The first equality is by definition. In the second equality, we used $A^T\nabla_{\tilde{\theta}} \log \tilde{\rho}(\tilde{\theta}) = \nabla_{{\theta}} \log {\rho}({\theta})$, which is due to \cref{appendix eqn change of variable for rho} and \cref{remark: change of variable grad and div}. In the third equality, we used the the relation $\rho(\theta') = \tilde{\rho}(\tilde{\theta}')|A|$ due to \cref{appendix eqn change of variable for rho}, and ${\rm d}\tilde{\theta}' = |A|{\rm d}\theta'$; we also used the condition in \cref{proposition:Stein-affine-invariant}. With this result, the mean-field equation is affine invariant (\cref{def: affine invariant mean-field equation}). The proof is complete.

\end{proof}

\newpage
\section{Proofs for the Convergence of Affine Invariant Gradient Flows}
\subsection{Proof of \Cref{proposition:FR-convergence}}
\label{proof:FR-convergence}
\begin{proof}
The Fisher-Rao gradient flow of the KL divergence~\cref{eq:mean-field-Fisher-Rao} 
can be solved analytically using the variation of constants formula as follows.
First note that
\begin{align*}
&\frac{\partial \log \rho_t(\theta)}{\partial t} 
= \log \rho_{\rm post}(\theta) - \log \rho_t(\theta) - \E_{\rho_t}[\log \rho_{\rm post}(\theta) - \log \rho_t(\theta)],
\end{align*}
so that
\begin{align*}
&\frac{\partial e^t \log \rho_t(\theta)}{\partial t} 
= e^t\log \rho_{\rm post}(\theta) - e^t\E_{\rho_t}[\log \rho_{\rm post}(\theta) - \log \rho_t(\theta)].
\end{align*}
Thus
\begin{equation*}\log \rho_t(\theta)
= (1-e^{-t})\log \rho_{\rm post}(\theta) + e^{-t}\log\rho_0(\theta) -  \int_0^{t} e^{\tau-t}\E_{\rho_\tau}[\log \rho_{\rm post}(\theta) - \log \rho_\tau(\theta)] \mathrm{d}\tau.
\end{equation*}
It follows that there exists some constant $Z_t$ such that
\begin{align}\label{e:rhot}
&\rho_t(\theta) =\frac{1}{Z_t} \rho_0(\theta)^{ e^{-t}}\rho_{\rm post}(\theta)^{1 - e^{-t}},\quad \frac{\rho_t(\theta)}{\rho_{\rm post}(\theta)}=\frac{1}{Z_t}\left(\frac{\rho_0(\theta)}{\rho_{\rm post}(\theta)}\right)^{e^{-t}}.
\end{align}

In the following, we first obtain the following lower bound on $Z_t$: 
\begin{align}\label{e:Ztbound}
    Z_t\geq e^{-Ke^{-t}(1+B)},
\end{align}
where the constants $K,B$ are from \cref{e:asup1} and \cref{e:asup2}. 
In fact, using our assumptions \cref{e:asup1} and \cref{e:asup2}, we have
\begin{align*}
    Z_t&=\int \left(\frac{\rho_0(\theta)}{\rho_{\rm post}(\theta)}\right)^{ e^{-t}}\rho_{\rm post}(\theta) \mathrm{d}\theta
    \geq \int \left(e^{-K(1+|\theta|^2)}\right)^{ e^{-t}}\rho_{\rm post}(\theta) \mathrm{d}\theta\\
    &=\int e^{-Ke^{-t}(1+|\theta|^2)}\rho_{\rm post}(\theta) \mathrm{d}\theta
    \geq  e^{\int-Ke^{-t}(1+|\theta|^2)\rho_{\rm post}(\theta)\mathrm{d}\theta} \geq e^{-Ke^{-t}(1+B)},
\end{align*}
where in the second to last inequality, we used Jensen's inequality and the fact that $e^x$ is convex. 

By plugging \cref{e:asup2} and \cref{e:Ztbound} into \cref{e:rhot}, we get
\begin{align}\label{e:density_ratio}
    \frac{\rho_t(\theta)}{\rho_{\rm post}(\theta)}
    =\frac{1}{Z_t}\left(\frac{\rho_0(\theta)}{\rho_{\rm post}(\theta)}\right)^{e^{-t}}
    \leq e^{Ke^{-t}(1+B)} e^{Ke^{-t}(1+|\theta|^2)}
    =e^{Ke^{-t}(2+B+|\theta|^2)}.
\end{align}
Using \cref{e:density_ratio}, we get the following upper bound on the KL divergence:
\begin{align}\begin{split}\label{e:KLub1}
    & {\rm KL}\Bigl[\rho_{t} \Big\Vert  \rho_{\rm post} \Bigr]\\
    &\quad=\int \rho_t(\theta) \log \frac{\rho_t(\theta)}{\rho_{\rm post}(\theta)}d \theta
    \leq 
    \int \rho_t(\theta) \log(e^{Ke^{-t}(2+B+|\theta|^2)})d \theta\\
    &\quad=\int \rho_t(\theta) Ke^{-t}(2+B+|\theta|^2)d \theta
    =Ke^{-t}\left((2+B)+\int  |\theta|^2 \rho_t(\theta)  d \theta\right).
\end{split}\end{align}
For the last integral in \cref{e:KLub1}, using \cref{e:rhot} and the
H\"older inequality, we can rewrite it as
\begin{align}\begin{split}\label{e:KLub2}
    &\phantom{{}={}}\frac{1}{Z_t}\int  |\theta|^2 \rho_0(\theta)^{e^{-t}}\rho_{\rm post}(\theta)^{1-e^{-t}}  d \theta
    =\frac{1}{Z_t}\int   (|\theta|^2\rho_0(\theta))^{e^{-t}}(|\theta|^2\rho_{\rm post}(\theta))^{1-e^{-t}}  d \theta\\
    &\leq \frac{1}{Z_t}\left(\int   |\theta|^2\rho_0(\theta)\mathrm{d}\theta\right)^{e^{-t}}\left(\int (|\theta|^2\rho_{\rm post}(\theta))  d \theta\right)^{1-e^{-t}}\leq Be^{K e^{-t}(1+B)},
\end{split}\end{align}
where for the last inequality we used \cref{e:Ztbound} and our assumption \cref{e:asup2}.

Combining \cref{e:KLub1} and \cref{e:KLub2} together, for $t\geq \log((1+B)K)$, we have
\begin{align}
    {\rm KL}\Bigl[\rho_{t} \Big\Vert  \rho_{\rm post}\Bigr]\leq Ke^{-t}(2+B + Be^{K e^{-t}(1+B)})
    \leq (2+B + eB)Ke^{-t}.
\end{align}
This completes the proof of \cref{proposition:FR-convergence}.
\end{proof}

\subsection{Proof of \Cref{proposition:W-convergence}}
\label{proof:W-convergence}
\begin{proof}
We calculate the decay of the KL divergence as follows:
\begin{equation}
\begin{aligned}
    \partial_t \mathrm{KL}[\rho_t \Vert  \rho_{\rm post}]  
    &= -\int \rho_t [\nabla_\theta\log(\frac{\rho_t}{\rho_{\rm post}})]^T \Prec [\nabla_\theta\log(\frac{\rho_t}{\rho_{\rm post}})] \mathrm{d}\theta
    \\
    &\leq -\lambda \int \rho_t [\nabla_\theta\log(\frac{\rho_t}{\rho_{\rm post}})]^T [\nabla_\theta\log(\frac{\rho_t}{\rho_{\rm post}})] \mathrm{d}\theta
    \\
    &\leq -2\lambda \alpha \mathrm{KL}[\rho_t \Vert  \rho_{\rm post}].
\end{aligned}
\end{equation}
In the last inequality, we used the the logarithmic Sobolev inequality for $\rho_{\rm post}$, which is ensured by the $\alpha$-strongly logconcave assumption.
Then we have 
\begin{equation}
\begin{aligned}
\frac{1}{2}\lVert\rho_t - \rho_{\rm post} \rVert^2_{L_1} \leq \mathrm{KL}[\rho_t \Vert  \rho_{\rm post}]  \leq  \mathrm{KL}[\rho_0 \Vert  \rho_{\rm post}]  e^{-2\alpha\lambda t},
\end{aligned}
\end{equation}
where we used the Pinsker inequality to bound the $L_1$ norm by the KL divergence.
\end{proof}

\newpage
\section{Proofs for Gaussian Approximate Gradient Flows}
\subsection{Preliminaries}
We start with the following Stein's identities concerning
the Gaussian density function $\rho_a.$

\begin{lemma}
\label{lemma: stein equality}
Assume $\theta \sim \N(m, C)$ with density $\rho_a(\theta)=\rho_a(\theta;m,C)$, 
we have
\begin{align}\label{e:derrho}
    \nabla_m \rho_a(\theta)=-\nabla_\theta \rho_a(\theta)\quad \text{and}\quad \nabla_C \rho_a(\theta)=\frac{1}{2}\Hess \rho_a(\theta).
\end{align}
Furthermore, for any scalar field $f:\mathbb{R}^{N_{\theta}} \to \mathbb{R}$ 
and vector field $g:\mathbb{R}^{N_{\theta}} \to \mathbb{R}^{N_{\theta}}$, we have
\begin{equation}
\begin{aligned} 
&\E_{\rho_a}[\nabla_{\theta} g(\theta)] = \nabla_{m} \E_{\rho_a}[g(\theta)] = \Cov[g(\theta), \theta] C^{-1}, \\
&\E_{\rho_a}[\Hess f(\theta)] = \Cov[\nabla_{\theta}f(\theta), \theta] C^{-1}= -C^{-1}\E_{\rho_a}\Bigl[\bigl(C - (\theta - m)(\theta - m)^T\bigr)f\Bigr] C^{-1}.
\end{aligned}
\end{equation}
\end{lemma}

\begin{proof}
For Gaussian density function $\rho_a$, we have
\begin{align*}
    \nabla_m \rho_a(\theta) 
    &= \nabla_m\frac{1}{\sqrt{|2\pi C|}} \exp\Bigl\{-\frac{1}{2}(\theta - m)C^{-1}(\theta - m)\Bigr\}  \\
    &=  C^{-1}(\theta - m)\rho_a(\theta)
  =-\nabla_\theta \rho_a(\theta),
  \\
  \nabla_C \rho_a(\theta) 
  &=\rho_a(\theta) \Bigl(-\frac{1}{2}\frac{\partial \log|C|}{\partial C} - \frac{1}{2}\frac{\partial (\theta - m)^TC^{-1}(\theta - m)}{\partial C} \Bigr) \\
  &=-\frac{1}{2}\rho_a(\theta) \Bigl(C^{-1} - C^{-1} (\theta - m)(\theta - m)^T  C^{-1}\Bigr)
  =\frac{1}{2}\Hess \rho_a(\theta). 
\end{align*}
For any scalar field $f(\theta)$ and vector field $g(\theta)$, we have
\begin{equation}
\begin{aligned} 
\E_{\rho_a}[\nabla_{\theta} g(\theta)] 
&= \int \nabla_{\theta} g(\theta) \rho_a(\theta) \mathrm{d}\theta
= -\int  g(\theta) \nabla_{\theta}\rho_a(\theta)^T \mathrm{d}\theta
= \nabla_{m} \E_{\rho_a}[g(\theta)] \\
&= \int g(\theta) (\theta - m)^TC^{-1}\rho_a(\theta)  \mathrm{d}\theta =    \Cov[g(\theta), \theta] C^{-1}, 
\\
\E_{\rho_a}[\Hess f(\theta)] 
&= \int \Hess f(\theta) \rho_a(\theta) \mathrm{d}\theta
= 
-\int \nabla_{\theta} f(\theta) \nabla_{\theta} \rho_a(\theta)^T \mathrm{d}\theta\\
&=
\int \nabla_{\theta} f(\theta) (\theta - m)^TC^{-1} \rho_a(\theta) \mathrm{d}\theta
= 
\Cov[\nabla_{\theta}f(\theta), \theta] C^{-1} \\
&= 
\int  f(\theta) \Hess \rho_a(\theta) \mathrm{d}\theta
=
-C^{-1}\E_{\rho_a}\Bigl[\bigl(C - (\theta - m)(\theta - m)^T\bigr)f\Bigr] C^{-1}.
\end{aligned}
\end{equation}
\end{proof}

The following lemma is proved in~\cite[Theorem 1]{lambert2022recursive}:
\begin{lemma}\label{l:gd}
Consider the KL divergence
\begin{align*}
{\rm KL}\Bigl[\rho_a(\theta)\Big\Vert  \rho_{\rm post}(\theta )\Bigr] 
= -\frac{1}{2}\log\bigl|C\bigr| - \int \rho_a(\theta) \log \rho_{\rm post} (\theta) \mathrm{d}\theta + {\rm const}. 
\end{align*}
For fixed $\rho_{\rm post}$ the minimizer of this divergence over the space $\PPG$, so that, for $\rhoa_\star=(m_{\star}, C_{\star})$ and $\rho_{\rhoa_\star}(\theta) = \N(m_{\star}, C_{\star})$, it follows that
\begin{align*}
\E_{\rho_{\rhoa_\star}}\bigl[\nabla_\theta  \log \rho_{\rm post}(\theta)   \bigr] = 0 \quad \textrm{and}
\quad C_{\star}^{-1} = -\E_{\rho_{\rhoa_\star}}\bigl[\Hess\log \rho_{\rm post}(\theta)\bigr].
\end{align*}
\end{lemma}

\subsection{Consistency of Riemannian Perspective in~\Cref{ssec:Riemannian}}
\label{proof:Riemannian}
Here we prove \cref{prop:Riemannian-G}, \cref{prop:1}, and \cref{prop:2},
which identifies specific gradient flows, within the paper, that satisfy the
assumptions required for application of \cref{prop:1}.

\begin{proof}[Proof of \cref{prop:Riemannian-G}]
Recall from \eqref{e:Tidentify} that any element in $ T_{\rho_{\rhoa_t}}\PPG$ is given in the form $\nabla_\rhoa \rho_{\rhoa_t}\cdot\sigma$. Thus we have, for $\sigma_t \in T_{\rho_t}\PP$ in \cref{eq:gftp}, it holds that
\begin{equation}
\label{eq:Riemannian-G-C.3}
    \begin{aligned}
g_{\rho_{\rhoa_t}}(\sigma_t, \nabla_\rhoa \rho_{\rhoa_t}\cdot \sigma) 
&= g_{\rho_{\rhoa_t}}\Bigl(-\M(\rho_{\rhoa_t})^{-1}\frac{\delta \mathcal{E}}{\delta \rho}\Bigr|_{\rho=\rho_{\rhoa_t}}, \nabla_\rhoa \rho_{\rhoa_t}\cdot \sigma\Bigr)\\
&=-\Bigl \langle\frac{\delta \mathcal{E}}{\delta \rho}\Bigr|_{\rho=\rho_{\rhoa_t}}, \nabla_\rhoa \rho_{\rhoa_t}\cdot \sigma \Bigr \rangle
\\
&= -\Bigl \langle\left.\frac{\partial \mathcal{E}(\rho_\rhoa)}{\partial\rhoa}\right|_{\rhoa=\rhoa_t}, \sigma \Bigr \ranglen
=\fg_{{\rhoa_t}}\Bigl(-\fM(\rhoa_t)^{-1}\left.\frac{\partial \mathcal{E}(\rho_\rhoa)}{\partial\rhoa}\right|_{\rhoa=\rhoa_t}, \sigma\Bigr)\\
&=g_{\rho_{\rhoa_t}}\Bigl(-\nabla_\rhoa \rho_{\rhoa_t}\cdot \fM(\rhoa_t)^{-1}\left.\frac{\partial \mathcal{E}(\rho_\rhoa)}{\partial\rhoa}\right|_{\rhoa=\rhoa_t}, \nabla_\rhoa \rho_{\rhoa_t}\cdot \sigma \Bigr).
    \end{aligned}
\end{equation}
Combining \cref{eq:Riemannian-G-C.3} and the definition \cref{eq:def-Riemannian} leads to
\begin{align}\label{e:projection}
    P^G\sigma_t = -\nabla_\rhoa \rho_{\rhoa_t}\fM(\rhoa_t)^{-1}\left.\frac{\partial \mathcal{E}(\rho_\rhoa)}{\partial\rhoa}\right|_{\rhoa=\rhoa_t}.
\end{align}
By plugging \eqref{e:projection} into \eqref{eq:G-GF-Riemannian}, we get
\begin{align}
\label{eq:Riemannian-G-C.4}
\nabla_{\rhoa}\rho_{\rhoa_t} \cdot\frac{\partial\rhoa_t}{\partial t}
=\frac{\partial \rho_{\rhoa_t}}{\partial t}=-\nabla_\rhoa \rho_{\rhoa_t}\cdot\fM(\rhoa_t)^{-1}\left.\frac{\partial \mathcal{E}(\rho_\rhoa)}{\partial\rhoa}\right|_{\rhoa=\rhoa_t}.
\end{align}
Since we may choose $\nabla_\rhoa \rho_{\rhoa_t}$ so that it has non-zero values in each one of its entries, we can 
 remove $\nabla_\rhoa \rho_{\rhoa_t}$ on both sides of \cref{eq:Riemannian-G-C.4}, and get \cref{eq:G-GF}.
\end{proof}

\begin{proof}[Proof of \cref{prop:1}]
The mean and covariance evolution equations of \cref{eq:G-GF-Riemannian} are 
\begin{equation}
\begin{aligned}
\label{eq:mC-Riemannian}
    &\frac{\mathrm{d} m_t}{\mathrm{d}t} = \int -P^G \sigma_t(\theta, \rho_{\rhoa_t}) \theta \mathrm{d}\theta,\\
    &\frac{\mathrm{d} C_t}{\mathrm{d}t} = \int -P^G \sigma_t(\theta, \rho_{\rhoa_t}) (\theta - m_t)(\theta - m_t)^T \mathrm{d}\theta,
\end{aligned}
\end{equation}
where $\rho_{\rhoa_t} = \N(m_t, C_t)$. For any $f(\theta) \in T^*_{\rho_{\rhoa_t}}\PPG$, we have 
\begin{align*}
\bigl \langle P^G\sigma_t(\theta, \rho_{\rhoa_t}), f(\theta)\bigr \rangle 
&= g_{\rho_{\rhoa_t}}(P^G\sigma_t(\theta, \rho_{\rhoa_t}), \M(\rho_{\rhoa_t})^{-1}f(\theta)) \\
&= g_{\rho_{\rhoa_t}}(\sigma_t(\theta, \rho_{\rhoa_t}), \M(\rho_{\rhoa_t})^{-1}f(\theta))
= \bigl \langle \sigma_t(\theta, \rho_{\rhoa_t}), f(\theta)\bigr \rangle.
\end{align*}
Using assumption \cref{eq:mcp-cond} and taking $f(\theta)$ to be linear and quadratic functions of $\theta$, \cref{eq:mC-Riemannian} become 
\begin{equation}
\begin{aligned}
    &\frac{\mathrm{d} m_t}{\mathrm{d}t} = \int -\sigma_t(\theta, \rho_{\rhoa_t}) \theta \mathrm{d}\theta, \\
    &\frac{\mathrm{d} C_t}{\mathrm{d}t} = \int -\sigma_t(\theta, \rho_{\rhoa_t}) (\theta - m_t)(\theta - m_t)^T \mathrm{d}\theta, 
\end{aligned}
\end{equation}
delivering~\cref{eq:mC-Momentum}.
This indicates that the mean and covariance evolution equations of the 
gradient flow~\cref{eq:G-GF-Riemannian} obtained from Riemannian perspective are 
the same as the closed system obtained by the moment closure approach. 
Combining \cref{prop:Riemannian-G} completes the proof.
\end{proof}

\begin{proof}[Proof of \cref{prop:2}]

Using the calculation in the proof of \cref{lemma: stein equality}, the tangent space of the Gaussian density manifold at $\rho_\rhoa$ with $\rhoa = [m ,C]$ is
\begin{equation}
    \begin{aligned}
    T_{\rho_{\rhoa}} \PPG 
    &= {\rm span}\bigl\{\rho_{\rhoa}[C^{-1}(\theta-m)]_i,\, \rho_{\rhoa}[C^{-1}(\theta-m)(\theta-m)^TC^{-1} - C^{-1}]_{ij}\bigr\}  \\
    &= {\rm span}\bigl\{\rho_{\rhoa}(\theta_i - \E_{\rho_{\rhoa}}[\theta_i]),\, \rho_{\rhoa}(\theta_i \theta_j - \E_{\rho_{\rhoa}}[\theta_i\theta_j]) \bigr\},
    \end{aligned}
\end{equation}
and $1\leq i,j \leq N_{\theta}$.

For the Fisher-Rao metric, we have 
\begin{equation}
\begin{aligned}
\M^{\rm FR}(\rho_{\rhoa})^{-1}{\rm span}\bigl\{\theta_i, \theta_i\theta_j\bigr\} 
= {\rm span}\bigl\{\rho_{\rhoa}(\theta_i - \E_{\rho_{\rhoa}}[\theta_i]), 
\rho_{\rhoa}(\theta_i\theta_j - \E_{\rho_{\rhoa}}[\theta_i\theta_j])\bigr\} 
= T_{\rho_{\rhoa}} \PPG.
\end{aligned}
\end{equation}

For the affine invariant Wasserstein metric with preconditioner $\Prec$ independent of $\theta$, we have 
\begin{equation}
\begin{aligned}
    &\M^{\rm AIW}(\rho_{\rhoa})^{-1}{\rm span}\bigl\{\theta_i, \theta_i\theta_j\bigr\} \\
    &\quad =  {\rm span}\bigl\{\nabla_{\theta}\cdot(\rho_\rhoa(\theta) \Prec(\rho_\rhoa) e_i),\,  
                   \nabla_{\theta}\cdot(\rho_\rhoa(\theta) \Prec(\rho_\rhoa) e_i \theta_j) \bigr\}\\
    &\quad =  
    {\rm span}\bigl\{ \nabla_{\theta}\cdot(\rho_\rhoa(\theta) (b' + A'\theta)) \quad \forall b'\in\R^{N_\theta}\, A'\in\R^{N_\theta\times N_\theta} \bigr\}    \\
    &\quad =  {\rm span}\bigl\{ \rho_\rhoa(\theta)[ {\rm tr}(A') - (\theta - m)^T C^{-1} (b' + A'\theta)] \quad \forall b'\in\R^{N_\theta}\, A'\in\R^{N_\theta\times N_\theta} \bigr\} \\
    &\quad = T_{\rho_{\rhoa}} \PPG.
\end{aligned}
\end{equation}
Here $e_i$ is the $i$-th unit vector.

For the affine invariant Stein metric with preconditioner $\Prec$ independent of $\theta$ and with a bilinear kernel $\kappa(\theta, \theta',\rho) = (\theta - m)^TA(\rho)(\theta' - m) + b(\rho)$ ($b\neq0$, $A$ nonsingular), we have 
\begin{equation}
\begin{aligned}
    &\M^{\rm AIS}(\rho_{\rhoa})^{-1}{\rm span}\bigl\{\theta_i, \theta_i\theta_j\bigr\} \\
    &\quad =  {\rm span}\Bigl\{\nabla_{\theta}\cdot\Bigl(\rho_{\rhoa}(\theta)\Prec(\rho_{\rhoa})\int \kappa\rho_{\rhoa}(\theta')e_i\mathrm{d}\theta' \Bigr) ,  
                  \nabla_{\theta}\cdot\Bigl(\rho_{\rhoa}(\theta)\Prec(\rho_{\rhoa})\int \kappa\rho_{\rhoa}(\theta')e_i\theta_j' \mathrm{d}\theta' \Bigr) \Bigr\}\\
    &\quad =  {\rm span}\Bigl\{\nabla_{\theta}\cdot\Bigl(\rho_{\rhoa}(\theta)\Prec(\rho_{\rhoa})b e_i \Bigr) ,  
                  \nabla_{\theta}\cdot\Bigl(\rho_{\rhoa}(\theta)\Prec(\rho_{\rhoa})[(\theta-m)^TA e_j c_{jj} e_i + b e_i m_j] \Bigr) \Bigr\}         \\
    &\quad =  
    {\rm span}\Bigl\{ \nabla_{\theta}\cdot(\rho_\rhoa(\theta) (b' + A'\theta)) \quad \forall b'\in\R^{N_\theta}\, A'\in\R^{N_\theta\times N_\theta} \Bigr\}  \\
    &\quad = T_{\rho_{\rhoa}} \PPG.
\end{aligned}
\end{equation}

\end{proof}

\subsection{Proof of \Cref{prop: affine invariant metric tensor G}}
\label{proof: prop affine invariant metric tensor G}
\begin{proof}
Following \Cref{ssec:GA}, we consider the invertible affine transformation $\varphi(\theta) = A\theta + b$ in $\R^{N_{\theta}}$.
For Gaussian density space where $\rhoa = (m,C)$, 
we have a corresponding invertible affine transformation in $\R^{N_\rhoa}$, such that
$\tilde{\rhoa} = A^G \rhoa + b^G$, where $A^G$ and $b^G$ depend only on $A$ and $b$ and are defined by the identities  $\tilde{m} = Am +b$ and $\tilde{C} = ACA^T.$ 
Furthermore, the corresponding transformation tangent vector $\sigma \in \R^{N_\rhoa}$ in $T_{\rho_\rhoa} \PPG$ is $\tilde{\sigma} = A^G\sigma \in \R^{N_\rhoa}$.

    From the definition \cref{eq:def-M_alpha} and \cref{prop: affine invariant metric tensor}, we have
    \begin{align*}
        \langle\fM(\rhoa)\sigma_1,\sigma_2\ranglen &= g_{\rho_\rhoa}(\nabla_\rhoa  \rho_\rhoa \cdot\sigma_1,\nabla_\rhoa \rho_\rhoa \cdot \sigma_2)
        = g_{\varphi \# \rho_\rhoa}(\varphi \#\nabla_\rhoa  \rho_\rhoa \cdot\sigma_1, \varphi \# \nabla_\rhoa \rho_\rhoa \cdot \sigma_2)
        \\&= g_{\rho_{\tilde\rhoa}}(\nabla_{\tilde\rhoa}  \rho_{\tilde\rhoa} \cdot\tilde{\sigma}_1, \nabla_{\tilde\rhoa} \rho_{\tilde{\rhoa}} \cdot \tilde{\sigma}_2)
        = \langle\fM(\tilde\rhoa)\tilde{\sigma}_1,\tilde{\sigma}_2\ranglen,
    \end{align*}
    which leads to $\fM(\rhoa) = (A^G)^T \fM(\tilde\rhoa)A^G$. Following the definition of $\tilde\rhoa_t$, we have
    \begin{align*}
     &\frac{\partial \tilde{\rhoa}_t}{\partial t} =  -A^{G}\fM(\rhoa_t)^{-1}\left.\frac{\partial  \mathcal{E}(\rho_{\rhoa})}{\partial \rhoa } \right|_{\rhoa=\rhoa_t}\\
     &\quad=  -A^{G}\fM(\rhoa_t)^{-1}(A^G)^T \left.\frac{\partial  \tilde{\mathcal{E}}(\rho_{\tilde\rhoa})}{\partial \tilde\rhoa } \right|_{\tilde\rhoa=\tilde\rhoa_t}\\
     &\quad=  -\fM(\tilde{\rhoa}_t)^{-1} \left.\frac{\partial  \tilde{\mathcal{E}}(\rho_{\tilde\rhoa})}{\partial \tilde\rhoa } \right|_{\tilde\rhoa=\tilde\rhoa_t}.
    \end{align*}

\end{proof}

\subsection{Proof of \Cref{lemma:Gaussian-interacting-particle}}
\label{proof:Gaussian-interacting-particle}
\begin{proof}
Define $e_t := \theta_t - m_t$, $\mA_t=\mA(\rho_t,\rho_{\rm post})$, and  $\mb_t=\mb(\rho_t,\rho_{\rm post})$. Taking expectation on both sides of \eqref{eq:MFG}, 
we obtain that the evolution of $m_t$ and $e_t$ are the ODEs:
$$\frac{\mathrm{d}m_t}{\mathrm{d}t} = \mb_t \quad \mathrm{and} \quad  
 \frac{\mathrm{d}e_t}{\mathrm{d}t} = \mA_t e_t.$$ 
To solve the ODE for $e_t$, let $B_t: \R \rightarrow \R^{N_{\theta}\times N_{\theta}}$ denote the solution of the ODE system
$$\frac{\mathrm{d}B_t}{\mathrm{d}t} = \mA_t B_t \quad \textrm{and}\quad B_0 = \I. $$
With this, one has the explicit formula of the solution $e_t = B_t e_0$. It implies that $e_t \sim \N(0, C_t)$ where $C_t = B_tC_0B_t^T$. Consequently, the law of $\theta_t$ is Gaussian and $C_t$ satisfies the  ODE
\begin{align*}
    \frac{\mathrm{d}C_t}{\mathrm{d}t}&= \frac{\mathrm{d}B_t}{\mathrm{d}t}C_0 B_t^T +  B_tC_0\frac{\mathrm{d}B_t}{\mathrm{d}t}^T = \mA_t C_t + C_t \mA_t^T.
\end{align*}
\end{proof}

\subsection{Proof of \cref{prop: affine invariant mean-field G}}
\label{proof: affine invariant mean-field G}
\begin{proof}
   Let $\varphi:\theta \rightarrow \tilde\theta$ denote an invertible affine transformation in $\R^{N_{\theta}}$, where $\varphi(\theta) = A\theta + b$ with $A \in \R^{N_{\theta}}\times \R^{N_{\theta}}$, $b\in \R^{N_{\theta}}$, and $A$ invertible. The affine invariance property of the mean and covariance evolution equation \cref{eq:MFG-2} leads to the following relation
   \begin{equation}
   \label{eq:MFG-3}
   \begin{aligned}
       &\mb(\varphi\#\rho_a, \varphi\#\rho_{\rm post}) = A\mb(\rho_a, \rho_{\rm post}),\\
        &\mA(\varphi\#\rho_a, \varphi\#\rho_{\rm post}) = A\mA(\rho_a, \rho_{\rm post}) A^{-1}.
   \end{aligned}
   \end{equation}
   Let $f(\theta_t; \rho_{\rhoa_t}, \rho_{\rm post}) = \mA(\rho_{\rhoa_t}, \rho_{\rm post}) (\theta_t - m_t) + \mb(\rho_{\rhoa_t}, \rho_{\rm post})$.
   Using \Cref{eq:MFG-3}, we get
   \begin{align*}
     Af(\theta_t; \rho_{\rhoa_t}, \rho_{\rm post}) 
     &= A\mA(\rho_{\rhoa_t}, \rho_{\rm post}) (\theta_t - m_t) + A\mb(\rho_{\rhoa_t}, \rho_{\rm post})
     \\
     &= A\mA(\rho_{\rhoa_t},\rho_{\rm post})A^{-1} A(\theta_t - m_t) + A\mb(\rho_{\rhoa_t},\rho_{\rm post})
     \\
     &= f(\varphi(\theta_t); \varphi\#\rho_t, \varphi\#\rho_{\rm post}).
   \end{align*}
   Therefore, the mean-field equation is affine invariant. The proof is complete.
\end{proof}

\newpage
\section{Proofs for the Convergence of Gaussian Approximate Gradient Flows}

\subsection{Proof of \Cref{proposition:analytical-solution-ngd}}
\label{proof:analytical-solution-ngd}
\begin{proof}
Under the Gaussian posterior assumption \cref{e:Gpost} and the fact that $$\frac{\mathrm{d}C_t^{-1}}{\mathrm{d}t} = -C_t^{-1}\frac{\mathrm{d}C_t}{\mathrm{d}t}C_t^{-1},$$ the Gaussian approximate Fisher-Rao gradient flow \cref{eq:Gaussian Fisher-Rao} becomes 
\begin{equation}
\begin{aligned}    
&\frac{\mathrm{d} m_t}{\mathrm{d}t}= C_t C_{\star}^{-1}\bigl(m_{\star} - m_t\bigr),  \\
&\frac{\mathrm{d} C_t^{-1}}{\mathrm{d}t}
=  -C_t^{-1}  + C_{\star}^{-1}.
\end{aligned}
\end{equation}
The covariance update equation has an analytical solution
\begin{equation}
\begin{aligned}    
&C_t^{-1} = (1 - e^{-t})C_{\star}^{-1} + e^{- t}C_0^{-1} .
\end{aligned}
\end{equation}
We reformulate \cref{e:mtCt} as
$$
m_{\star} - m_t =  e^{-t} C_t C_0^{-1} \bigl( m_{\star} - m_0 \bigr)  .
$$
Computing its time derivative leads to
\begin{equation}
\begin{aligned}    
\frac{\mathrm{d} m_t }{\mathrm{d}t} &= -e^{-t}C_tC_0^{-1}(m_0 - m_{\star}) +  e^{-t}\frac{\mathrm{d}C_t}{\mathrm{d}t}C_0^{-1}(m_0 - m_{\star})\\
&= -e^{-t}C_tC_0^{-1}(m_0 - m_{\star}) +  e^{-t}(C_t - C_t C_{\star}^{-1}C_t)C_0^{-1}(m_0 - m_{\star})\\
&= C_t C_{\star}^{-1} e^{-t}C_tC_0^{-1}(m_{\star}-m_0 )\\
&= C_t C_{\star}^{-1}\bigl(m_{\star} - m_t\bigr). 
\end{aligned}
\end{equation}
\end{proof}

\subsection{Proof of Convergence for Gaussian Posterior (\Cref{p:rate})}
\label{proof:p:rate}
\begin{proof}
Under the assumption \cref{e:Gpost}, for $\theta_t\sim \N(m_t,C_t)$ with density $\rho_{\rhoa_t}$, we have
\begin{align}
\E_{\rho_{\rhoa_t}} \bigl[ \nabla_{\theta}  \log \rho_{\rm post}(\theta_t)  \bigr]= -C_{\star}^{-1}(m_t-m_{\star}),\quad \E_{\rho_{\rhoa_t}}\bigl[\Hess \log \rho_{\rm post}(\theta_t)\bigr]=-C_{\star}^{-1}.
\end{align}
Here $m_{\star}$ and $C_{\star}$ are the posterior mean and covariance given in
\cref{eq:Gaussian-KL0}.
We have explicit expressions for the Gaussian approximate gradient flows \cref{eq:g-gf,eq:Gaussian Fisher-Rao,eq:Gaussian-Wasserstein}:
\begin{align}
\begin{split}\label{e:gd}
    \text{Gaussian} &\text{ approximate gradient flow:}
    \\
    &\partial_t{m}_t = -C_{\star}^{-1}(m_t-m_{\star}),
\quad 
\partial_t{C}_t = \frac{1}{2}C_t^{-1} - \frac{1}{2}C_{\star}^{-1}.\\
\text{Gaussian} &\text{ approximate Fisher-Rao gradient flow:}\\
    &\partial_t{m}_t = -C_t C_{\star}^{-1}(m_t-m_{\star}),
\quad 
\partial_t{C}_t = C_t-C_t C_{\star}^{-1}C_t.\\
\text{Gaussian} &\text{  approximate Wasserstein gradient flow:}
\\
&\partial_t{m}_t = -C_{\star}^{-1}(m_t-m_{\star}),
\quad
\partial_t{C}_t = 2\I -C_tC^{-1}_{\star} - C^{-1}_{\star} C_t.
\end{split}
\end{align}

For the dynamics of $m_t$ for both Gaussian approximate gradient flow and Gaussian approximate Wasserstein gradient flow, we have
\begin{align}
    m_t-m_{\star}=e^{-t C_{\star}^{-1}}(m_0-m_{\star}).
\end{align}
By taking the $2-$norm on both sides, using $\|\cdot\|_2$ to denote
both the vector and induced matrix norms, and recalling that
the largest eigenvalue of $C_{\star}$ is $\la$, we obtain
\begin{align}
    \|m_t-m_{\star}\|_2\leq \|e^{-t C_{\star}^{-1}}\|_2\|m_0-m_{\star})\|_2\leq e^{-t/\la}\|m_0-m_{\star}\|_2=\bigO(e^{-t/\la}).
\end{align}
The bound can be achieved, when $m_0 - m_{\star}$ has nonzero component in the $C_{\star}$ eigenvector direction corresponding to $\lambda_{\star,\max}$.

For the Gaussian approximate Fisher-Rao gradient flow, thanks to the explicit formula \cref{e:mtCt}, we find that
\begin{align}
   \|m_t-m_{\star} \|_2\leq e^{-t} \max \{ \|C_{\star}\|_2, \|C_0\|_2\} \|C_0^{-1}\|_2\|m_0-m_{\star}\|_2=\bigO(e^{-t}).
\end{align}
The bound can be achieved when $m_0 - m_{\star}$ is nonzero.

Next, we analyze the dynamics of the covariance matrix $C_t$. Our initialization $C_0=\lambda_0\I$, commutes with $C_{\star}, C_{\star}^{-1}$. It follows that  $C_t$ commutes with $C_{\star}, C_{\star}^{-1}$ for any $t\geq 0$ and all gradient flows in \cref{e:gd}, since $0$ is the unique solution of the evolution ordinary differential equation of $C_tC_{{\star}} - C_{{\star}}C_t$. So we can diagonalize $C_t, C_t^{-1}, C_{\star}, C_{\star}^{-1}$ simultaneously, and write down the dynamics of the eigenvalues of $C_t$. For any eigenvalue $\lambda_t$ of $C_t$, it satisfies the following differential equations,
\begin{align}\begin{split}\label{e:gd2}
    \text{Gaussian approximate gradient flow:}
    \quad 
&\partial_t \lambda_t=\frac{1}{2 \lambda_t}-\frac{1}{2 \lambda_{\star}},\\
\text{Gaussian  approximate Fisher-Rao gradient flow:}
\quad 
    &
\partial_t{\lambda}_t = \lambda_t-\lambda^2_t \lambda_{\star}^{-1},\\
\text{Gaussian  approximate Wasserstein gradient flow:}
\quad 
&\partial_t{\lambda}_t = 2 -\frac{2\lambda_t}{\lambda_{\star}},
\end{split}\end{align}
where $\lambda_{\star}$ is the corresponding eigenvalue of $C_{\star}$. 
From \cref{e:gd2}, we know that $\lambda_t$ is bounded between $\lambda_0$ and $\lambda_{\star}$. Moreover,  the ordinary differential equations in \cref{e:gd2} can be solved explicitly. 

For the Gaussian approximate gradient flow
\begin{align}
    \lambda_t-\lambda_{\star}=(\lambda_0-\lambda_{\star})e^{-\frac{t}{2\lambda_{\star}^2}-\frac{\lambda_t-\lambda_0}{\lambda_{\star}}},\quad |\lambda_t-\lambda_{\star}|=\bigO(e^{-t/2\lambda_{\star}^2}).
\end{align}
Since the largest eigenvalue of $C_{\star}$ is $\la$, we conclude that 
$$\|C_t-C_{\star}\|_2=\bigO(e^{-t/2\la^2}).$$

For the Gaussian approximate Fisher-Rao gradient flow, 
\begin{align}
\lambda_t=\frac{\lambda_{\star}}{1+\left(\frac{\lambda_{\star}}{\lambda_0}-1\right)e^{-t}},\quad |\lambda_t-\lambda_{\star}|=\bigO(e^{-t}). 
\end{align}
It follows that 
$$\|C_t-C_{\star}\|_2=\bigO(e^{-t}).$$

For the Gaussian approximate Wasserstein gradient flow
\begin{align}
      \lambda_t=\lambda_{\star} + e^{-2t/\lambda_{\star}}(\lambda_0 - \lambda_{\star}),\quad |\lambda_t-\lambda_0|=\bigO(e^{-2t/\lambda_{\star}}).
\end{align}
Since the large eigenvalue of $C_{\star}$ is $\la$, we conclude that 
$$\|C_t-C_{\star}\|_2=\bigO(e^{-2t/\la}).$$
\end{proof}

\subsection{Proof of Convergence for Logconcave Posterior}
\label{proof:a:rate}

Before proving \cref{p:logconcave}, 
we first show that the covariance matrices $C_t$ are well conditioned for
all positive times $t$, 
under the gradient flows 
\cref{eq:g-gf,eq:Gaussian Fisher-Rao,eq:Gaussian-Wasserstein}.

\begin{lemma}
\label{p:stability}
Under assumption \cref{a:logconcave}, if the initial covariance matrix satisfies $\lambda_{0,\min}\I\preceq C_0\preceq \lambda_{0,\max}\I$, then for \cref{eq:g-gf,eq:Gaussian Fisher-Rao,eq:Gaussian-Wasserstein}, we have for all $t \ge 0$,
$$\min\{\lambda_{0,\min}, 1/\beta\}\I\preceq C_t \preceq \max\{\lambda_{0,\max},1/\alpha\}\I.$$ 
\end{lemma}

\begin{proof}
Since the parametrized family $C_t$ of Hermitian matrices is continuously differentiable in $t$, we can parameterize its eigenvalues as $\lambda_1(t),\lambda_2(t), \cdots,\lambda_{N_\theta}(t)$ such that they are 
continuously differentiable in $t$~\cite[Chapter I, Section 5]{rellich1969perturbation}. We remark these eigenvalues may not preserve any
ordering. The corresponding normalized eigenvectors are $u_1(t), u_2(t), \cdots, u_{N_\theta}(t)$.
At time $t=0$, since $\lambda_{0,\min}\I \preceq C_0$, we have $\lambda_1(0), \lambda_2(0), \cdots \lambda_{N_\theta}(0)\geq \lambda_{0,\min}$. 
We have the following variational formula for the eigenvalues 
\begin{align}\label{e:eigvariation}
    \partial_t \lambda_i=u_i^T \partial_t C_t u_i.
\end{align}

For the Gaussian approximate gradient flow \cref{eq:g-gf}, using assumption \cref{a:logconcave}, we have 
\begin{align}\label{e:Ctbb}
    \frac{1}{2}C_t^{-1}-\frac{1}{2}\beta \I\preceq \partial_t{C}_t =\frac{1}{2}C_t^{-1}+\frac{1}{2}\E_{\rho_{\rhoa_t}}[\Hess \log \rho_{\rm post}(\theta)]\preceq \frac{1}{2}C_t^{-1}-\frac{1}{2}\alpha \I.
\end{align}
By plugging \cref{e:Ctbb} into \cref{e:eigvariation} and noticing $u_i(t)^T C_t^{-1}u_i(t)=1/\lambda_i(t)$, we get 
\begin{align}
\frac{1}{2\lambda_i(t)}-\frac{\beta}{2}\leq    \partial_t {\lambda}_i(t) \leq\frac{1}{2\lambda_i(t)}-\frac{\alpha}{2},\quad \lambda_{0,\min}\leq \lambda_i(0)\leq \lambda_{0,\max}.
\end{align}
It follows that if $\lambda_i(t)\geq 1/\alpha$, then $\partial_t{\lambda}_i(t)\leq0$; if $\lambda_i(t)\leq 1/\beta$, then $\partial_t{\lambda}_i(t)\geq0$. Recall from our choice of $C_0$, we have $\lambda_{0,\min}\leq \lambda_i(0)\leq \lambda_{0,\max}$. We conclude that $\min\{\lambda_{0,\min}, 1/\beta\}\leq \lambda_i(t)\leq \max\{\lambda_{0,\max},1/\alpha\}$.

For the Gaussian approximate Fisher-Rao gradient flow in \cref{eq:Gaussian Fisher-Rao}, we have 
\begin{align}\label{e:Ctbb2}
    C_t-\beta C_t^2\preceq\partial_t C_t=C_t+C_t\E_{\rho_{\rhoa_t}}[\Hess\log \rho_{\rm post}(\theta)]C_t\preceq C_t-\alpha C_t^2.
\end{align}
By plugging \cref{e:Ctbb2} into \cref{e:eigvariation} and noticing $u_i(t)^T C^2_t u_i(t)=\lambda^2_i(t)$, we get 
\begin{align}
\lambda_i(t)-\beta \lambda_i(t)^2\leq    \partial_t{\lambda}_i(t)\leq\lambda_i(t)-\alpha \lambda_i(t)^2,\quad \lambda_{0,\min}\leq \lambda_i(0)\leq \lambda_{0,\max}.
\end{align}
If $\lambda_i(t)\geq 1/\alpha$, then $\partial_t{\lambda}_i(t)\leq0$; if $\lambda_i(t)\leq 1/\beta$, then $\partial_t{\lambda}_i(t)\geq0$. We conclude that $\min\{\lambda_{0,\min}, 1/\beta\}\leq \lambda_i(t)\leq \max\{\lambda_{0,\max},1/\alpha\}$.

For the Gaussian approximate Wasserstein gradient flow in \cref{eq:Gaussian-Wasserstein}, we have 
\begin{align}\label{e:Ctbb3}
    \partial_t C_t=2\I+\E_{\rho_{\rhoa_t}}[\Hess\log \rho_{\rm post}(\theta)]C_t+C_t\E_{\rho_{\rhoa_t}}[\Hess\log \rho_{\rm post}(\theta)].
\end{align}
Plugging \cref{e:Ctbb3} into \cref{e:eigvariation} and noticing 
\begin{align}
&u_i(t)^T \E_{\rho_{\rhoa_t}}[\Hess\log \rho_{\rm post}(\theta)]C_t u_i(t)=\lambda_i(t)u_i(t)^T \E_{\rho_{\rhoa_t}}[\Hess\log \rho_{\rm post}(\theta)] u_i(t),\\
&-\beta\leq u_i(t)^T \E_{\rho_{\rhoa_t}}[\Hess\log \rho_{\rm post}(\theta)] u_i(t)\leq -\alpha,
\end{align} 
we get 
\begin{align}
2-2\beta \lambda_i(t)\leq    \partial_t{\lambda}_i(t)\leq 2-2\alpha \lambda_i(t),\quad \lambda_{0,\min}\leq \lambda_i(0)\leq \lambda_{0,\max}.
\end{align}
If $\lambda_i(t)\geq 1/\alpha$, then $\partial_t{ \lambda}_i(t)\leq0$; if $\lambda_i(t)\leq 1/\beta$, then $\partial_t{ \lambda}_i(t)\geq0$. We conclude that $\min\{\lambda_{0,\min}, 1/\beta\}\leq \lambda_i(t)\leq \max\{\lambda_{0,\max},1/\alpha\}$.
\end{proof}

\vspace{0.1in}

\begin{proof}[Proof of \cref{p:logconcave}]
The time derivative of the KL divergence is 
\begin{align}\begin{split}\label{e:dtKL}
    &\phantom{{}={}}\partial_t {\mathrm{KL}}\Bigl[\rho_{\rhoa_t}\Big\Vert  \rho_{\rm post}\Bigr]
    =\int \partial_t \rho_{\rhoa_t}(\theta) \log \frac{\rho_{\rhoa_t}(\theta)}{\rho_{\rm post}(\theta)} \mathrm{d}\theta\\
    &=\partial_t{m}_t^T\int \partial_{m_t} \N(m_t, C_t) \log \frac{\rho_{\rhoa_t}(\theta)}{\rho_{\rm post}(\theta)} \mathrm{d}\theta
    +{\rm tr} \Bigl[ \partial_t{ C}_t \int \partial_{C_t} \N(m_t, C_t) \log \frac{\rho_{\rhoa_t}(\theta)}{\rho_{\rm post}(\theta)} \mathrm{d}\theta \Bigr]\\
    &=-\partial_t{m}_t^T\int \nabla_\theta \rho_{\rhoa_t}(\theta) \log \frac{\rho_{\rhoa_t}(\theta)}{\rho_{\rm post}(\theta)} \mathrm{d}\theta
    +\frac{1}{2}{\rm tr}\Bigl[\partial_t{ C}_t \int \Hess\rho_{\rhoa_t}(\theta) \log \frac{\rho_{\rhoa_t}(\theta)}{\rho_{\rm post}(\theta)} \mathrm{d}\theta\Bigr],
\end{split}\end{align}
where in the last equality we use \cref{e:derrho}.
After an integration by parts, we can further simplify \cref{e:dtKL} as
\begin{equation}
\label{e:dtKL2}
    \begin{aligned}
    \partial_t {\mathrm{KL}}\Bigl[\rho_{\rhoa_t}\Big\Vert  \rho_{\rm post}\Bigr]
    =&-\partial_t m_t^T\E_{\rho_{\rhoa_t}}\left[\nabla_\theta \log \rho_{\rm post}(\theta) \right]\\
    &{}-\frac{1}{2}{\rm tr}\Bigl[\partial_t C_t (C_t^{-1}+\E_{\rho_{\rhoa_t}}\left[\Hess \log \rho_{\rm post}(\theta) \right]) \Bigr].
\end{aligned}
\end{equation}

We first discuss the Gaussian approximate Wasserstein gradient flow. For logconcave posterior distribution satisfying \cref{a:logconcave}, ${\mathrm{KL}}[\cdot \Vert  \rho_{\rm post}(\theta )]$ is $\alpha$-convex 
\begin{align}\label{e:convex}
    {\mathrm{KL}}[\mu_t \Vert  \rho_{\rm post}]\leq (1-t) {\mathrm{KL}}[\mu_0 \Vert  \rho_{\rm post}]
    +t {\mathrm{KL}}[\mu_1 \Vert  \rho_{\rm post}]-\frac{\alpha t(1-t)}{2}W^2_2(\mu_0, \mu_1),
\end{align}
where $\mu_t$ is the geodesic from $\mu_0$ to $\mu_1$ under the Wasserstein metric. See \cite[Chapter 9.4]{ambrosio2005gradient}. The Gaussian space is geodesically 
closed under the Wasserstein metric: if $\mu_0, \mu_1$ are two Gaussian distributions, every measure on the geodesic $\mu_t$ from $\mu_0$ to $\mu_1$ is still a Gaussian distribution. It follows from the convexity that the minimizer $\rho_{\rhoa_\star}$ of ${\mathrm{KL}}[\rho_{\rhoa}\Vert  \rho_{\rm post}(\theta )]$ is unique.

Employing \cref{eq:Gaussian-Wasserstein} in \cref{e:dtKL2} leads to 
\begin{align*}
    \partial_t {\mathrm{KL}}&\Bigl[\rho_{\rhoa_t}\Big\Vert  \rho_{\rm post}\Bigr]
    \\
    =&-\left\|\E_{\rho_{\rhoa_t}}\left[\nabla_\theta \log \rho_{\rm post}(\theta) \right]\right\|_2^2
      -\frac{1}{2}{\rm tr}\Bigl[ (2\I + \E_{\rho_{\rhoa_t}}\bigl[\Hess\log \rho_{\rm post}(\theta) \bigr]C_t 
    \\
    &+ C_t\E_{\rho_{\rhoa_t}}\bigl[\Hess\log \rho_{\rm post}(\theta) \bigr]) (C_t^{-1}+\E_{\rho_{\rhoa_t}}\left[\Hess \log \rho_{\rm post}(\theta) \right]) \Bigr]
    \\
    =&-\left\|\E_{\rho_{\rhoa_t}}\left[\nabla_\theta \log \rho_{\rm post}(\theta) \right]\right\|_2^2
    \\
     &-{\rm tr}\Bigl[ (C_t^{-1}+\E_{\rho_{\rhoa_t}}\left[\Hess \log \rho_{\rm post}(\theta) \right])C_t (C_t^{-1}+\E_{\rho_{\rhoa_t}}\left[\Hess \log \rho_{\rm post}(\theta) \right])\Bigr].
\end{align*}
It has been proven in \cite[Appendix D]{lambert2022variational} that when $-\Hess\log \rho_{\rm post}(\theta)\succeq \alpha I$, we have 
\begin{align}\begin{split}\label{e:lowb1}
    &\left\|\E_{\rho_{\rhoa_t}}\left[\nabla_\theta \log \rho_{\rm post}(\theta) \right]\right\|_2^2
    \\
    &+{\rm tr}\Big[(C_t^{-1}+\E_{\rho_{\rhoa_t}}\left[\Hess \log \rho_{\rm post}(\theta) \right])C_t (C_t^{-1}+\E_{\rho_{\rhoa_t}}\left[\Hess \log \rho_{\rm post}(\theta) \right])\Bigr]\\
    &\geq 2\alpha \left({\mathrm{KL}}\Bigl[\rho_{\rhoa_t}\Big\Vert  \rho_{\rm post}\Bigr]-
     {\mathrm{KL}}\Bigl[\rho_{\rhoa_\star}\Big\Vert  \rho_{\rm post}\Bigr]
    \right).
\end{split}\end{align}
It follows that the gradient flow under the
Wasserstein metric converges exponentially fast, 
\begin{align}
    {\mathrm{KL}}\Bigl[\rho_{\rhoa_t}\Big\Vert  \rho_{\rm post}\Bigr]
    \leq e^{-2\alpha t}{\mathrm{KL}}\Bigl[\rho_{\rhoa_0}\Big\Vert  \rho_{\rm post}\Bigr]+(1-e^{-2\alpha t})
     {\mathrm{KL}}\Bigl[\rho_{\rhoa_\star}\Big\Vert  \rho_{\rm post}\Bigr],
\end{align}
where $\rho_{\rhoa_0}\sim \N(m_{0}, C_{0})$ is the initial condition and $\rho_{\rhoa_\star}\sim \N(m_{\star}, C_{\star})$ is the unique global minimizer of 
\cref{eq:Gaussian-KL0}. Next, we discuss the Gaussian approximate gradient flow~\cref{eq:g-gf}. 
Employing \cref{eq:g-gf} in \cref{e:dtKL2} leads to
\begin{align}\begin{split}\label{e:gdflow}
    \partial_t {\mathrm{KL}}&\Bigl[\rho_{\rhoa_t}\Big\Vert  \rho_{\rm post}\Bigr]
    =-\left\|\E_{\rho_{\rhoa_t}}\left[\nabla_\theta \log \rho_{\rm post}(\theta) \right]\right\|_2^2\\
    &
    -\frac{1}{4}{\rm tr}\Bigl[(C_t^{-1}+\E_{\rho_{\rhoa_t}}\left[\Hess \log \rho_{\rm post}(\theta) \right]) (C_t^{-1}+\E_{\rho_{\rhoa_t}}\left[\Hess \log \rho_{\rm post}(\theta) \right])\Bigr].
\end{split}\end{align}
From \cref{p:stability}, we have the upper bound of $C_t$: $C_t\preceq \max\{1/\alpha, \lambda_{0,\max}\}\I$. In particular for any symmetric matrix $A$, we have $AC_tA\preceq A\max\{1/\alpha, \lambda_{0,\max}\}A$, and ${\rm tr}[AC_tA]\preceq \max\{1/\alpha, \lambda_{0,\max}\}{\rm tr}[A^2]$. By comparing \cref{e:gdflow} and \cref{e:lowb1}, we have 
\begin{align}
\begin{split}
    &\left\|\E_{\rho_{\rhoa_t}}\left[\nabla_\theta \log \rho_{\rm post}(\theta) \right]\right\|_2^2
    +\frac{1}{4}{\rm tr}\Bigl[(C_t^{-1}+\E_{\rho_{\rhoa_t}}\left[\Hess \log \rho_{\rm post}(\theta) \right])^2\Bigr]
    \\
&\geq\frac{\alpha}{\max\{1, 4/\alpha, 4\lambda_{0,\max}\}} \left({\mathrm{KL}}\Bigl[\rho_{\rhoa_t}\Big\Vert  \rho_{\rm post}\Bigr]-
     {\mathrm{KL}}\Bigl[\rho_{\rhoa_\star}\Big\Vert  \rho_{\rm post}\Bigr]
    \right).
\end{split}
\end{align}
It then follows that the gradient flow converges exponentially fast: 
\begin{align}
    {\mathrm{KL}}\Bigl[\rho_{\rhoa_t}\Big\Vert  \rho_{\rm post}\Bigr]
    \leq e^{-K t}{\mathrm{KL}}\Bigl[\rho_{\rhoa_0}\Big\Vert  \rho_{\rm post}\Bigr]+(1-e^{-K t})
    {\mathrm{KL}}\Bigl[\rho_{\rhoa_\star}\Big\Vert  \rho_{\rm post}\Bigr],
\end{align}
where $K=2\alpha/\max\{1, 4/\alpha, 4\lambda_{0,\max}\}$.

Next, we discuss the Gaussian approximate Fisher-Rao gradient flow \cref{eq:Gaussian Fisher-Rao}. Using \cref{eq:Gaussian Fisher-Rao} in \cref{e:dtKL2} leads to
\begin{align}\begin{split}\label{e:ngdflow}
    &\phantom{{}={}}\partial_t {\mathrm{KL}}\Bigl[\rho_{\rhoa_t}\Big\Vert  \rho_{\rm post}\Bigr]
    =-\E_{\rho_{\rhoa_t}}\left[\nabla_\theta \log \rho_{\rm post}(\theta) \right]^T C_t \E_{\rho_{\rhoa_t}}\left[\nabla_\theta \log \rho_{\rm post}(\theta) \right]\\
    &-\frac{1}{2}{\rm Tr} \Bigl[ C_t(C_t^{-1}+\E_{\rho_{\rhoa_t}}[\Hess\log \rho_{\rm post}(\theta)])C_t (C_t^{-1}+\E_{\rho_{\rhoa_t}}\left[\Hess \log \rho_{\rm post}(\theta) \right])\Bigr].
\end{split}\end{align}
From \cref{p:stability}, we have the lower bound of $C_t$: $C_t\succeq \min\{1/\beta, \lambda_{0,\min}\}\I$. In particular for any vector $u$, $u^T C_t u\geq \min\{1/\beta, \lambda_{0,\min}\}\|u\|_2^2$; for any positive definite symmetric matrix $A$, we have  ${\rm tr}[C_tA]\geq \min\{1/\beta, \lambda_{0,\min}\}{\rm tr}A$. By comparing \cref{e:ngdflow} and \cref{e:lowb1}, we have 
\begin{align}\begin{split}
&\E_{\rho_{\rhoa_t}}\left[\nabla_\theta \log \rho_{\rm post}(\theta) \right]^T C_t \E_{\rho_{\rhoa_t}}\left[\nabla_\theta \log \rho_{\rm post}(\theta) \right]
\\
+&\frac{1}{2}{\rm tr}\Bigl[C_t(C_t^{-1}+\E_{\rho_{\rhoa_t}}[\Hess\log \rho_{\rm post}(\theta)])C_t (C_t^{-1}+\E_{\rho_{\rhoa_t}}\left[\Hess \log \rho_{\rm post}(\theta) \right])\Bigr]
\\
\geq& \alpha \min\{1/\beta, \lambda_{0,\min}\} \left({\mathrm{KL}}\Bigl[\rho_{\rhoa_t}\Big\Vert  \rho_{\rm post}\Bigr]-
     {\mathrm{KL}}\Bigl[\rho_{\rhoa_\star}\Big\Vert  \rho_{\rm post}\Bigr]
    \right).
\end{split}\end{align}
Then, it follows that the Gaussian approximate Fisher-Rao gradient flow converges exponentially fast: 
\begin{align}
    \label{eq:G-Fisher-Rao-K}
    {\mathrm{KL}}\Bigl[\rho_{\rhoa_t}\Big\Vert  \rho_{\rm post}\Bigr]
    \leq e^{-K t}{\mathrm{KL}}\Bigl[\rho_{\rhoa_0}\Big\Vert  \rho_{\rm post}\Bigr]+(1-e^{- K t})
     {\mathrm{KL}}\Bigl[\rho_{\rhoa_\star}\Big\Vert  \rho_{\rm post}\Bigr],
\end{align}
where $K=\alpha\min\{1/\beta, \lambda_{0,\min}\}$.
~
\\
~

To prove~\cref{e:W2distance}, we first prove the following lemma.
\begin{lemma}
\label{lem:astar}
KL divergence can be used to upper bound Wasserstein distance: for any Gaussian measure $\rho_\rhoa$, it holds that
\begin{align}
    \frac{\alpha}{2} W_2^2(\rho_{\rhoa}, \rho_{\rhoa_\star})\leq {\mathrm{KL}}\Bigl[\rho_\rhoa\Big\Vert  \rho_{\rm post}\Bigr]-
     {\mathrm{KL}}\Bigl[\rho_{\rhoa_\star}\Big\Vert  \rho_{\rm post}\Bigr],
\end{align}
where $\rhoa_\star=(m_{\star}, C_{\star})$ and $\rho_{\rhoa_\star}\sim \N(m_{\star}, C_{\star})$ is the unique global minimizer of 
\cref{eq:Gaussian-KL0}. 
\end{lemma}

We take $\mu_0=\rho_{\rhoa_\star}$ and $\mu_1=\rho_\rhoa$ in \cref{e:convex}, and $\mu_t$ the geodesic flow from $\rho_{\rhoa_\star}$ to $\rho_\rhoa$. Let $f(t)={\mathrm{KL}}[\mu_t \Vert  \rho_{\rm post}]-{\mathrm{KL}}[\rho_{\rhoa_\star} \Vert  \rho_{\rm post}]$. Then \cref{e:convex} implies that $f(t)$ is convex
\begin{align}
    f\bigl((1-t)c_1+tc_2\bigr)
    &\leq (1-t)f(c_1)+tf(c_2)-\frac{\alpha W^2_2(\mu_{c_1}, \mu_{c_2})}{2}t(1-t)\\
    &=(1-t)f(c_1)+tf(c_2)-\frac{\alpha (c_2-c_1)^2 W^2_2(\rho_{\rhoa}, \rho_{\rhoa_\star})}{2}t(1-t),
\end{align}
for any $0\leq c_1\leq c_2\leq 1$. It follows that $f(t)$ is convex and by taking $t=1/2$ and $c_1,c_2\rightarrow c$
\begin{align}\label{e:ftt}
    f''(c)=
\frac{f(c_1)+f(c_2)-2f((c_1+c_2)/2)}{(c_1-c_2)^2/4}
    \geq \alpha W_2^2(\rho_{\rhoa},\rho_{\rhoa_\star}). 
\end{align}
Since $\rho_{\rhoa_\star}$ is the minimizer of ${\mathrm{KL}}[\rho_{\rhoa} \Vert  \rho_{\rm post}]$, we have that $f(0)=0$ and $f'(0)=0$. By integrating \cref{e:ftt}, we get 
\begin{align}\label{e:W2bound}
    {\mathrm{KL}}[\rho_{\rhoa} \Vert  \rho_{\rm post}]-{\mathrm{KL}}[\rho_{\rhoa_\star} \Vert  \rho_{\rm post}]=f(1)\geq \frac{\alpha}{2} W_2^2(\rho_{\rhoa},\rho_{\rhoa_\star}).
\end{align}

\Cref{e:W2distance} follows from combining \cref{e:W2bound} and \cref{e:KLconv},
\begin{align}
    W_2^2(\rho_{\rhoa_t},\rho_{\rhoa_\star})\leq \frac{2}{\alpha}\left({\mathrm{KL}}[\rho_{\rhoa_t} \Vert  \rho_{\rm post}]-{\mathrm{KL}}[\rho_{\rhoa_\star} \Vert  \rho_{\rm post}]\right)\\
    \leq \frac{2e^{-Kt}}{\alpha}\left({\mathrm{KL}}[\rho_{\rhoa_0} \Vert  \rho_{\rm post}]-{\mathrm{KL}}[\rho_{\rhoa_\star} \Vert  \rho_{\rm post}]\right).
\end{align}
\end{proof}

\subsection{Proof of \Cref{p:logconcave-local}}
\label{proof:proof-counter-example-Gaussian-Fisher-Rao}

Let  $\rho_{\rhoa_\star}$ be $\N(m_{\star}, C_{\star})$, the
unique minimizer of \cref{eq:Gaussian-KL0}, noting that this is also the unique 
stationary point of Gaussian approximate Fisher-Rao gradient flow~\cref{eq:Gaussian Fisher-Rao}; see ~\Cref{l:gd} for definition of $\rhoa_\star$. It satisfies
\begin{equation}
\label{eq:stationary}
   \E_{\rho_{\rhoa_\star}}\bigl[\nabla_\theta  \log \rho_{\rm post}(\theta)   \bigr] = 0
\quad 
\textrm{and}
\quad
\E_{\rho_{\rhoa_\star}}\bigl[-\Hess\log \rho_{\rm post}(\theta) \bigr] = C_{\star}^{-1}. 
\end{equation}

For $N_{\theta} = 1$, we can calculate the linearized Jacobian matrix of the ODE system~\cref{eq:Gaussian Fisher-Rao} around {$(m_{\star}, C_{\star})$:
\begin{equation}
\begin{split}
\label{eq:Jacobian}
&\frac{\partial \textrm{RHS}}{\partial(m, C)}\Big\vert_{(m, C) = (m_{\star}, C_{\star})} 
\\
=&
 \begin{bmatrix}
-1 & - \frac{1}{2}\E_{\rho_{\rhoa_\star}}\bigl[-\Hess\log \rho_{\rm post}(\theta) (\theta - m_{\star})\bigr]  \\
-\E_{\rho_{\rhoa_\star}}\bigl[-\Hess\log \rho_{\rm post}(\theta) (\theta - m_{\star})\bigr] C_{\star} &  -\frac{1}{2} - \frac{1}{2}\E_{\rho_{\rhoa_\star}}[-\Hess\log \rho_{\rm post}(\theta) (\theta - m_{\star})^2 ]
\end{bmatrix}.
\end{split}
\end{equation}
}

{We further define}
\begin{equation}
\begin{aligned}
\label{eq:A1A2}
A_1 &:= \E_{\rho_{\rhoa_\star}}\bigl[-\Hess\log \rho_{\rm post}(\theta) (\theta - m_{\star})\bigr],
\\
A_2 &:= \E_{\rho_{\rhoa_\star}}[-\Hess\log \rho_{\rm post}(\theta) (\theta - m_{\star})^2 ] \geq 0.
\end{aligned}
\end{equation}
Using the Cauchy-Schwarz inequality and \cref{eq:stationary}, we have 
\begin{align}
\label{eqn: D.30}
A_2 C_{\star}^{-1} = A_2 \E_{\rho_{\rhoa_\star}}\bigl[-\Hess\log \rho_{\rm post}(\theta) \bigr] \geq A_1^2. 
\end{align}
By direct calculations, the two eigenvalues of \cref{eq:Jacobian} satisfy
\begin{subequations}
\label{eq:lambda}
\begin{alignat}{3}
    \lambda_1  &= \frac{-(\frac{3}{2} + \frac{1}{2}A_2) - \sqrt{(\frac{1}{2} - \frac{1}{2}A_2)^2 + 2A_1^2 C_{\star}}}{2} \leq -1
\\
\label{eq:lambda2}
\lambda_2  &= \frac{-(\frac{3}{2} + \frac{1}{2}A_2) + \sqrt{(\frac{1}{2} - \frac{1}{2}A_2)^2 + 2A_1^2 C_{\star}}}{2} 
\\
&= -\frac{1 + A_2 - A_1^2 C_{\star}}{(\frac{3}{2} + \frac{1}{2}A_2) + \sqrt{(\frac{1}{2} - \frac{1}{2}A_2)^2 + 2A_1^2 C_{\star}}} \leq -\frac{1}{3+ A_2}, \nonumber
\end{alignat}
\end{subequations}
where in the last inequality, we have used \eqref{eqn: D.30}. In the following, we prove bounds on $\lambda_2$.
\vspace{1em}

\paragraph{\bf Step 1 (Upper bound)}
Since the upper bounds \eqref{eq:lambda} of the two eigenvalues depend on $A_2$, we will  first prove that
\begin{align}
\label{eq:A_2}
    A_2 \leq \bigl(4 + \frac{4}{\sqrt{\pi}}\bigr)\bigl(\log(\frac{\beta}{\alpha}) + 1\bigr).
\end{align}
Without loss of generality, we assume $m_{\star} = 0$; otherwise we can always achieve this through a change of variable. Considering now only the right half of the integration (i.e. integration from $0$ to $+\infty$)
defining $A_2$, we have
\begin{equation}
\label{eq:A_2_half}
\begin{aligned}
& \int_0^{+\infty} -\Hess\log \rho_{\rm post}(\theta)  \theta^2 \frac{1}{\sqrt{2\pi C_{\star}}} e^{-\frac{1}{2} \frac{\theta^2}{C_{\star}}}\dd \theta 
\\
= & \int_{0}^{A} -\Hess\log \rho_{\rm post}(\theta)  \theta^2 \frac{1}{\sqrt{2\pi C_{\star}}} e^{-\frac{1}{2} \frac{\theta^2}{C_{\star}}}\dd \theta + \int_{A}^{+\infty} -\Hess\log \rho_{\rm post}(\theta)\theta^2 \frac{1}{\sqrt{2\pi C_{\star}}} e^{-\frac{1}{2} \frac{\theta^2}{C_{\star}}}\dd \theta
\\
\leq & \frac{A^2}{C_{\star}} + \int_{A}^{+\infty} -\Hess\log \rho_{\rm post}(\theta)\theta^2 \frac{1}{\sqrt{2\pi C_{\star}}} e^{-\frac{1}{2} \frac{\theta^2}{C_{\star}}}\dd \theta
\quad (\textrm{Using } \cref{eq:stationary} \textrm{ and } \theta \leq A)
\\
\leq & \frac{A^2}{C_{\star}} + \beta \int_{A}^{+\infty} \frac{\theta ^3}{A} \frac{1}{\sqrt{2\pi C_{\star}}} e^{-\frac{1}{2} \frac{\theta^2}{C_{\star}}}\dd \theta
\quad (\textrm{Using } \theta \geq A \textrm{ and } \beta\textrm{-smoothness of } \log \rho_{\rm post})
\\
= & \frac{A^2}{C_\star} + \frac{2\beta C_\star^2}{A\sqrt{2\pi C_\star}} (\frac{A^2}{2C_\star} + 1) e^{-\frac{A^2}{2C_\star}} \quad (\textrm{Direct calculation})\\
\leq & \frac{A^2}{C_{\star}} + \frac{\beta }{\sqrt{\pi} \alpha} (\frac{A}{\sqrt{2C_{\star}}} + \frac{\sqrt{2C_{\star}}}{A} )  e^{-\frac{A^2}{2C_{\star}}}  \qquad (\textrm{Using } C_{\star} \leq \frac{1}{\alpha})\\
 \leq & (2 + \frac{2}{\sqrt{\pi}})(\log(\frac{\beta}{\alpha}) + 1),
\end{aligned}
\end{equation}
where in the last inequality, we have chosen $A$ such that $\frac{A^2}{2C_{\star}} = \max\{\log(\frac{\beta}{\alpha}), 1\}$ since the previous derivations work for any positive $A$.
{We can get a similar bound for the left half of the integration defining $A_2$~(i.e. integration from $-\infty$ to $0$). Combining these two bounds leads to \cref{eq:A_2}.}
{Bringing \cref{eq:A_2} into \cref{eq:lambda} leads to}
\begin{align}
\label{eq:counter:lambda2}
\lambda_2  \leq -\frac{1}{(7+\frac{4}{\sqrt{\pi}})\bigl(1 + \log(\frac{\beta}{\alpha})\bigr)}.
\end{align}
{Therefore, we finish the proof of \eqref{eq:lambda_max} in \Cref{p:logconcave-local}.}

\vspace{1em}
\paragraph{\bf Step 2 (Lower bound)}
Next, we will construct an example to show the bound is sharp. The basic idea is to construct a sequence of triplets $\rho_{{\rm post},n}$,  $\beta_n$, and $\alpha_n$, where $\lim_{n \to \infty} \frac{\beta_n}{\alpha_n} = \infty$, and the corresponding $-\lambda_{2,n} = \mathcal{O}\bigl(1/ \log(\frac{\beta_n}{\alpha_n})\bigr)$. In the following proof, we ignore the subscript $n$ for simplicity.

We consider the following sequence of posterior density functions $\rho_{\rm post}$, such that
\begin{equation}
\label{eq:logconcave-slow}
\begin{aligned}
    -\Hess \log \rho_{\rm post}(\theta) &= \int H(x) \frac{e^{-\frac{(\theta - x)^2}{2\sigma^2}}}{\sigma\sqrt{2\pi}} {\rm d}x
\qquad H(x) = 
\begin{cases}
  \beta   & \gamma - \sigma < x < \gamma + \sigma \\
  \alpha  & \textrm{ Otherwise}
\end{cases},
  \\
\nabla_{\theta}\log \rho_{\rm post}(\theta) &= \int_{-\infty}^{\theta}-\Hess \log \rho_{\rm post}(\theta') d\theta' + c,
\end{aligned}
\end{equation}
where $-\Hess \log \rho_{\rm post}(\theta)$ is a smoothed bump function containing four parameters $\gamma, \sigma > 0$ and $0 < \alpha < \beta$. Clearly, we have \[\alpha\I \preceq -\Hess \log \rho_{\rm post}(\theta) \preceq \beta \I.\]
Moreover, we will have another parameter $c$ to determine $\nabla_{\theta}\log \rho_{\rm post}(\theta)$.
It is worth mentioning that for such $\alpha$-strongly logconcave posterior, the Gaussian variational inference \eqref{eq:Gaussian-KL0} has a unique minimizer $(m^\star, C^\star)$, which is determined by the stationary point condition in~\cref{eq:stationary}; see also \cite{lambert2022variational}.

The \textbf{intuition} behind the construction of the bump function is as follows. 
Our objective is to ensure that the dominant eigenvalue, denoted as $\lambda_2$ in equation \eqref{eq:Jacobian}, as large as possible (thus leading to a lower bound).
Since $\lambda_2$ satisfies
\begin{align*}
\lambda_2 = -\frac{1 + A_2 - A_1^2 C_{\star}}{\frac{3}{2}+\frac{1}{2}A_2 + \sqrt{(\frac{1}{2}A_2 - \frac{1}{2})^2 + 2A_1^2 C_{\star}}}  &\geq -\frac{1 + A_2 - A_1^2 C_{\star}}{1 + A_2 } =
-\frac{\frac{1}{A_2} + (1 - \frac{A_1^2 C_{\star}}{A_2})}{1 + \frac{1}{A_2} },
\end{align*}
we require $A_2$ to be as large as possible, while ensuring that the expression $(1 - \frac{A_1^2 C_{\star}}{A_2})$ as small as possible. Recall the definitions of $A_1$ and $A_2$ in \cref{eq:A1A2}; for the latter term, we have 
\begin{align}
A_2 C_{\star}^{-1} &= A_2 \E_{\rho_{\rhoa_\star}}\bigl[-\Hess\log \rho_{\rm post}(\theta) \bigr] \\
& = \E_{\rho_{\rhoa_\star}}[-\Hess\log \rho_{\rm post}(\theta) (\theta - m_{\star})^2 ] \E_{\rho_{\rhoa_\star}} \bigl[-\Hess\log \rho_{\rm post}(\theta) \bigr] \\
&\geq (\E_{\rho_{\rhoa_\star}}[-\Hess\log \rho_{\rm post}(\theta) (\theta - m_{\star})])^2 = A_1^2
\end{align}
due to the Cauchy-Schwarz inequality. Thus we get $(1 - \frac{A_1^2 C_{\star}}{A_2}) \geq 0$. To make this term as close to zero as possible, we consider when the Cauchy-Schwarz inequality can become equality. In fact, we need $-\Hess \log \rho_{\rm post}(\theta)$ to take the form of a delta function. As in the assumption we have $\log \rho_{\rm post}(\theta) \in C^2$ so this is not achievable. To approximate this condition, we can construct $-\Hess \log \rho_{\rm post}(\theta)$ as a bump function $H(x)$ and gradually narrow the width of the bump to approach zero. The Gaussian kernel is employed to smooth the bump function and simplify the subsequent calculations; this is why the form of $H$ in \eqref{eq:logconcave-slow} is constructed.

Now, we will provide a detailed construction. Instead of specifying $\beta$ and $c$, we can specify $m_{\star}$ and $C_{\star}$ since there is a one-to-one correspondence between them. We specify\footnote{These choices of parameters are motivated by the subsequent calculations.} that
\begin{align}
\label{eq:construct-A}
    \sigma = \gamma^{1.5} \quad m_{\star} = 0 \quad C_{\star} = -\frac{\gamma^2}{2\log\gamma}  - \gamma^3 \quad \textrm{and} \quad\alpha = \frac{1}{(-\log \gamma) C_\star}. 
\end{align}
Then, $\beta$ and $c$ are determined 
by the stationary point condition~\eqref{eq:stationary}, namely
\begin{equation}
\label{eq:stationary-cov}
\begin{aligned}
 {C_{\star}^{-1}}
 &= \E_{\rho_{\rhoa_\star}}\bigl[-\Hess\log \rho_{\rm post}(\theta) \bigr] =   \int H(x) \frac{1}{\sqrt{2\pi(\sigma^2 + C_{\star})}}e^{-\frac{(x - m_{\star})^2}{2(\sigma^2 + C_{\star})}} {\rm d}x \\
 &= \alpha + (\beta - \alpha)\int_{\gamma-\sigma}^{\gamma+\sigma}\frac{1}{\sqrt{2\pi(\sigma^2 + C_{\star})}}e^{-\frac{(x - m_{\star})^2}{2(\sigma^2 + C_{\star})}}{\rm d}x \, .
 \\
&  \hspace{-1cm}0 = \E_{\rho_{\rhoa_\star}}\bigl[\nabla_{\theta}\log \rho_{\rm post}(\theta) \bigr]  = \int_{-\infty}^{\infty} \rho_{\rhoa_\star}(\theta) \int_{-\infty}^{\theta}-\Hess \log \rho_{\rm post}(\theta') {\rm d}\theta' {\rm d}\theta + c
 \end{aligned}
\end{equation}
We will let $\gamma \to 0$ later. Note that in the above system, when $\gamma$ is close to zero, for any $C_{\star} = \frac{1}{-\alpha \log \gamma}\leq \frac{1}{\alpha}$, we will have $\beta > \alpha$, so the above lead to valid specification of parameters. We now only have one free parameter $\gamma$. We can now also view the triplet $\rho_{{\rm post}}$,  $\beta$, and $\alpha$ as functions parameterized by $\gamma$.

In the following, we will let $\gamma \rightarrow 0$ and estimate the leading order of $A_1$, $A_2$, $\alpha$ and $\beta$ in terms of $\gamma$. 
Let denote $m_{\star  \sigma} = \frac{xC_{\star} + m_{\star}\sigma^2}{\sigma^2 + C_{\star}}$ and $C_{\star  \sigma} = \frac{\sigma^2C_\star}{\sigma^2 + C_{\star}}$, we have
\begin{align*}
&    -\Hess \log \rho_{\rm post}(\theta) \N(m_{\star}, C_{\star}) =  \int H(x) \frac{e^{-\frac{(\theta - m_{\star  \sigma})^2}{2C_{\star  \sigma}}}}{\sqrt{2\pi C_{\star  \sigma}}}\frac{1}{\sqrt{2\pi(\sigma^2 + C_{\star})}}e^{-\frac{(x - m_{\star})^2}{2(\sigma^2 + C_{\star})}}{\rm d}x.
\end{align*}
Bringing this to \cref{eq:A1A2} leads to
\begin{equation}
\label{eq:A1A2-intg}
    \begin{split}
&A_1 = \E_{\rho_{a_\star}}\bigl[-\Hess\log \rho_{\rm post}(\theta) (\theta - m_{\star})\bigr]
=
\int H(x) \frac{m_{\star  \sigma} - m_{\star}}{\sqrt{2\pi(\sigma^2 + C_{\star})}}e^{-\frac{(x - m_{\star})^2}{2(\sigma^2 + C_{\star})}}{\rm d}x,
\\
&A_2 = \E_{\rho_{a_\star}}[-\Hess\log \rho_{\rm post}(\theta) (\theta - m_{\star})^2 ]
=
\int H(x) \frac{C_{\star\sigma} + (m_{\star\sigma}-m_{\star})^2}{\sqrt{2\pi(\sigma^2 + C_{\star})}}e^{-\frac{(x - m_{\star})^2}{2(\sigma^2 + C_{\star})}}{\rm d}x .
    \end{split}
\end{equation}
With the definitions in~\cref{eq:construct-A}, we have the following estimation about Gaussian integration
\begin{align}
\label{eq:construct-A-G}
\int_{\gamma-\sigma}^{\gamma + \sigma} x^ne^{-\frac{x^2}{2(\sigma^2 + C_{\star})}}{\rm d}x  
= 2\sqrt{\gamma}\gamma^{n+2} (1 +\bigO(\log\gamma\sqrt\gamma)).
\end{align}
Bringing the definition of $\sigma$ and $C_{\star}$ and using change-of-variable with $y = \frac{x}{\gamma}$, \cref{eq:construct-A-G} becomes
\begin{align}
\label{eq:x-y}
\int_{\gamma-\sigma}^{\gamma + \sigma} x^ne^{-\frac{x^2}{2(\sigma^2 + C_{\star})}}{\rm d}x  
= \gamma^{n+1}\int_{1 -\sqrt\gamma}^{1+\sqrt\gamma} y^ne^{y^2\log\gamma}{\rm d}y. 
\end{align}
Bringing the following inequalities into \cref{eq:x-y} leads to \cref{eq:construct-A-G}
\begin{align*}
\gamma^{n+1}\int_{1 -\sqrt\gamma}^{1+\sqrt\gamma} y^ne^{y^2\log\gamma}{\rm d}y \geq& 2\sqrt{\gamma}\gamma^{n+1}  (1-\sqrt\gamma)^ne^{(1+\sqrt\gamma)^2\log\gamma}  \\
                                                            =& 2\sqrt{\gamma}\gamma^{n+2} (1 + \bigO(2\log\gamma\sqrt\gamma)),
\\
\gamma^{n+1}\int_{1 -\sqrt\gamma}^{1+\sqrt\gamma} y^ne^{y^2\log\gamma}{\rm d}y \leq& 2\sqrt{\gamma}\gamma^{n+1}  (1+\sqrt\gamma)^ne^{(1 -\sqrt\gamma)^2\log\gamma}  \\
                                                           =& 2\sqrt{\gamma}\gamma^{n+2} (1 - \bigO(2\log\gamma\sqrt\gamma)).
\end{align*}
Here we used the Taylor expansions of $(1 - \sqrt{\gamma})^ne^{(\gamma + 2\sqrt{\gamma})\log\gamma} = \bigO\bigl((1 - n\sqrt{\gamma}) (1 + (\gamma + 2\sqrt{\gamma})\log\gamma)\bigr)$ and $(1 + \sqrt{\gamma})^ne^{(\gamma - 2\sqrt{\gamma})\log\gamma} = \bigO\bigl((1 + n\sqrt{\gamma}) (1 + (\gamma - 2\sqrt{\gamma})\log\gamma)\bigr)$.
Then the estimation for $A_1$, $A_2$ from \cref{eq:A1A2-intg} and the covariance condition in ~\cref{eq:stationary} become
\begin{equation}
\label{eq:Taylor-A}
    \begin{aligned}
A_1 &=
\frac{(\beta  -\alpha)C_{\star}}{\sigma^2 + C_{\star}}\int_{\gamma-\sigma}^{\gamma + \sigma} x\frac{1}{\sqrt{2\pi(\sigma^2 + C_{\star})}}e^{-\frac{x^2}{2(\sigma^2 + C_{\star})}}{\rm d}x
\\
&=
\frac{(\beta  -\alpha)C_{\star}}{\sigma^2 + C_{\star}}
\frac{2\gamma^{2.5}\sqrt{-\log \gamma}}{\sqrt{\pi}}(1 +\bigO(\log\gamma\sqrt\gamma))  \quad (\textrm{Using } \eqref{eq:construct-A-G})\\
&=
\frac{2(\beta  -\alpha)}{\sqrt{\pi}}\gamma^{2.5}(-\log\gamma)^{0.5  }  + \bigO\Bigl(\frac{2(\beta  -\alpha)}{\sqrt{\pi}}\gamma^{3}(-\log\gamma)^{1.5  } \Bigr),
\\
A_2 &=\alpha C_{\star} + 
(\beta - \alpha)\int_{\gamma-\sigma}^{\gamma + \sigma} \frac{C_{\star\sigma} + m_{\star\sigma}^2}{\sqrt{2\pi(\sigma^2 + C_{\star})}}e^{-\frac{x^2}{2(\sigma^2 + C_{\star})}}{\rm d}x
\\
&=\alpha C_{\star} + 
(\beta - \alpha)\frac{2\sqrt{-\log\gamma}\gamma^{3.5}(1+2\gamma\log\gamma)(1+\gamma+2\gamma\log\gamma)}{\sqrt{\pi}}(1 +\bigO(\log\gamma\sqrt\gamma)) \quad (\textrm{Using }~\eqref{eq:construct-A-G})
\\
&=\alpha C_{\star} 
+ 
\frac{2(\beta - \alpha)}{\sqrt{\pi}}\gamma^{3.5}(-\log\gamma)^{0.5}  
+
\bigO\Bigl(  
\frac{2(\beta - \alpha)}{\sqrt{\pi}}\gamma^{4}(-\log\gamma)^{1.5}
\Bigr),
\\
{C_{\star}}^{-1}  &= \alpha + (\beta - \alpha)\int_{\gamma-\sigma}^{\gamma+\sigma}\frac{1}{\sqrt{2\pi(\sigma^2 + C_{\star})}}e^{-\frac{x^2}{2(\sigma^2 + C_{\star})}}{\rm d}x 
\\
&= \alpha + (\beta - \alpha)2\gamma^{1.5}\frac{\sqrt{-\log\gamma}}{\sqrt\pi} 
(1 + \bigO(\log\gamma\sqrt\gamma))  \quad (\textrm{Using } \eqref{eq:construct-A-G})
\\
&= \alpha + \frac{2(\beta - \alpha)}{\sqrt{\pi}}\gamma^{1.5}(-\log\gamma)^{0.5} 
+\bigO\Bigl( \frac{2(\beta - \alpha)}{\sqrt{\pi}}\gamma^{2}(-\log\gamma)^{1.5} \Bigr).
\end{aligned}
\end{equation}
Combining \cref{eq:Taylor-A} and \cref{eq:construct-A} leads to
\begin{align}
\label{eq:beta_alpha}
&\alpha =\bigO(\frac{1}{\gamma^2})\qquad 
    \frac{\beta}{\alpha} = \bigO(\frac{(-\log\gamma)^{0.5}}{\gamma^{1.5}}) \qquad \beta - \alpha = \bigO(\frac{(-\log\gamma)^{0.5}}{\gamma^{3.5}})
    \\
    &A_1 = \bigO(\frac{-\log \gamma}{\gamma}) \qquad A_2 = \bigO(-\log \gamma)
    \\
    &1 - \frac{A_1^2 C_{\star}}{A_2} = \frac{A_2C_{\star}^{-1} - A_1^2 }{A_2C_{\star}^{-1}} = \frac{\bigO\Bigl((-\log\gamma)/\gamma^2\Bigr)}{\bigO\Bigl((-\log\gamma)^2/\gamma^2\Bigr)}
    = \bigO(\frac{1}{-\log\gamma}).
\end{align}

Finally, we can bound the large eigenvalue of \cref{eq:Jacobian} by
\begin{align*}
\lambda_2 = -\frac{1 + A_2 - A_1^2 C_{\star}}{\frac{3}{2}+\frac{1}{2}A_2 + \sqrt{(\frac{1}{2}A_2 - \frac{1}{2})^2 + 2A_1^2 C_{\star}}}  \geq 
-\frac{\frac{1}{A_2} + (1 - \frac{A_1^2 C_{\star}}{A_2})}{1 + \frac{1}{A_2} }  
= -\bigO(\frac{1}{-\log \gamma}). 
\end{align*}
Here we use $A_2 \geq 0$ and $1 - \frac{A_1^2 C_{\star}}{A_2} \geq 0 $.
\Cref{eq:beta_alpha} $\log(\frac{\beta}{\alpha}) = \bigO(-\log\gamma)$ indicates that for the constructed logconcave density, the local convergence rate is not faster than $-\bigO(1/\log (\frac{\beta}{\alpha})).$

\subsection{Proof of \Cref{proposition:counter-example-slow}}
\label{proof:counter-example-slow}

Consider the following example, where $\theta \in \R$ and $\V(\theta) = \sum_{k=0}^{2K+1} a_{2k} \theta^{2k}$ with $a_{4K+2} > 0$. We will choose $a_{2k}$ later so that the convergence of these dynamics is $O(t^{-\frac{1}{2K}})$. Recall the Gaussian approximate Fisher-Rao gradient flow is
\begin{equation}
\begin{aligned}
\frac{\dd  m_t}{\dd  t}  & =  C_t\E_{\rho_{\rhoa_t}}[\nabla_\theta \log \rho_{\rm post} ],\\
\frac{\dd  C_t}{\dd  t}
&= C_t + C_t \E_{\rho_{\rhoa_t}}[\Hess \log \rho_{\rm post}]C_t.
\end{aligned}
\end{equation}


We first calculate the explicit formula of the dynamics. For the mean part, we have
\begin{equation}
\label{eqn: D45}
\begin{aligned}
\E_{\rho_{\rhoa}}\bigl[\nabla_\theta  \log \rho_{\rm post}(\theta)   \bigr] =& -\sum_{k=1}^{2K+1} 2k a_{2k} \E_{\rho_{\rhoa}}[\theta^{2k-1}]\\
=& -\sum_{k=1}^{2K+1} 2k a_{2k} \sum_{i=0}^{k-1} \binom{2k-1}{2i+1}m^{2i+1}C^{k-i-1}\frac{(2k-2i-2)!}{2^{k-i-1}(k-i-1)!}\, .
\end{aligned}
\end{equation}
In the above we have used the explicit formula for the moments of Gaussian distributions. By \eqref{eqn: D45}, we know that when $m = 0$, $\E_{\rho_{\rhoa}}\bigl[\nabla_\theta  \log \rho_{\rm post}(\theta)   \bigr] = 0$. Later, we will initialize the dynamics at $m_0 = 0$; as a consequence, $m_t = 0$ so we only need to consider the convariance dynamics given $m=0$.


For the covariance part (assuming $m=0$), using Stein's identity, we get
\begin{equation}
\begin{aligned}
\E_{\rho_{\rhoa}}\bigl[\Hess \log \rho_{\rm post}(\theta) \bigr]_{m = 0}=
\frac{\partial \E_{\rho_{\rhoa}}\bigl[\nabla_\theta  \log \rho_{\rm post}(\theta)   \bigr]}{\partial m}|_{m = 0}=
 - f(C),
\end{aligned}
\end{equation}
where
\begin{equation}
\begin{aligned}
f(C) = \sum_{k=1}^{2K+1} 2k(2k-1) a_{2k}  C^{k-1}\frac{(2k-2)!}{2^{k-1}(k-1)!}.
\end{aligned}
\end{equation}

%
We choose $\{a_{2k}\}_{k=1}^{2K+1}$ such that 
$$ 2k(2k-1)  \frac{(2k-2)!}{2^{k-1}(k-1)!} a_{2k} = \binom{2K+1}{k} (-1)^{2K+1-k},$$ 
which leads to the identity
$$1 - f(C)C = -(C - 1)^{2K+1}.$$

Now, we calculate the explicit form of the dynamics, with the above choice of $\V$. 
We initialize the dynamics with $m_0 = 0$. 
For the Gaussian approximate Fisher-Rao gradient flow~\cref{eq:Gaussian Fisher-Rao}, we have
\begin{equation}
\small
\begin{aligned}    
&\frac{\dd m_t}{\dd t} =  0, \\
&\frac{\dd C_t}{\dd t}  =  C_t(1 - f(C_t)C_t) = -C_t(C_t-1)^{2K+1}.
\end{aligned}
\end{equation}
It is clear that the convergence rate to $C = 1$ is $\bigO(t^{-\frac{1}{2K}})$, if we initialize $C_0$ close to $1$.

In fact, we can also obtain convergence rates for other gradient flows under different metrics. For the vanilla Gaussian approximate gradient flow~\cref{eq:g-gf}, we have
\begin{equation}
\small
\begin{aligned}    
&\frac{\dd m_t}{\dd t} =  0, \\
&\frac{\dd C_t}{\dd t}  =  \frac{1}{2C_t}(1 - f(C_t)C_t) = - \frac{(C_t-1)^{2K+1}}{2C_t}.
\end{aligned}
\end{equation}
For the Gaussian approximate Wasserstein gradient flow~\cref{eq:Gaussian-Wasserstein}, we have
\begin{equation}
\small
\begin{aligned}    
&\frac{\dd m_t}{\dd t} =  0, \\
&\frac{\dd C_t}{\dd t}  =  2(1 - f(C_t)C_t)=-2(C_t-1)^{2K+1}.
\end{aligned}
\end{equation}
In all cases the convergence rate to $C = 1$ is $\bigO(t^{-\frac{1}{2K}})$.

\newpage
\section{Numerical Integration}
\label{sec:integration}
In this section, we discuss how we compute the reference values of $\E\theta$, $\Cov\theta$, and $\E(\cos(\omega^T \theta + b)$ for logconcave posterior and Rosenbrock posterior in~\cref{sec:Numerics}.
For any Gaussian distribution, we have
\begin{align*}
 &\int \cos(\omega^T \theta + b) \N(m, C) \mathrm{d}\theta = \exp\{-\frac{1}{2}\omega^T C\omega \} \cos(\omega^T m+b).
\end{align*}

We can rewrite the logconcave function,
\begin{align*}
 &\V(\theta) =  \frac{1}{2} \frac{({\theta^{(1)}} - {\theta^{(2)}}/\sqrt{\lambda})^2}{10/\lambda} +\frac{{\theta^{(2)}}^4}{20}.
\end{align*}
Therefore, 
\begin{align*}
 &\int e^{-\V(\theta)} \mathrm{d}{\theta^{(1)}} = \sqrt{20\pi/\lambda} e^{-{\theta^{(2)}}^4/20},\\
 &\int e^{-\V(\theta)} {\theta^{(1)}}  \mathrm{d}{\theta^{(1)}} = \sqrt{20\pi/\lambda} \frac{{\theta^{(2)}}}{\sqrt{\lambda}}  e^{-{\theta^{(2)}}^4/20},\\
 &\int e^{-\V(\theta)} {\theta^{(2)}}  \mathrm{d}{\theta^{(1)}} = \sqrt{20\pi/\lambda}   \theta^{(2)}e^{-{\theta^{(2)}}^4/20},\\
 &\int e^{-\V(\theta)} {\theta^{(1)}}^2 \mathrm{d}{\theta^{(1)}} = \sqrt{20\pi/\lambda}(\frac{{\theta^{(2)}}^2}{\lambda} + \frac{10}{\lambda}) e^{-{\theta^{(2)}}^4/20},\\
 &\int e^{-\V(\theta)} {\theta^{(1)}} {\theta^{(2)}}  \mathrm{d}{\theta^{(1)}} = \sqrt{20\pi/\lambda} \frac{{\theta^{(2)^2}}}{\sqrt{\lambda}}  e^{-{\theta^{(2)}}^4/20},\\
 &\int e^{-\V(\theta)} {\theta^{(2)^2}}  \mathrm{d}{\theta^{(1)}} = \sqrt{20\pi/\lambda}   \theta^{(2)^2}e^{-{\theta^{(2)}}^4/20},\\
 &\int e^{-\V(\theta)} \cos({\omega^{(1)}} {\theta^{(1)}} + {\omega^{(2)}} {\theta^{(2)}} + b) \mathrm{d}{\theta^{(1)}}\\
&\quad\quad\quad =  \sqrt{20\pi/\lambda}e^{-\frac{5}{\lambda}{\omega^{(1)}}^2}\cos({\omega^{(1)}}{\theta^{(2)}}/\sqrt{\lambda} + {\omega^{(2)}} {\theta^{(2)}} + b)e^{-{\theta^{(2)}}^4/20}.
\end{align*}
We have
$$
\E[\theta] = \begin{bmatrix}
  0\\
  0
\end{bmatrix}.
$$
Other 2D integrations can be addressed by first performing 1D integration
with respect to $\theta^{(1)}$ analytically, and then the second 1D integration with respect to $\theta^{(2)}$
is computed numerically with $10^7$ uniform points.

We can rewrite the Rosenbrock function as
$$
\V(\theta) = \frac{1}{2}\frac{( {\theta^{(2)}} - {\theta^{(1)}}^2)^2}{10/\lambda} + \frac{(1 - {\theta^{(1)}})^2}{20}. 
$$
Therefore, 
\begin{align*}
 &\int e^{-\V(\theta)} \mathrm{d}{\theta^{(2)}} = \sqrt{20\pi/\lambda} e^{-(1 - {\theta^{(1)}})^2/20},\\
 &\int e^{-\V(\theta)} {\theta^{(1)}}  \mathrm{d}{\theta^{(2)}} = \sqrt{20\pi/\lambda} {\theta^{(1)}} e^{-(1 - {\theta^{(1)}})^2/20},\\
 &\int e^{-\V(\theta)} {\theta^{(2)}}  \mathrm{d}{\theta^{(2)}} = \sqrt{20\pi/\lambda} {\theta^{(1)}}^2 e^{-(1 - {\theta^{(1)}})^2/20},\\
 &\int e^{-\V(\theta)} {\theta^{(1)^2}}  \mathrm{d}{\theta^{(2)}} = \sqrt{20\pi/\lambda} {\theta^{(1)^2}} e^{-(1 - {\theta^{(1)}})^2/20},\\
 &\int e^{-\V(\theta)} {\theta^{(2)}}^2 \mathrm{d}{\theta^{(2)}} = \sqrt{20\pi/\lambda}({\theta^{(1)}}^4 + \frac{10}{\lambda}) e^{-(1 - {\theta^{(1)}})^2/20},\\
 &\int e^{-\V(\theta)} {\theta^{(1)}\theta^{(2)}}  \mathrm{d}{\theta^{(2)}} = \sqrt{20\pi/\lambda} {\theta^{(1)}}^3 e^{-(1 - {\theta^{(1)}})^2/20},\\
 &\int e^{-\V(\theta)} \cos({\omega^{(2)}} {\theta^{(2)}} + {\omega^{(1)}} {\theta^{(1)}} + b) \mathrm{d}{\theta^{(2)}}\\
&\quad\quad\quad =  \sqrt{20\pi/\lambda}e^{-\frac{5}{\lambda}{\omega^{(2)}}^2}\cos({\omega^{(2)}} {\theta^{(1)}}^2 + {\omega^{(1)}} {\theta^{(1)}} + b)e^{-(1 - {\theta^{(1)}})^2/20}.
\end{align*}
We have
$$
\int\int e^{-\V(\theta)} \mathrm{d}{\theta^{(1)}}\mathrm{d}{\theta^{(2)}} = \frac{20\pi}{\sqrt{\lambda}}\qquad
\E[\theta] = \begin{bmatrix}
  1\\
  11
\end{bmatrix}
\qquad 
\Cov[\theta] = \begin{bmatrix}
  10& 20\\
  20& \frac{10}{\lambda} + 240
\end{bmatrix}.
$$
Other 2D integrations can be addressed by first performing 1D integration
with respect to $\theta^{(2)}$ analytically, and then the second 1D integration with respect to $\theta^{(1)}$
is computed numerically with $10^7$ uniform points.
\end{document}